\documentclass{ecai} 

\usepackage{latexsym}
\usepackage{amssymb}
\usepackage{amsmath}
\usepackage{amsthm}
\usepackage{booktabs}
\usepackage{enumitem}
\usepackage{graphicx}
\usepackage[table]{xcolor}
\usepackage{xspace}
\usepackage{algorithm}
\usepackage[noend]{algpseudocode}
\usepackage{mathtools}
\usepackage{tikz}
\usepackage{hhline}
\usepackage{multirow}
\usepackage{anyfontsize}
\usepackage{microtype}

\usetikzlibrary{fit,decorations.pathmorphing}

\newtheorem{theorem}{Theorem}
\newtheorem{lemma}[theorem]{Lemma}

\newcommand{\BibTeX}{B\kern-.05em{\sc i\kern-.025em b}\kern-.08em\TeX}

\colorlet{gray}{black!33}

\algnewcommand{\IIf}[1]{\State\algorithmicif\ #1\ \algorithmicthen}
\algnewcommand{\EndIIf}{}
\algnewcommand{\IIF}[1]{\State \color{gray}\algorithmicif\ #1\ }
\algnewcommand{\ENDIIF}{\color{black}}
\algnewcommand{\IIFE}[3]{\State \color{gray}\algorithmicif\ #1\ \algorithmicthen\ #2\ \algorithmicelse\ #3}
\algnewcommand{\ENDIIFE}{\color{black}}
\algnewcommand{\STATE}[1]{\colorlet{oldcolor}{.}\State \color{gray} #1 \color{oldcolor}}

\algblockdefx[IF]{IF}{ENDIF}[1]
  {\colorlet{oldcolor}{.}\color{gray}\algorithmicif\ #1\ \algorithmicthen\color{oldcolor}}
  {\colorlet{oldcolor}{.}\color{gray}\algorithmicend\ \algorithmicprocedure\color{oldcolor}}
\algblockdefx[PROCEDURE]{PROCEDURE}{ENDPROCEDURE}[2]
  {\colorlet{oldcolor}{.}\color{gray}\algorithmicprocedure\ \mathsc{#1}(#2)\color{oldcolor}}
  {\colorlet{oldcolor}{.}\color{gray}\algorithmicend\ \algorithmicprocedure\color{oldcolor}}
\algblockdefx[FOR]{FORALL}{ENDFOR}[1]
  {\colorlet{oldcolor}{.}\color{gray}\algorithmicforall\ #1\color{oldcolor}}
  {\algorithmicend\ \algorithmicfor}

\makeatletter
\ifthenelse{\equal{\ALG@noend}{t}}%
  {\algtext*{ENDIF}}
  {}%
\ifthenelse{\equal{\ALG@noend}{t}}%
  {\algtext*{ENDFOR}}
  {}%
\ifthenelse{\equal{\ALG@noend}{t}}%
  {\algtext*{ENDPROCEDURE}}
  {}%
\makeatother

\newcommand{\secref}[1]{\S\ref{#1}}

\newcommand{\LAMP}{LAMP\xspace}
\newcommand{\LandmarkPlan}{\ensuremath{\textsc{LandmarkPlan}}\xspace}

\newcommand{\cpvar}[1]{\ensuremath{#1_{C}}}
\newcommand\detvar[1]{\ensuremath{#1_D}}
\newcommand{\prob}{P}

\newcommand{\mathsc}[1]{{\normalfont\textsc{#1}}}
\newcommand{\var}[1]{\ensuremath{\mathit{#1}}}
\newcommand\domain[1]{\texttt{#1}}
 
\DeclareMathOperator*{\argmax}{arg\,max} 

\newcommand{\gradient}[6]{\ifdimcomp{#1pt}{>}{#4pt}{#1}{\ifdimcomp{#1pt}{<}{#2pt}{#1}{\ifdimcomp{#1pt}{>}{#3pt}{\pgfmathparse{100*(#1-#3)/(#4-#3)}\xdef\temp{\pgfmathresult}\cellcolor{#5!\temp!white}}{\pgfmathparse{100*(#3-#1)/(#3-#2}\xdef\temp{\pgfmathresult}\cellcolor{#6!\temp!white}}}}}
\newcommand\gr[1]{\gradient{#1}{0}{50}{100}{red}{teal}}
\newlength{\colwidth}

\newcommand{\continuefigcaption}[1]{\renewcommand{\thefigure}{A\arabic{figure} (cont'd)}\addtocounter{figure}{-1}\caption{#1}\renewcommand{\thefigure}{A\arabic{figure}}}

\begin{document}

\begin{frontmatter}


\paperid{3549} 


\title{Landmark-Assisted Monte Carlo Planning}


\author[A,C]{\fnms{David}~\snm{Chan}\thanks{Corresponding Author. Email: dhchan@cs.umd.edu}\orcid{}}
\author[C]{\fnms{Mark}~\snm{Roberts}\orcid{}}
\author[A,B]{\fnms{Dana}~\snm{Nau}\orcid{}}

\address[A]{Department of Computer Science, University of Maryland, College Park, MD, USA }
\address[B]{Institute for Systems Research, University of Maryland, College Park, MD, USA}
\address[C]{Navy Center for Applied Research in Artificial Intelligence, U.S. Naval Research Laboratory, Washington, DC, USA}


\begin{abstract}
Landmarks---conditions that must be satisfied at some point in every solution plan---have contributed to major advancements in classical planning, but they have seldom been used in stochastic domains. 
We formalize probabilistic landmarks and adapt the UCT algorithm to leverage them as subgoals to decompose MDPs; 
core to the adaptation is balancing between greedy landmark achievement and final goal achievement. 
Our results in benchmark domains show that well-chosen landmarks can significantly improve the performance of UCT in online probabilistic planning, while the best balance of greedy versus long-term goal achievement is problem-dependent.
The results suggest that landmarks can provide helpful guidance for anytime algorithms solving MDPs. 
\end{abstract}

\end{frontmatter}

\section{Introduction}

Landmarks are conditions that must be true in any solution to a planning problem. In classical planning, landmarks have been used to focus search \cite{hoffmann2004ordered} and for heuristic guidance \cite{richter2010lama}. Ongoing work has developed landmarks for more sophisticated heuristics (see \secref{sec:related}).

In stochastic planning, however, landmarks have only been used in limited settings. \citet{speck2015necessary} used a kind of landmark analysis to identify necessary observations. There is also some notable work on using landmarks, or critical states, for plan explanation: \citet{sreedharan2020tldr} defined policy landmarks in MDPs to generate user-explainable policies. To our knowledge, there is scant work exploring the direct application of landmarks to MDP algorithms.

To address this gap, we develop an algorithm called \LAMP---\textbf{L}andmark-\textbf{A}ssisted \textbf{M}onte Carlo \textbf{P}lanning---that uses landmarks as subgoals to solve stochastic planning problems, much like their original use in classical planning.  During rollouts, exploiting landmarks is often more helpful with fewer rollouts and less helpful for a larger rollout budget.  So, we develop a mechanism that balances a focus on achieving near-term landmarks with a focus on the long-term goal in a manner reminiscent of Weighted A* within LAMA~\cite{richter2010lama}.
 
The contributions in this paper include: 
\begin{itemize}
    \item formalizing probabilistic landmarks as a  natural extension of classical landmarks, first defined by \citet{porteous2001extraction}; 
    \item adapting UCT to use landmarks as subgoals during rollouts and developing a weighting parameter that balances between greedily focusing on the next landmark to achieve and the final goal; 
    \item describing the \LAMP algorithm, which uses landmark-assisted UCT rollouts to learn a landmark-sensitive $Q$-function; and,
    \item demonstrating significant performance improvements by \LAMP over standard UCT in five of six probabilistic planning benchmark domains, with the optimal weighting varying across problems.
\end{itemize}
The results from \LAMP highlight the effectiveness of landmarks for probabilistic planning. Similar to the finding that landmarks often resulted in better anytime search for LAMA \cite{richter2010joy}, our results show that using landmarks early in search can yield better solutions.  

\section{Background: classical landmarks}
\label{sec:background}

Following the definitions of \citet{ghallab2016automated}, a classical planning domain $\cpvar{\Sigma}$ is a tuple $(S, \cpvar{A}, \cpvar{\gamma})$, where $S$ and $\cpvar{A}$ are finite sets of states and actions, respectively; and $\cpvar{\gamma}:S\times\cpvar{A}\to S$ is a deterministic state-transition function. An action $a \in \cpvar{A}$ is applicable in state $s \in S$ if $\cpvar{\gamma}(s, a)$ is defined; we write $\mathrm{Applicable}(s)$ to denote the set of all applicable actions in $s$. A \textbf{plan} is a sequence of actions $\cpvar{\pi} = \langle a_1, \dots, a_n\rangle$.  We write $\cpvar{\gamma}(s, \cpvar{\pi})$ to denote the state produced by starting at $s$ and applying the actions in $\cpvar{\pi}$ in order, if all of them are applicable. Let $\cpvar{\prob} = (\cpvar{\Sigma}, s_0, g)$ be a classical planning problem, where $s_0 \in S$ is the initial state and $g$ is the goal condition.  A solution to $\cpvar{P}$ is a plan $\cpvar{\pi}$ such that $\cpvar{\gamma}(s_0, \cpvar{\pi}) \models g$.

Adopting definitions of landmarks from \citet{richter2010lama}, let $\cpvar{P} = (\cpvar{\Sigma}, s_0, g)$ be a classical planning problem and let $\cpvar{\pi}=\langle a_1, \dots, a_n\rangle$ be a plan. A condition $\varphi$ is: 
\begin{itemize}
    \item 
    \textbf{true at time $i$} in $\cpvar{\pi}$ if $\cpvar{\gamma}(s_0, \langle a_1, \dots, a_i\rangle) \models \varphi$.
    \item 
    \textbf{added at time $i$} in $\cpvar{\pi}$ if $\varphi$ is true at time $i$ but not at time $i - 1$. 
    \item 
    \textbf{first added at time $i$} in $\cpvar{\pi}$ if $\varphi$ is true at time $i$, but not at any time $j<i$. 
\end{itemize}

Then, a condition $\varphi$ is a \textbf{landmark} for $\cpvar{\prob}$ if for all plans $\cpvar{\pi}$ such that $\cpvar{\gamma}(s_0, \cpvar{\pi}) \models g$, $\varphi$ is true at some time in $\cpvar{\pi}$. 
Let $\varphi_1$ and $\varphi_2$ be conditions. In problem $\cpvar{\prob}$, there is a:
\begin{itemize}
    \item \textbf{natural ordering} $\varphi_1 \to \varphi_2$ if for all $i$, in every plan where $\varphi_2$ is true at time $i$, $\varphi_1$ is true at some time $j<i$. 
    \item \textbf{necessary ordering} $\varphi_1 \to_{\textnormal{n}} \varphi_2$ if for all $i$, in every plan where $\varphi_2$ is added at time $i$, $\varphi_1$ is true at time $i - 1$. 
    \item \textbf{greedy-necessary ordering} $\varphi_1 \to_{\textnormal{gn}} \varphi_2$ if for all $i$, in every plan where $\varphi_2$ is first added at time $i$, $\varphi_1$ is true at time $i - 1$. 
\end{itemize}
They also defined reasonable orderings \cite{richter2010lama}, but such orderings are not mandatory. They are not used in \LAMP, so we omit them here. 

\section{Probabilistic landmarks}
\label{sec:probabilistic_landmarks}

A probabilistic planning domain $\Sigma$ is a tuple $(S, A, \gamma, \Pr)$, where $S$ and $A$ are finite sets of states and actions, respectively; $\gamma: S \times A \to 2^S$ is a state transition function; and $\Pr(s' \mid s, a)$ is a distribution over $\gamma(s, a)$ giving the probability of reaching state $s'$ when executing action $a$ in state $s$.
An action $a \in A$ is applicable in state $s \in S$ if $\gamma(s, a) \neq \varnothing$; we write $\mathrm{Applicable}(s)$ to denote the set of all applicable actions. A \textbf{policy} is a partial function $\pi: S \to A$ with domain $\mathrm{Dom}(\pi) \subseteq S$ such that for all $s \in \mathrm{Dom}(\pi)$, $\pi(s) \in \mathrm{Applicable}(s)$. The policy $\pi$ is total if $\mathrm{Dom}(\pi) = S$. We will assume all actions have unit cost. 

Let $\prob = (\Sigma, s_0, g)$ be a probabilistic planning problem, where $s_0$ is the initial state and $g$ is the goal condition. Solutions to $\prob$ take the form of a policy. For a policy $\pi$ and initial state $s_0$, a \textbf{history} $\sigma$ is a finite sequence of states $\langle s_0, s_1, \dots, s_h\rangle$, starting in $s_0$, such that for all $i\in\{1,\dots,h\}$, $s_{i}\in\gamma(s_{i-1},\pi(s_{i-1}))$ and for all $i$, $((s_i \models g) \vee (s_i \notin \mathrm{Dom}(\pi))) \iff i = h$. Let $|\sigma|$ denote the length of $\sigma$, and $\sigma_{-1}$ denote the final state in $\sigma$. The probability of a history $\sigma$ of a policy $\pi$ from state $s_0$ is defined to be $\Pr(\sigma \mid \pi, s_0) = \prod_{i = 1}^{|\sigma|} \Pr(s_{i}\mid s_{i - 1}, \pi(s_{i - 1}))$. We write $H(s_0, \pi, g)$ to denote the subset of all histories of $\pi$ from $s_0$ that end in a state that satisfies $g$. A policy $\pi$ is \textbf{safe} for $P$ if $\sum_{\sigma \in H(s_0, \pi, g)} \Pr(\sigma \mid \pi, s_0) = 1$.

We extend the definition of landmarks and landmark orderings to probabilistic planning domains. Let $\prob=(\Sigma,s_0,g)$ be a probabilistic planning problem. For a policy $\pi$ and a history $\sigma = \langle s_0, \dots, s_h\rangle \in H(s_0, \pi, g)$, a condition $\varphi$ is: 
\begin{itemize}
    \item 
    \textbf{true at time $i$} in $\sigma$ if $s_i \models \varphi$. 
    \item 
    \textbf{added at time $i$} in $\sigma$ if $s_i \models \varphi$ and $s_{i - 1} \not\models \varphi$. 
    \item 
    \textbf{first added at time $i$} in $\sigma$ if $s_i \models \varphi$ and $s_{j} \not\models \varphi$ for all $j<i$. 
\end{itemize}
Then, a condition $\varphi$ is a \textbf{landmark} for $\prob$ if for all policies $\pi$ and all histories $\sigma \in H(s_0, \pi, g)$, $\varphi$ is true at some time in $\sigma$. Let $\varphi_1$ and $\varphi_2$ be conditions. In problem $\prob$, there is a:
\begin{itemize}
    \item \textbf{natural ordering} $\varphi_1 \to \varphi_2$ if for all $i$, for every policy $\pi$ and every history $\sigma \in H(s_0, \pi, g)$ where $\varphi_2$ is true at time $i$, $\varphi_1$ is true at some time $j<i$. 
    \item \textbf{necessary ordering} $\varphi_1 \to_{\textnormal{n}} \varphi_2$ if for all $i$, for every policy $\pi$ and every history $\sigma \in H(s_0, \pi, g)$ where $\varphi_2$ is added at time $i$, $\varphi_1$ is true at time $i - 1$. 
    \item \textbf{greedy-necessary ordering} $\varphi_1 \to_{\textnormal{gn}} \varphi_2$ if for all $i$, for every policy $\pi$ and every history $\sigma \in H(s_0, \pi, g)$ where $\varphi_2$ is first added at time $i$, $\varphi_1$ is true at time $i - 1$. 
\end{itemize}

We will construct probabilistic landmarks for $\prob = (\Sigma, s_0, g)$ using the \textbf{all-outcomes determinization} of $\Sigma = (S, A, \gamma, \Pr)$, which is a classical domain, denoted $\detvar{\Sigma} = (S, \detvar{A}, \detvar{\gamma})$, where each action in $\Sigma$ is replaced by a set of deterministic actions, one for each possible outcome of the original action. More formally, let $s \in S$ and $a \in \mathrm{Applicable}(s)$, and suppose $\gamma(s, a) = \{o_1, \dots, o_k\}$. We transform $a$ into a set of deterministic actions $\det(s, a) = \{d_1, \dots, d_k\}$ corresponding to each outcome in $\gamma(s, a)$, where $o_i = \detvar{\gamma}(s, d_i)$ for all $i \in \{1, \dots, k\}$. Then $\detvar{A} = \bigcup_{s \in S, a \in \mathrm{Applicable}(s)} \det(s, a)$. Landmarks in $\detvar{\Sigma}$ carry over to $\Sigma$ as follows: 

\begin{lemma}
\label{lem:1}
    Let $\prob = (\Sigma, s_0, g)$ be a probabilistic planning problem, where $\Sigma = (S, A, \gamma, \Pr)$. Let $\detvar{\prob} = (\detvar{\Sigma}, s_0, g)$ be the classical planning problem where $\detvar{\Sigma} = (S, \detvar{A}, \detvar{\gamma})$ is the all-outcomes determinization of $\Sigma$. Let $\pi$ be a policy for $P$. For every history $\sigma = \langle s_1, \dots, s_h\rangle \in H(s_0, \pi, g)$, there exists a plan $\detvar{\pi} = \langle a_1, \dots, a_h\rangle$ for $\detvar{\prob}$ such that for all $i \in \{1, \dots, h\}$, $\detvar{\gamma}(s_0, \langle a_1, \dots, a_i\rangle) = s_i$; that is,  $\detvar{\pi}$ reaches the same sequence of states as $\sigma$.
\end{lemma}

\begin{proof}
    Let $\sigma = \langle s_0, \dots, s_h \rangle \in H(s_0, \pi, g)$ be a history. For all $i \in \{1, \dots, h\}$, $s_i \in \gamma(s_{i - 1}, \pi(s_{i - 1}))$, thus there exists $d_i \in \det(s_{i - 1}, \pi(s_{i - 1})) \subseteq \detvar{A}$ where $s_i = \detvar{\gamma}(s_{i - 1}, d_i)$. 
    By induction, it follows that for all $i \in \{1, \dots, h\}$, $\detvar{\gamma}(s_0, \langle d_1, \dots, d_i\rangle) = s_i$, so $\detvar{\pi} = \langle d_1, \dots, d_h\rangle$ reaches the same sequence of states as $\sigma$. 
\end{proof}

\begin{theorem}
    Let $\prob = (\Sigma, s_0, g)$ be a probabilistic planning problem, where $\Sigma = (S, A, \gamma, \Pr)$. Let $\detvar{\prob} = (\detvar{\Sigma}, s_0, g)$ be the classical planning problem where $\detvar{\Sigma}$ is the all-outcomes determinization of $\Sigma$. If $\varphi$ is a landmark for $\detvar{\prob}$, then $\varphi$ is a landmark for $\prob$. 
\end{theorem}

\begin{proof}
    Let $\pi$ be a policy for $\prob$, and let $\sigma \in H(s_0, \pi, g)$. By Lemma~\ref{lem:1}, there exists a plan $\detvar{\pi}$ that reaches the same states as $\sigma$. Suppose $\varphi$ is a landmark for $\detvar{\prob}$. Then $\varphi$ must be true at some time in $\detvar{\pi}$, and so it follows that $\varphi$ is also true at some time in $\sigma$. 
\end{proof}

\noindent
Moreover, landmark orderings also carry over from the all-outcomes determinization to the probabilistic domain.
\begin{theorem}
    Let $\prob = (\Sigma, s_0, g)$ be a probabilistic planning problem, where $\Sigma = (S, A, \gamma, \Pr)$. Let $\detvar{\prob} = (\detvar{\Sigma}, s_0, g)$ be the classical planning problem where $\detvar{\Sigma} = (\detvar{S}, \detvar{A}, \detvar{\gamma})$ is the all-outcomes determinization of $\Sigma$. Let $\varphi_1$ and $\varphi_2$ be conditions. 
    \begin{enumerate}[leftmargin=.65cm,label=(\arabic*)]
        \item If $\varphi_1 \to \varphi_2$ in $\detvar{\prob}$, then $\varphi_1 \to \varphi_2$ in $\prob$.
        \item If $\varphi_1 \to_{\textnormal{n}} \varphi_2$ in $\detvar{\prob}$, then $\varphi_1 \to_{\textnormal{n}} \varphi_2$ in $\prob$.
        \item If $\varphi_1 \to_{\textnormal{gn}} \varphi_2$ in $\detvar{\prob}$, then $\varphi_1 \to_{\textnormal{gn}} \varphi_2$ in $\prob$.
    \end{enumerate}
\end{theorem}

\begin{proof}
    Proofs for (1), (2), and (3) are similar. To prove (1) (resp. (2); (3)), suppose $\varphi_1$ and $\varphi_2$ are landmarks in $\detvar{\prob}$, with $\varphi_1 \to \varphi_2$ (resp. $\varphi_1 \to_{\textnormal{n}} \varphi_2$; $\varphi_1 \to_{\textnormal{gn}} \varphi_2$). Let $\pi$ be a policy for $\prob$, let $\sigma \in H(s_0, \pi, g)$, and suppose $\varphi_2$ is true (resp. added; first added) at time $i$ in $\sigma$.  By Lemma~\ref{lem:1}, there exists a plan $\detvar{\pi}$ that reaches the same sequence of states as $\sigma$, so $\varphi_2$ is also true (resp. added; first added) at time $i$ in $\detvar{\pi}$. Then, $\varphi_1$ is true at some time $j<i$ (resp. time $i - 1$; time $i - 1$) in $\detvar{\pi}$. Thus, $\varphi_1$ is also true at time $j$ (resp. $i - 1$; $i - 1$) in $\sigma$. 
\end{proof}

\noindent Therefore, to generate landmarks for a probabilistic planning problem $\prob = (\Sigma, s_0, g)$, we can use existing classical landmark extraction algorithms, like $\mathrm{LM}^{\text{RHW}}$ \cite{richter2008landmarks}, to generate classical landmarks for the all-outcomes determinized planning problem $\detvar{\prob} = (\detvar{\Sigma}, s_0, g)$, and directly use those landmarks and their orderings in $\prob$. 

\subsection{Landmarks as subgoals}

Before we introduce \LAMP (\secref{sec:algorithm}), we will discuss the expected benefit that landmarks might provide while also considering some subtleties of using landmarks to solve MDPs that make this a nontrivial task. In classical and probabilistic planning, subgoals can be used to decompose large planning problems into smaller, easier-to-solve subproblems. A domain-independent approach to generating subgoals is to compute landmarks and use them as subgoals. We can decompose a planning task into subtasks, each with the goal of achieving one landmark, then combine their solutions, a procedure inspired by \citet{hoffmann2004ordered}. However, a straightforward sequential approach to achieving landmarks is not guaranteed to succeed, even if a solution exists for the original problem.  A key limitation is that achieving one landmark can lead to a state from which future landmarks, or the final goal itself, are unreachable. We will elaborate on this point later.

Let $\cpvar{\prob} = (\cpvar{\Sigma}, s_0, g)$ be a classical planning problem, let $\langle \varphi_1, \dots, \varphi_k\rangle$ be a sequence of landmarks, and suppose $\varphi_k \equiv g$. For each $i \in \{1, \dots, k\}$, suppose there exists a solution plan $\cpvar{\pi}^i$ for the problem $(\cpvar{\Sigma}, s_{i - 1}, \varphi_i)$, where $s_{i - 1} = \cpvar{\gamma}(s_0, \cpvar{\pi}^1 \circ \cdots \circ \cpvar{\pi}^{i - 1})$ for $i>1$. Then, a solution to $\cpvar{\prob}$ can be obtained by concatenating $\cpvar{\pi}^1 \circ \dots \circ \cpvar{\pi}^k$. Note that this solution depends on the choice of sequence $\langle \varphi_1, \dots, \varphi_k\rangle$ as well as the plans $\cpvar{\pi}^i$.

A similar approach can be applied to probabilistic planning problems. To illustrate, let $\prob = (\Sigma, s_0, g)$ be a probabilistic planning problem and let $\langle \varphi_1, \dots, \varphi_k\rangle$ be a sequence of landmarks where $\varphi_k \equiv g$. For all $i \in \{1, \dots, k\}$, suppose there exists a safe policy $\pi_i$ which achieves $\varphi_i$ from all states $s_{i - 1} \models \varphi_{i - 1}$ reachable by running $\pi_{i - 1}$, or from $s_0$, if $i = 1$. We can obtain a solution to $\prob$ by running policies $\pi_1, \dots, \pi_k$ in sequence (see Appendix \ref{appendix:sequential} for more details). This procedure again depends on the choice of sequence $\langle \varphi_1, \dots, \varphi_k\rangle$ and policies $\pi_i$.

In the ideal best-case, the landmark-based approach could yield up to an exponential speed-up in plan time. Consider an unbounded search space with branching factor $b$ and a goal $g$ at depth $d$ from the initial state $s_0$. A breadth-first planner would take $\mathcal{O}(b^d)$ time to reach the goal. However, suppose there exists a sequence of landmarks $\langle \varphi_1, \dots, \varphi_k\rangle$ where $\varphi_k \equiv g$, $\varphi_1$ is at depth $d/k$ from $s_0$, and for all $i \in \{1, \dots, k - 1\}$, if $s_{i} \models \varphi_{i}$, then $\varphi_{i + 1}$ is at depth $d / k$ from $s_{i}$. In this best case, a breadth-first planner takes $\mathcal{O}(b^{d/k})$ to achieve each landmark. With $k$ landmarks, this reduces the total planning time to $\mathcal{O}(kb^{d/k})$, an exponential speed-up (see Appendix \ref{appendix:sequential} Figure \ref{fig:speedup}). 

However, the existence of such policies for each landmark is not guaranteed. This limitation underpins the incompleteness of this sequential landmark achievement approach: even when a safe policy for the original planning problem exists, this approach might fail to reach the top-level goal $g$. These failures can occur in two situations: the current, chosen landmark is unreachable, or a deadlock state is encountered from which the final goal is unreachable \cite{hoffmann2004ordered} (see Appendix \ref{appendix:sequential} Figure \ref{fig:locks} for an example). 

These failures stem from the approach's inherent myopia: by focusing on achieving landmarks, the planner can be led to states that hinder or prevent further progress towards future landmarks or the final goal. In the classical setting, \citet{hoffmann2004ordered} proposed a ``safety net'' to address the issue of unreachable landmarks: in the event that an unreachable landmark is selected, the planner reverts to a base planner that seeks to achieve the final goal, disregarding the landmark graph. In the probabilistic setting, we aim to mitigate these failure modes with an approach that balances the pursuit of landmarks with the pursuit of the final goal by tuning a greediness parameter (\secref{sec:choosing_actions}). 

\section{Landmark-assisted Monte Carlo planning}
\label{sec:algorithm}

Intuitively, using landmarks during MDP planning is straightforward. Algorithm~\ref{alg:landmarkplan} shows a simple procedure, called \LandmarkPlan, that first computes a collection of landmarks $\Phi$ (Line 2) and then iteratively selects a landmark (Lines 5--7) and actions to achieve that landmark (Lines 8--10).  The behavior of \LandmarkPlan depends largely on how selection on Lines 7 and 9 are implemented. 

Algorithm \ref{alg:lamp} shows \LAMP, the main algorithm of this paper, which is a specific implementation of \LandmarkPlan based on Monte Carlo tree search. Like \LandmarkPlan, it has stages for selecting landmarks (Lines 7--9) and actions (Lines 10--12).  \LAMP learns $Q$-functions (Lines 5--6) to predict the best landmark ($Q_{LM}$), the best action for the current landmark ($Q_{\varphi}$), and the best action for the final goal ($Q_g$). Crucially, \LAMP uses a greediness parameter $\alpha$ to achieve a balance between greedy actions to achieve the current landmark with actions that advance toward the top-level goal (Line 11). If a dead-end is encountered, the algorithm fails. 

Section~\ref{sec:results} will demonstrate that \LAMP performs well in several benchmark problems and that the best $\alpha$ is problem dependent; the results present strong evidence that landmarks can be helpful for solving MDPs. More generally, landmarks can be incorporated into MDP algorithms in many ways; \LAMP is one instance in a family based on \LandmarkPlan which uses UCT \cite{kocsis2006bandit}, a simple, well-established planning algorithm that allows us to easily modulate the amount of work done to learn a policy and examine the effect of using landmarks. Future work should explore other variations of \LandmarkPlan, including non-sampling-based approaches.

Next, we describe the key components of \LAMP in determining: how to choose a landmark (\secref{sec:lamp-choose-landmark}), how to choose an action (\secref{sec:choosing_actions}), and how to learn the $Q$-functions (\secref{sec:lamp-uct}).

\begin{algorithm}[t]
    \caption{A general procedure for landmark-guided planning. }
    \label{alg:landmarkplan}
    \begin{algorithmic}[1]
\Procedure{LandmarkPlan}{$P$}
   \State $\Phi \gets \mathsc{LandmarkGraph}(\prob) \cup \{g\}$
   \State $s \gets s_0$; $\varphi \gets \mathsc{nil}$
   \While{$\Phi \neq \varnothing$ and $s \not\models g$}
      \If{$\varphi = \mathsc{nil}$ or $s \models \varphi$}  $\>$ \Comment{select landmark}
         \State $\Phi \gets \Phi \backslash \{\varphi\}$
         \State select $\varphi \in \mathrm{leaves}(\Phi)$
      \Else  \Comment{select action}
         \State select $a \in \mathrm{Applicable}(s)$ to achieve $\varphi$
         \State $s \gets \mathsc{apply}(s, a)$
      \EndIf
   \EndWhile
\EndProcedure
\end{algorithmic}
\end{algorithm}
\begin{algorithm}[t]
    \caption{The \LAMP algorithm, a particular implementation of \LandmarkPlan. $\prob = (\Sigma, s_0, g)$ is a planning problem, $\var{n\_rollouts}$ is the number of rollouts, $\var{budget}$ is a cost limit, $\var{depth}$ is the maximum rollout depth, and $\alpha$ is the greediness parameter. }
    \label{alg:lamp}
    \begin{algorithmic}[1]
   \Procedure{LAMP}{$\prob, \var{n\_rollouts}, \var{budget}, \var{depth}, \alpha$}
      \State $\Phi \gets \mathsc{LandmarkGraph}(\prob) \cup \{g\}$
      \State $\varphi \gets \mathsc{nil}$; $s \gets s_0$; $\var{cost} \gets 0$
      \While{$\Phi \neq \varnothing$ and $s \not\models g$ and $\var{cost} < \var{budget}$}
      \For{$i \gets 1, \dots, \var{n\_rollouts}$}  \Comment{learn $Q_{LM}$,  $Q_{\varphi}$, $Q_g$}
         \State $\mathsc{Rollout}(s, \varphi, \Phi, \var{depth}, \alpha, cost)$
      \EndFor
         \If{$\varphi = \mathsc{nil}$ or $s \models \varphi$} \Comment{select landmark}
            \State $\Phi \gets \Phi \backslash \{\varphi\}$
            \State $\displaystyle{\varphi \gets \argmax_{\varphi' \in \mathrm{leaves}(\Phi)} (Q_{LM}(\Phi, \varphi'))}$ 
         \Else \Comment{select action}
            \State $\displaystyle{a \gets \argmax_{a \in \mathrm{Applicable}(s)}(\alpha Q_\varphi(s, a) +(1 - \alpha)Q_g(s, a))}$
            \State $s \gets \mathsc{apply}(s, a)$
         \EndIf
         \State $\var{cost} \gets \var{cost} + 1$
      \EndWhile
\EndProcedure
\end{algorithmic}

\end{algorithm}

\subsection{Choosing landmarks in \LAMP}
\label{sec:lamp-choose-landmark}

To learn an ordering of the landmarks, we frame the selection of landmarks as a planning problem. Let $\prob = (\Sigma, s_0, g)$ be a probabilistic planning problem and let $\Phi$ be a collection of landmarks for $\prob$. Let $\Sigma_\Phi$ be a planning domain with state space $S_\Phi = 2^\Phi$, action space $A_\Phi = \Phi$, and state-transition function $\gamma_\Phi: (s, a) \mapsto s \backslash \{a\}$ if $a \in s$. Actions in this domain are deterministic, meaning $\Pr(s' \mid s, a) = \mathbf{1}_{\gamma_\Phi(s, a)}(s')$. However, these actions do not have unit cost. Instead, the cost of executing $a \in A_\Phi$ is determined by the cost of running policy $\pi_a$ in $\Sigma$ to achieve landmark $a$. This cost depends not only on the action $a \in A_\Phi$ and the current state $s \in S_\Phi$, but also the state in $\Sigma$ from which $\pi_a$ is initiated, meaning that the cost also depends on the history of previously executed actions. 

To further restrict the set of applicable actions in $\Sigma_\Phi$, we can impose a strict partial order $\prec$ on $\Phi$ using landmark orderings. Then, for any $\Phi' \subseteq \Phi$, we can view $\Phi'$ as a directed acyclic graph and define 
\begin{equation}
    \mathrm{leaves}(\Phi') = \{\varphi \in \Phi' \mid \varphi \text{ has no $\prec$-predecessors}\}.
\end{equation}
We can then restrict $\mathrm{Applicable}(s) = \mathrm{leaves}(s)$. We will also assume $g \in \Phi$, with $\varphi \prec g$ for all $\varphi \in \Phi \backslash \{g\}$ to ensure the top-level goal $g$ is achieved. Then, landmark selection can be viewed as a planning problem $\prob_\Phi = (\Sigma_\Phi, \Phi, g_\Phi)$, where $s \models g_\Phi$ iff $s = \varnothing$. 

\LAMP's landmark selection bears some similarity to the options framework in hierarchical reinforcement learning \cite{sutton1999between}. That is, $\Sigma_\Phi$ is semi-Markov. To see why, consider an option $( \mathcal{I}, \dot{\pi}, \beta )$ which consists of an initiation set $\mathcal{I} \subseteq S$, a policy $\dot{\pi} : S \times A \to [0, 1]$, and a termination condition $\beta: S \to [0, 1]$. A total, safe policy $\pi_a$ for landmark $a \in A_\Phi$ can be viewed as an option $\langle\mathcal{I}, \dot{\pi}, \beta\rangle_a$ with initiation set $\mathcal{I} = S$, policy $\dot{\pi} = \mathbf{1}_{[\pi_a(s') = a']}(s', a')$, and termination condition $\beta = \mathbf{1}_{[s' \models a]}(s')$.  Since options are known to be semi-Markov by \citet{sutton1999between}, we know that $\Sigma_\Phi$ is semi-Markov, meaning that landmarks offer a similar benefit to options in providing temporal abstractions to decompose a problem. 

\subsection{Choosing actions in \LAMP}
\label{sec:choosing_actions}

To solve $\prob$ using the approach outlined in Algorithm~\ref{alg:landmarkplan}, we must solve two planning problems concurrently: the landmark selection problem in $\Sigma_\Phi$ (Algorithm~\ref{alg:lamp} Lines~7--9) and the action selection problem in $\Sigma$ (Algorithm~\ref{alg:lamp} Lines~10--12). For our implementation of \LAMP, we use UCT for both landmark selection and action selection, but we note that other techniques can be applied. 

A greedy approach is one in which selection on Line 11 is done solely with respect to the current landmark $\varphi$, without regard for the top-level goal $g$. While this reduces the complexity of the planning task, it may compromise optimality. We can see this demonstrated in the example in Figure~\ref{fig:greedy}. Suppose we want to achieve landmark $\varphi$ with top-level goal $g$. Assuming actions are deterministic, we can achieve $\varphi$ by either executing $\langle a_1, a_2\rangle$ or $\langle a_4\rangle$. Although $\langle a_1, a_2\rangle$ leads to a shorter solution for the top-level goal, the greedy planner favors $\langle a_4\rangle$ because it achieves the landmark $\varphi$ more efficiently.

To address this, we adopt a hybrid approach in \LAMP that balances rewards for both $\varphi$ and $g$. With a greediness parameter $\alpha \in [0, 1]$, suppose $Q_g$ and $Q_\varphi$ are action value functions with respect to $g$ and $\varphi$, respectively. Rather than using the greedy approach $\argmax_{a\in \mathrm{Applicable}(s)} Q_\varphi(s, a)$, the balanced approach uses 
\begin{equation}
    \label{eqn:greediness}
    \argmax_{a \in \mathrm{Applicable}(s)} \left(\alpha Q_\varphi(s, a) + (1 - \alpha)Q_g(s, a)\right)
\end{equation}
(cf. Algorithm~\ref{alg:lamp} Line~11). In Figure~\ref{fig:greedy}, suppose $Q_g(s_0, a_1) = \frac{1}{4}$, $Q_g(s_0, a_4) = \frac{1}{8}$, $Q_\varphi(s_0, a_1) = \frac{1}{2}$, $Q_\varphi(s_0, a_4) = 1$. So $\alpha>\frac{1}{5}$ selects action $a_4$, while $\alpha < \frac{1}{5}$ selects action $a_1$. Random selection is used to break ties. 

\begin{figure}[t]
    \centering \begin{tikzpicture}[scale=1.5]
    \node[circle, draw=black,minimum size=0.6cm,inner sep=0pt] (r1) at (0, 0) {\normalsize{$s_0$}};
    \node[circle, draw=black,minimum size=0.6cm,inner sep=0pt] (s1) at (1, .5) {\normalsize{$s_1$}};
    \node[circle, draw=black,minimum size=0.6cm,inner sep=0pt] (s2) at (2, .5) {\normalsize{$s_2$}};
    \node[circle, draw=black,minimum size=0.6cm,inner sep=0pt] (s3) at (.75, -.5) {\normalsize{$s_3$}};
    \node[circle, draw=black,minimum size=0.6cm,inner sep=0pt] (s4) at (1.5, -.5) {\normalsize{$s_4$}};
    \node[circle, draw=black,minimum size=0.6cm,inner sep=0pt] (s5) at (2.25, -.5) {\normalsize{$s_5$}};
    \node[circle, draw=black,minimum size=0.6cm,inner sep=0pt] (g1) at (3, 0) {\normalsize{$s_g$}};
    \draw [-latex] (r1) to node[above] {\normalsize{$a_1$}} (s1) ;
    \draw [-latex] (r1) to node[below, xshift=-.25em, yshift=.125ex] {\normalsize{$a_4$}} (s3);
    \draw [-latex] (s1) to node[above] {\normalsize{$a_2$}} (s2);
    \draw [-latex] (s3) to node[below] {\normalsize{$a_5$}} (s4);
    \draw [-latex] (s2) to node[above] {\normalsize{$a_3$}} (g1);
    \draw [-latex] (s4) to node[below] {\normalsize{$a_6$}} (s5);
    \draw [-latex] (s5) to node[below] {\normalsize{$a_7$}} (g1);
    \node[draw=black,dashed,inner sep=5pt,rounded corners=.3cm,rotate fit=38.66,fit=(s2)(s3)] {};
    \node at (s3) [xshift=-.5cm, yshift=-.5cm] {$\varphi$};
\end{tikzpicture}
    \caption{Suppose $\varphi$ is a landmark such that $s_2 \models \varphi$ and $s_3 \models \varphi$. Starting at $s_0$, a greedy algorithm may achieve $\varphi$ by performing action $a_4$, leading to a worse solution.}
    \label{fig:greedy}
\end{figure}
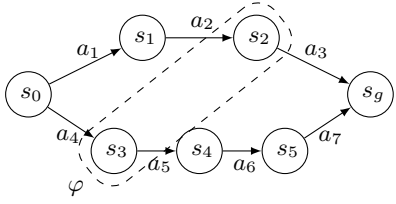

\subsection{Learning $Q$-functions for \LAMP}
\label{sec:lamp-uct}

To learn these $Q$-values, \LAMP uses an approach based on Monte Carlo tree search (MCTS). MCTS is a general approach to Markov decision processes that has garnered considerable attention after its noteworthy success at computer Go \cite{silver2016mastering}. It is often applied in domains with a large branching factor where it may be difficult or impossible to explore the entire search tree. MCTS combines a tree search with Monte Carlo rollouts and uses the outcome of these rollouts to evaluate states in a search tree. 

MCTS iteratively builds a search tree using rollouts from the current state. Within the family of MCTS algorithms, Upper Confidence Bound applied to Trees (UCT) is the most widely used in practice \cite{kocsis2006bandit}. It uses a solution to the multi-armed bandit problem \cite{auer2002finite} that provides a principled balance between exploration and exploitation when selecting nodes to expand. 

UCT planning traditionally uses reward signals from rollouts to iteratively update an action value function $Q$, where $Q(s, a)$ is an approximation of the expected reward of executing action $a$ in state $s$. To tailor UCT for our setting of goal-directed MDPs, we use GUBS (Goals with Utility-Based Semantic), a criterion for solving SSPs with dead-ends which provides a principled tradeoff between maximizing the probability of reaching the goal and minimizing expected cost with a utility-based model \cite{freire2017gubs, freire2019exact}. The utility function $U$ captures a smooth tradeoff between goal achievement and cost, where the utility of a history $\sigma$ is 
\begin{equation}
    U(\sigma) = u(|\sigma|) + K_g \mathbf{1}_{[\sigma_{-1} \models g]},
\end{equation}
$u$ is a strictly decreasing utility function over cost, and $K_g$ is a constant utility for reaching the goal $g$. 

By using UCT with the GUBS criterion, $Q$-values estimate expected utility as specified by the GUBS utility function \cite{crispino2024uctgubs}. This allows it to maintain the exploration-exploitation balance of UCT while properly handling scenarios with unavoidable dead ends---cases where conventional expected cost minimization breaks down.

In \LAMP, we further adapt UCT to learn action values for not only the top-level goal $g$, but also each landmark $\varphi \in \Phi$, as well as to learn as landmark-values. 
\begin{itemize}
    \item $Q_g(s, a)$ estimates the expected utility of executing action $a$ in state $s$ for top-level goal $g$
    \item $Q_\varphi(s, a)$ estimates the expected utility of executing action $a$ in state $s$ for landmark $\varphi$
    \item $Q_{LM}(\Phi, \varphi)$ estimates the expected utility of first achieving landmark $\varphi \in \Phi$ for top-level goal $g$.
\end{itemize}
We also maintain corresponding values for $N_g$, $N_\varphi$, and $N_{LM}$, where $N(s, a)$ is the number of times $a$ was selected in state $s$, and $N(s)$ is the number of times $s$ was visited. 

\begin{algorithm}[t]
    \caption{The Monte Carlo rollout procedure for \LAMP. $s$ is the current state, $\varphi$ is the current landmark, $\Phi$ is the current landmark graph, $d$ is the remaining rollout depth, $\alpha$ is the greediness parameter,  $c$ is the cumulative cost, $u$ is a utility function over cost, and $K_g$ is the goal-utility constant. $Q_{LM}$, $N_{LM}$, $Q_{\varphi}$, $N_{\varphi}$, $Q_{g}$, and $N_{g}$ are maps with default value 0. \textsc{Rollout} returns a tuple $(\lambda_\varphi, \beta_\varphi, \lambda_g, \beta_g)$, where $\lambda_\varphi$ and $\lambda_g$ are rollout costs for $\varphi$ and $g$, respectively, and $\beta_\varphi$ and $\beta_g$ are booleans indicating the achievement of $\varphi$ and $g$, respectively. The standard UCT rollout procedure is shown in gray and modifications or additions are emphasized in black. }
    \label{alg:rollout}
    \begin{algorithmic}[1]
\PROCEDURE{Rollout}{$s, \textcolor{black}{\varphi}, \textcolor{black}{\Phi}, d, \textcolor{black}{\alpha}, c$}
   \IIF{\textcolor{black}{$\Phi = \varnothing$}} \Return $(\textcolor{black}{0, \mathsf{true}}, 0, \mathsf{true})$ 
   \ENDIIF  
   \If{$\varphi =  \mathsc{nil}$ or $s \models \varphi$}
      \State $\varphi' \gets \underset{\varphi' \in \mathrm{leaves}(\Phi) \backslash \{\varphi\}}{\mathrm{arg\;max}}\mathsc{UCB1}(Q_{LM}, N_{LM}, \Phi, \varphi')$
      \State $(\lambda_\varphi, \beta_\varphi, \lambda_g, \beta_g) \gets \mathsc{Rollout}(s, \varphi', \Phi \backslash \{\varphi\}, d, \alpha, c)$
      \State $\mathsc{UCB-Update}(Q_{LM}, N_{LM}, \Phi, \varphi, \lambda_g + c, \beta_g)$
      \State \Return $(0, \mathsf{true}, \lambda_g, \beta_g)$
   \EndIf
   \IIF{$d = 0$ or $\mathrm{Applicable}(s) = \varnothing$} \Return $\left(\textcolor{black}{d, \mathsf{false}}, d, \mathsf{false}\right)$ 
   \ENDIIF

   \STATE{\(\begin{multlined}a \gets \underset{a \in \mathrm{Applicable}(s)}{\mathrm{arg\;max}}(\textcolor{black}{\alpha\mathsc{UCB1}(Q_\varphi, N_\varphi, s, a) +}\\[-2ex]\hspace{3.365cm} \textcolor{black}{(1 - \alpha)}\mathsc{UCB1}(Q_g, N_g, s, a))\end{multlined}\)}
   \STATE{$s' \gets \mathsc{simulate}(s, a)$}
   \STATE{$(\textcolor{black}{\lambda_\varphi, \beta_\varphi}, \lambda_g, \beta_g) \gets \mathsc{Rollout}(s', \textcolor{black}{\varphi}, \textcolor{black}{\Phi}, d - 1, \textcolor{black}{\alpha}, c + 1)$}
   \STATE{\textcolor{black}{$\lambda_\varphi \gets \lambda_\varphi + 1$; }$\lambda_g \gets \lambda_g + 1$}
   \State $\mathsc{UCB-Update}(Q_\varphi, N_\varphi, s, a, \lambda_\varphi + c, \beta_\varphi)$
   \STATE{$\mathsc{UCB-Update}(Q_g, N_g, s, a, \lambda_g + c, \beta_g)$}
   \STATE{\Return $(\textcolor{black}{\lambda_\varphi, \beta_\varphi,} \lambda_g, \beta_g)$}
\ENDPROCEDURE

\PROCEDURE{UCB-Update}{$Q, N, s, a, \lambda, \beta$}
    \STATE{$Q(s, a) \gets \frac{N(s, a)Q(s, a) + u(\lambda) + K_g\mathbf{1}_{[\beta]}}{1 + N(s, a)}$}
    \STATE{$N(s) \gets N(s) + 1$; $N(s, a) \gets N(s, a) + 1$}
\ENDPROCEDURE
\end{algorithmic}

\end{algorithm}

We modified the UCT rollouts to learn these $Q$-values simultaneously (Algorithm~\ref{alg:rollout}). In standard UCT, during each rollout, actions are selected using UCB1 values for each state-action pair, where 
\begin{equation}
   \textstyle{ \mathrm{UCB1}(Q, N, s, a) = Q(s, a) + C\sqrt{\frac{\log(N(s))}{N(s, a)}}}
\end{equation}
and $C$ is the exploration constant. In state $s$, the action that yields the highest $\mathrm{UCB1}$-value is selected. However, in Algorithm~\ref{alg:rollout} Line~9, we adjusted the algorithm to account for the greediness parameter $\alpha$; actions are selected using a linear combination of UCB1 values corresponding to both the top-level goal and the current landmark (cf. Equation~\ref{eqn:greediness}).  

During each rollout, four values are backpropagated through the search tree: the rollout cost for the current landmark $\varphi$, the rollout cost for the top-level goal $g$, and boolean values to indicate whether $\varphi$ and $g$ were actually achieved. The subgoal cost is used to update $Q_\varphi$, while the top-level cost is used to update $Q_g$ and $Q_{LM}$. Whenever a landmark $\varphi$ is achieved, it is removed from the landmark graph $\Phi$, and a new landmark $\varphi' \in \mathrm{leaves}(\Phi)$ is selected using UCB1 values calculated from $Q_{LM}$. The rollout then continues with the new subgoal, and a subgoal cost of 0 is backpropagated. Whenever the top-level goal is reached, the rollout terminates and both a top-level cost and subgoal cost of 0 are backpropagated. 

We additionally imposed a maximum depth for rollouts. If a terminal state is not reached after a fixed number of actions, the rollout is halted. During rollouts, action pruning, as implemented in PROST \cite{keller2012prost}, is used to reduce the branching factor by removing actions with the same distribution over successor states as another action. 

We emphasize that these modifications to UCT only add minimal computational overhead. \LAMP adds decision nodes to the Monte Carlo search tree whenever a landmark needs to be selected, which means that during each rollout, the number of nodes visited increases at most by the number of landmarks in $\Phi$. Assuming the number of landmarks is small, i.e. $|\Phi| \ll |S|$, the runtime for a rollout does not significantly increase; each rollout incurs at most $\mathcal{O}(|\Phi|)$ overhead. If no landmarks are present, a single trivial subgoal corresponding to the final goal $g$ is used, and LAMP is equivalent to standard UCT.

\section{Experiments using \LAMP}
\label{sec:results}

We evaluated \LAMP (Algorithm~\ref{alg:lamp}) on a set of benchmark probabilistic planning problems.  For our experiments, we set an execution budget of 200 action executions, a maximum rollout depth of 20, and an exploration constant of $\sqrt{2}$. We used a non-heuristic implementation of UCT; all $Q$ and $N$ values were initialized to 0. For the GUBS criterion, the goal utility constant $K_g$ was set to 1, and the utility function $u: \var{cost} \mapsto \exp(-\frac{1}{10} \var{cost})$ was used, mirroring the default parameter settings of \cite{crispino2024uctgubs}. 

We also used Fast-Downward's implementation of the $\mathrm{LM}^{\text{RHW}}$ landmark extraction algorithm \cite{richter2008landmarks} to compute a landmark graph, including all landmark orderings. Landmarks trivially satisfied in the initial state were pruned from the landmark graph, and the top-level goal was added to the landmark graph, ordered such that all other landmarks are predecessors of the top-level goal. 

In the experiments that follow, we varied the number of rollouts $\var{n\_rollouts} \in \{5, 10, 20, 50, 100, 200, 500, 1000, 2000, 5000\}$ and the greediness parameter $\alpha \in \{0, 0.2, 0.5, 0.8, 1\}$. When $\alpha = 0$, \LAMP is equivalent to standard UCT and serves as the baseline. The code for our implementation has been made available \cite{CODE}.

\subsection{Early benchmarks}

\begin{figure*}
\begin{minipage}{0.33\textwidth}
\raggedleft
\colorlet{maxcolor}{white}
\colorlet{mincolor}{blue!50}

\newcolumntype{C}[1]{>{\raggedleft\arraybackslash}m{#1}}

\renewcommand\gr[1]{\gradient{#1}{16.12333333333333}{103.41333333333333}{190.69333333333333}{maxcolor}{mincolor}}

\setlength{\tabcolsep}{2pt}
\setlength{\colwidth}{0.7cm}

\begin{tabular}{rr|C{\colwidth}C{\colwidth}C{\colwidth}C{\colwidth}C{\colwidth}|}
  & \multicolumn{1}{c}{} & \multicolumn{5}{c}{\begin{tabular}{@{}c@{}}less\\[-1ex]greedy\end{tabular} $\longleftarrow \alpha \longrightarrow$ \begin{tabular}{@{}c@{}}more\\[-1ex]greedy\end{tabular}}\\
    & \multicolumn{1}{c}{} & \multicolumn{1}{c}{0\tiny{/UCT}} & \multicolumn{1}{c}{0.2} & \multicolumn{1}{c}{0.5} & \multicolumn{1}{c}{0.8} & \multicolumn{1}{c}{1}\\
    \hhline{~~|-----}
  \parbox[t]{3mm}{\multirow{10}{*}{\rotatebox[origin=c]{90}{number of rollouts}}} & 5 &\gr{190.69} 190.7 & \gr{154.81} \textbf{154.8} & \gr{144.08} \underline{\textbf{144.1}} & \gr{150.72} \textbf{150.7} & \gr{163.00} \textbf{163.0}\\
& 10\ & \gr{179.31} 179.3 & \gr{113.65} \underline{\textbf{113.7}} & \gr{123.05} \textbf{123.1} & \gr{135.75} \textbf{135.7} & \gr{122.07} \textbf{122.1}\\
& 20\ & \gr{168.85} 168.9 & \gr{106.31} \underline{\textbf{106.3}} & \gr{108.32} \textbf{108.3} & \gr{109.60} \textbf{109.6} & \gr{133.25} \textbf{133.3}\\
& 50\ & \gr{148.96} 149.0 & \gr{62.80} \underline{\textbf{62.8}} & \gr{75.36} \textbf{75.4} & \gr{72.47} \textbf{72.5} & \gr{89.15} \textbf{89.1}\\
& 100\ & \gr{113.41} 113.4 & \gr{42.07} \underline{\textbf{42.1}} & \gr{47.49} \textbf{47.5} & \gr{48.65} \textbf{48.7} & \gr{72.01} \textbf{72.0}\\
& 200\ & \gr{69.55} 69.5 & \gr{28.75} \textbf{28.7} & \gr{27.93} \underline{\textbf{27.9}} & \gr{36.31} \textbf{36.3} & \gr{40.93} \textbf{40.9}\\
& 500\ & \gr{37.41} 37.4 & \gr{20.68} \underline{\textbf{20.7}} & \gr{22.45} \textbf{22.5} & \gr{23.17} \textbf{23.2} & \gr{24.45} \textbf{24.5}\\
& 1000\ & \gr{27.55} 27.5 & \gr{18.37} \underline{\textbf{18.4}} & \gr{19.56} \textbf{19.6} & \gr{21.16} \textbf{21.2} & \gr{22.16} \textbf{22.2}\\
& 2000\ & \gr{20.83} 20.8 & \gr{17.81} \underline{\textbf{17.8}} & \gr{18.47} \textbf{18.5} & \gr{18.09} \textbf{18.1} & \gr{19.44} 19.4\\
& 5000\ & \gr{19.83} 19.8 & \gr{16.13} \underline{\textbf{16.1}} & \gr{16.45} \textbf{16.5} & \gr{17.63} 17.6 & \gr{18.41} 18.4\\
  \hhline{~~|-----}
\end{tabular}\quad\mbox{}
\begin{center}
    \captsize \hspace{.8cm} \textbf{(a)} \domain{prob\_blocksworld} problem \texttt{p5}
\end{center}
\end{minipage}
\hfill
\begin{minipage}{0.33\textwidth}
\raggedleft
\colorlet{maxcolor}{white}
\colorlet{mincolor}{blue!50}

\newcolumntype{C}[1]{>{\raggedleft\arraybackslash}m{#1}}

\renewcommand\gr[1]{\gradient{#1}{13.736666666666666}{96.69333333333333}{179.64}{maxcolor}{mincolor}}

\setlength{\tabcolsep}{2pt}
\setlength{\colwidth}{0.7cm}

\begin{tabular}{rr|C{\colwidth}C{\colwidth}C{\colwidth}C{\colwidth}C{\colwidth}|}
  & \multicolumn{1}{c}{} & \multicolumn{5}{c}{\begin{tabular}{@{}c@{}}less\\[-1ex]greedy\end{tabular} $\longleftarrow \alpha \longrightarrow$ \begin{tabular}{@{}c@{}}more\\[-1ex]greedy\end{tabular}}\\
    & \multicolumn{1}{c}{} & \multicolumn{1}{c}{0\tiny{/UCT}} & \multicolumn{1}{c}{0.2} & \multicolumn{1}{c}{0.5} & \multicolumn{1}{c}{0.8} & \multicolumn{1}{c}{1}\\
    \hhline{~~|-----}
  \parbox[t]{3mm}{\multirow{10}{*}{\rotatebox[origin=c]{90}{number of rollouts}}} & 5 &\gr{179.64} 179.6 & \gr{60.72} \underline{\textbf{60.7}} & \gr{90.20} \textbf{90.2} & \gr{77.09} \textbf{77.1} & \gr{77.16} \textbf{77.2}\\
& 10\ & \gr{175.69} 175.7 & \gr{23.03} \underline{\textbf{23.0}} & \gr{32.24} \textbf{32.2} & \gr{29.61} \textbf{29.6} & \gr{25.68} \textbf{25.7}\\
& 20\ & \gr{159.33} 159.3 & \gr{27.16} \textbf{27.2} & \gr{27.81} \textbf{27.8} & \gr{23.63} \underline{\textbf{23.6}} & \gr{23.63} \underline{\textbf{23.6}}\\
& 50\ & \gr{105.03} 105.0 & \gr{22.21} \textbf{22.2} & \gr{22.69} \textbf{22.7} & \gr{16.39} \underline{\textbf{16.4}} & \gr{27.29} \textbf{27.3}\\
& 100\ & \gr{73.99} 74.0 & \gr{18.63} \textbf{18.6} & \gr{21.93} \textbf{21.9} & \gr{16.85} \underline{\textbf{16.9}} & \gr{17.67} \textbf{17.7}\\
& 200\ & \gr{42.13} 42.1 & \gr{17.45} \textbf{17.5} & \gr{17.53} \textbf{17.5} & \gr{15.55} \underline{\textbf{15.5}} & \gr{16.29} \textbf{16.3}\\
& 500\ & \gr{24.84} 24.8 & \gr{16.65} \textbf{16.7} & \gr{16.31} \textbf{16.3} & \gr{14.40} \underline{\textbf{14.4}} & \gr{14.92} \textbf{14.9}\\
& 1000\ & \gr{18.21} 18.2 & \gr{17.71} 17.7 & \gr{17.63} 17.6 & \gr{16.88} \underline{16.9} & \gr{16.92} 16.9\\
& 2000\ & \gr{16.24} 16.2 & \gr{15.31} 15.3 & \gr{14.67} \textbf{14.7} & \gr{14.28} \underline{\textbf{14.3}} & \gr{14.69} \textbf{14.7}\\
& 5000\ & \gr{15.19} 15.2 & \gr{15.16} 15.2 & \gr{13.83} \textbf{13.8} & \gr{13.75} \underline{\textbf{13.7}} & \gr{14.20} \textbf{14.2}\\
  \hhline{~~|-----}
\end{tabular}\quad\mbox{}
\begin{center}
    \captsize \hspace{.9cm} \textbf{(b)} \domain{elevators} problem \texttt{p5}
\end{center}
\end{minipage}
\hfill
\begin{minipage}{0.33\textwidth}
\raggedleft
\colorlet{maxcolor}{white}
\colorlet{mincolor}{blue!50}
\newcolumntype{C}[1]{>{\raggedleft\arraybackslash}m{#1}}
\renewcommand\gr[1]{\gradient{#1}{187.6}{199}{200.0}{maxcolor}{mincolor}}
\setlength{\tabcolsep}{2pt}
\setlength{\colwidth}{0.7cm}

\begin{tabular}{rr|C{\colwidth}C{\colwidth}C{\colwidth}C{\colwidth}C{\colwidth}|}
  & \multicolumn{1}{c}{} & \multicolumn{5}{c}{\begin{tabular}{@{}c@{}}less\\[-1ex]greedy\end{tabular} $\longleftarrow \alpha \longrightarrow$ \begin{tabular}{@{}c@{}}more\\[-1ex]greedy\end{tabular}}\\
    & \multicolumn{1}{c}{} & \multicolumn{1}{c}{0\tiny{/UCT}} & \multicolumn{1}{c}{0.2} & \multicolumn{1}{c}{0.5} & \multicolumn{1}{c}{0.8} & \multicolumn{1}{c}{1}\\
    \hhline{~~|-----}
  \parbox[t]{3mm}{\multirow{10}{*}{\rotatebox[origin=c]{90}{number of rollouts}}} & 5 &\gr{200.0} 200.0 & \gr{200.0} 200.0 & \gr{200.0} 200.0 & \gr{200.0} 200.0 & \gr{199.8} \underline{199.8}\\
  & 10\ & \gr{200.0} 200.0 & \gr{200.0} 200.0 & \gr{199.9} 199.9 & \gr{200.0} 200.0 & \gr{199.7} \underline{199.7}\\
  & 20\ & \gr{200.0} 200.0 & \gr{195.2} \underline{195.2} & \gr{199.2} 199.2 & \gr{200.0} 200.0 & \gr{198.4} 198.4\\
  & 50\ & \gr{200.0} 200.0 & \gr{199.3} 199.3 & \gr{197.2} \underline{197.2} & \gr{200.0} 200.0 & \gr{199.0} 199.0\\
  & 100\ & \gr{200.0} 200.0 & \gr{197.0} \underline{197.0} & \gr{198.3} 198.3 & \gr{198.6} 198.6 & \gr{199.5} 199.5\\
  & 200\ & \gr{200.0} 200.0 & \gr{198.4} 198.4 & \gr{197.6} \underline{197.6} & \gr{198.2} 198.2 & \gr{198.0} 198.0\\
  & 500\ & \gr{200.0} 200.0 & \gr{196.7} 196.7 & \gr{195.1} \underline{\textbf{195.1}} & \gr{197.3} 197.3 & \gr{198.5} 198.5\\
  & 1000\ & \gr{200.0} 200.0 & \gr{197.5} 197.5 & \gr{195.3} \underline{\textbf{195.3}} & \gr{196.8} 196.8 & \gr{197.1} 197.1\\
  & 2000\ & \gr{200.0} 200.0 & \gr{193.0} \textbf{193.0} & \gr{192.3} \textbf{192.3} & \gr{190.6} \textbf{190.6} & \gr{190.3} \underline{\textbf{190.3}}\\
  & 5000\ & \gr{200.0} 200.0 & \gr{189.3} \textbf{189.3} & \gr{187.6} \underline{\textbf{187.6}} & \gr{188.5} \textbf{188.5} & \gr{189.3} \textbf{189.3}\\
  \hhline{~~|-----}
\end{tabular}%
\quad\mbox{}
\begin{center}
    \captsize \hspace{.8cm} \textbf{(c)} \domain{zenotravel} problem \texttt{p5}
\end{center}
\end{minipage}
\caption{Average cost (lower is better) of the solutions generated by \LAMP. $\mathrm{LM}^{\text{RHW}}$ generated 10, 11, and 12 nontrivial landmarks for these problem instances, respectively. Underlined values indicate the best performing $\alpha$-value in the row. Boldfaced values indicate where \LAMP significantly ($p < 0.0125$) dominated standard UCT ($\alpha = 0$) at the same number of rollouts. Statistical analysis and results from other problem instances are in Appendix \ref{appendix:results_early}.  
}
\label{fig:early}
\end{figure*}

Our first set of experiments use domains and problem instances from the probabilistic track of the Fifth International Planning Competition, which are publicly available \cite{ippc2006}. 

\begin{itemize}
    \item \domain{prob\_blocksworld}: a probabilistic variation of the standard blocksworld where blocks have a chance of slipping from the gripper onto the table when being picked up, stacked, and unstacked. 
    \item \domain{elevators}: a person must navigate between floors and elevator shafts to collect coins while avoiding gates which may send the person back to the first floor. 
    \item \domain{zenotravel}: people must move between cities using airplanes. An airplane requires fuel to fly, and can be flown at two different speeds---the higher speed requiring more fuel. 
\end{itemize}

In Figure \ref{fig:early}, we selected one problem instance from each domain. For each problem, we reported the average cost of solutions generated by \LAMP, averaged over 75 runs. Runs that exceeded the execution budget of 200 were halted and contributed 200 to the average. Additional results from another 26 problem instances in these domains were qualitatively similar, and are presented in Appendix~\ref{appendix:results_early}. 

These results show that a greedy approach can significantly improve performance when the planner is constrained to fewer rollouts. Even though standard UCT ($\alpha = 0$) provably converges to an optimal solution \cite{kocsis2006bandit}, in domains like \domain{prob\_blocksworld}, \domain{elevators} and \domain{zenotravel}, a greedy approach performs significantly better than standard UCT because the number of rollouts is insufficient for the non-greedy approach to converge. For the \domain{zenotravel} problem, standard UCT was never able to reach a goal within the cost budget while the greedy approach could, even with 5000 rollouts. However, we also observe that the optimal greediness value $\alpha$ depends not only on the problem, but also the number of rollouts performed. A greediness value of $\alpha \approx 0.2$ yielded the lowest average cost for the \domain{prob\_blocksworld} problem, except in the cases of 5 and 200 rollouts where $\alpha = 0.5$ performed best. In contrast, $\alpha = 0.8$ yielded the lowest average cost for the \domain{elevators} problem, except in the cases of 5 and 10 rollouts. For the \domain{zenotravel} problem, the optimal greediness value varied between $0.2$ and $1$. 

These observed differences are often significant. We ran a $t$-test between \LAMP and baseline UCT ($\alpha = 0$) at the same number of rollouts for each greediness value $\alpha$. \LAMP statistically ($p < 0.0125$, the Bonferroni adjusted $\alpha_{\textit{stat}}=0.05/4$) dominates standard UCT in 4 of 10 rows for the \domain{zenotravel} problem, 9 of 10 rows for the \domain{elevators} problem, and all 10 rows for the \domain{prob\_blocksworld} problem (details in  Appendix~\ref{appendix:results_early}). In the complete set of results, \LAMP significantly outperformed standard UCT in 25 of 29 early benchmark problems, and in 172 of 290 total rows. 

\subsection{Probabilistically interesting benchmarks}
The \domain{prob\_blocksworld}, \domain{elevators} and \domain{zenotravel} domains do not contain deadlock states, i.e. states from which the top-level goal is not reachable, meaning some may find these problems to be ``\emph{probabilistically uninteresting}'', a term coined by \citet{little2007probabilistic}. There is, however, no free lunch. As we will see, in problems with deadlock states, a greedy approach together with a poorly placed landmark can quickly lead the planner to fail. In the results that follow, we identify problem instances from three deadlocking domains that demonstrate this to varying degrees:

\begin{itemize}
    \item \domain{exploding\_blocksworld} \cite{ippc2006}: a dead-end variation of standard blocksworld where blocks can detonate upon being put down. 
    \item \domain{tireworld} \cite{ippc2006}: a vehicle must navigate a road network to reach a goal. Every time the vehicle moves, it has a 40\% chance of getting a flat tire. Some locations have spare tires. 
    \item \domain{triangle\_tireworld} \cite{little2007probabilistic}: a variation of \texttt{tireworld} with a 50\% flat-tire rate and a triangular road network (Figure~\ref{fig:tireworld_triangle}). 
\end{itemize}

\begin{figure*}
\begin{minipage}{0.33\textwidth}
\raggedleft
\colorlet{maxcolor}{red!50}
\colorlet{mincolor}{white}

\newcolumntype{C}[1]{>{\raggedleft\arraybackslash}m{#1}}

\renewcommand\gr[1]{\gradient{#1}{0}{0.1}{0.6}{maxcolor}{mincolor}}

\setlength{\tabcolsep}{2pt}
\setlength{\colwidth}{0.7cm}

\begin{tabular}{rr|C{\colwidth}C{\colwidth}C{\colwidth}C{\colwidth}C{\colwidth}|}
  & \multicolumn{1}{c}{} & \multicolumn{5}{c}{\begin{tabular}{@{}c@{}}less\\[-1ex]greedy\end{tabular} $\longleftarrow \alpha \longrightarrow$ \begin{tabular}{@{}c@{}}more\\[-1ex]greedy\end{tabular}}\\
    & \multicolumn{1}{c}{} & \multicolumn{1}{c}{0\tiny{/UCT}} & \multicolumn{1}{c}{0.2} & \multicolumn{1}{c}{0.5} & \multicolumn{1}{c}{0.8} & \multicolumn{1}{c}{1}\\
    \hhline{~~|-----}
  \parbox[t]{3mm}{\multirow{10}{*}{\rotatebox[origin=c]{90}{number of rollouts}}} & 5 &\gr{0.01} 0.01 & \gr{0.11} \underline{0.11} & \gr{0.05} 0.05 & \gr{0.09} 0.09 & \gr{0.08} 0.08\\
& 10\ & \gr{0.04} 0.04 & \gr{0.17} 0.17 & \gr{0.27} \underline{\textbf{0.27}} & \gr{0.09} 0.09 & \gr{0.12} 0.12\\
& 20\ & \gr{0.01} 0.01 & \gr{0.32} \underline{\textbf{0.32}} & \gr{0.24} \textbf{0.24} & \gr{0.12} 0.12 & \gr{0.16} \textbf{0.16}\\
& 50\ & \gr{0.03} 0.03 & \gr{0.29} \underline{\textbf{0.29}} & \gr{0.21} \textbf{0.21} & \gr{0.13} 0.13 & \gr{0.11} 0.11\\
& 100\ & \gr{0.08} 0.08 & \gr{0.27} \textbf{0.27} & \gr{0.28} \underline{\textbf{0.28}} & \gr{0.21} 0.21 & \gr{0.12} 0.12\\
& 200\ & \gr{0.08} 0.08 & \gr{0.47} \underline{\textbf{0.47}} & \gr{0.32} \textbf{0.32} & \gr{0.37} \textbf{0.37} & \gr{0.08} 0.08\\
& 500\ & \gr{0.32} 0.32 & \gr{0.37} 0.37 & \gr{0.48} \underline{0.48} & \gr{0.36} 0.36 & \gr{0.09} 0.09\\
& 1000\ & \gr{0.31} 0.31 & \gr{0.43} 0.43 & \gr{0.59} \underline{\textbf{0.59}} & \gr{0.47} 0.47 & \gr{0.08} 0.08\\
& 2000\ & \gr{0.41} 0.41 & \gr{0.49} 0.49 & \gr{0.57} 0.57 & \gr{0.59} \underline{0.59} & \gr{0.03} 0.03\\
& 5000\ & \gr{0.48} 0.48 & \gr{0.53} 0.53 & \gr{0.60} \underline{0.60} & \gr{0.51} 0.51 & \gr{0.16} 0.16\\
  \hhline{~~|-----}
\end{tabular}%
\quad\mbox{}\\[6pt]
\begin{flushright}
    \captsize \textbf{(a)} \domain{exploding\_blocksworld} problem \texttt{p4}
\end{flushright}
\end{minipage}
\hfill
\begin{minipage}{0.33\textwidth}
\raggedleft
\colorlet{maxcolor}{red!50}
\colorlet{mincolor}{white}

\newcolumntype{C}[1]{>{\raggedleft\arraybackslash}m{#1}}

\renewcommand\gr[1]{\gradient{#1}{0.36}{0.67}{0.97}{maxcolor}{mincolor}}

\setlength{\tabcolsep}{2pt}
\setlength{\colwidth}{0.7cm}

\begin{tabular}{rr|C{\colwidth}C{\colwidth}C{\colwidth}C{\colwidth}C{\colwidth}|}
  & \multicolumn{1}{c}{} & \multicolumn{5}{c}{\begin{tabular}{@{}c@{}}less\\[-1ex]greedy\end{tabular} $\longleftarrow \alpha \longrightarrow$ \begin{tabular}{@{}c@{}}more\\[-1ex]greedy\end{tabular}}\\
    & \multicolumn{1}{c}{} & \multicolumn{1}{c}{0\tiny{/UCT}} & \multicolumn{1}{c}{0.2} & \multicolumn{1}{c}{0.5} & \multicolumn{1}{c}{0.8} & \multicolumn{1}{c}{1}\\
    \hhline{~~|-----}
  \parbox[t]{3mm}{\multirow{10}{*}{\rotatebox[origin=c]{90}{number of rollouts}}} & 5 &\gr{0.37} 0.37 & \gr{0.71} \underline{\textbf{0.71}} & \gr{0.64} \textbf{0.64} & \gr{0.67} \textbf{0.67} & \gr{0.60} 0.60\\
  & 10\ & \gr{0.49} 0.49 & \gr{0.81} \underline{\textbf{0.81}} & \gr{0.75} \textbf{0.75} & \gr{0.81} \underline{\textbf{0.81}} & \gr{0.68} 0.68\\
  & 20\ & \gr{0.77} 0.77 & \gr{0.81} \underline{0.81} & \gr{0.80} 0.80 & \gr{0.81} \underline{0.81} & \gr{0.81} \underline{0.81}\\
  & 50\ & \gr{0.77} 0.77 & \gr{0.81} 0.81 & \gr{0.85} \underline{0.85} & \gr{0.84} 0.84 & \gr{0.76} 0.76\\
  & 100\ & \gr{0.85} 0.85 & \gr{0.84} 0.84 & \gr{0.87} \underline{0.87} & \gr{0.77} 0.77 & \gr{0.76} 0.76\\
  & 200\ & \gr{0.84} 0.84 & \gr{0.77} 0.77 & \gr{0.83} 0.83 & \gr{0.80} 0.80 & \gr{0.87} \underline{0.87}\\
  & 500\ & \gr{0.81} 0.81 & \gr{0.84} 0.84 & \gr{0.85} \underline{0.85} & \gr{0.85} \underline{0.85} & \gr{0.84} 0.84\\
  & 1000\ & \gr{0.87} 0.87 & \gr{0.93} 0.93 & \gr{0.95} \underline{0.95} & \gr{0.84} 0.84 & \gr{0.79} 0.79\\
  & 2000\ & \gr{0.96} 0.96 & \gr{0.97} \underline{0.97} & \gr{0.97} \underline{0.97} & \gr{0.80} 0.80 & \gr{0.83} 0.83\\
  & 5000\ & \gr{0.92} 0.92 & \gr{0.97} \underline{0.97} & \gr{0.97} \underline{0.97} & \gr{0.93} 0.93 & \gr{0.87} 0.87\\
  \hhline{~~|-----}
\end{tabular}%
\quad\mbox{}\\[6pt]
\begin{center}
    \captsize \hspace{1cm} \textbf{(b)} \domain{tireworld} problem \texttt{p15}
\end{center}
\end{minipage}
\hfill
\begin{minipage}{0.33\textwidth}
\raggedleft
\colorlet{maxcolor}{red!50}
\colorlet{mincolor}{white}

\newcolumntype{C}[1]{>{\raggedleft\arraybackslash}m{#1}}

\renewcommand\gr[1]{\gradient{#1}{0}{0.3}{1.00}{maxcolor}{mincolor}}

\setlength{\tabcolsep}{2pt}
\setlength{\colwidth}{0.7cm}

\begin{tabular}{rr|C{\colwidth}C{\colwidth}C{\colwidth}C{\colwidth}C{\colwidth}|}
  & \multicolumn{1}{c}{} & \multicolumn{5}{c}{\begin{tabular}{@{}c@{}}less\\[-1ex]greedy\end{tabular} $\longleftarrow \alpha \longrightarrow$ \begin{tabular}{@{}c@{}}more\\[-1ex]greedy\end{tabular}}\\
    & \multicolumn{1}{c}{} & \multicolumn{1}{c}{0\tiny{/UCT}} & \multicolumn{1}{c}{0.2} & \multicolumn{1}{c}{0.5} & \multicolumn{1}{c}{0.8} & \multicolumn{1}{c}{1}\\
    \hhline{~~|-----}
  \parbox[t]{3mm}{\multirow{10}{*}{\rotatebox[origin=c]{90}{number of rollouts}}} & 5 &\gr{0.71} \underline{0.71} & \gr{0.21} 0.21 & \gr{0.27} 0.27 & \gr{0.29} 0.29 & \gr{0.35} 0.35\\
  & 10\ & \gr{0.71} \underline{0.71} & \gr{0.21} 0.21 & \gr{0.28} 0.28 & \gr{0.27} 0.27 & \gr{0.25} 0.25\\
  & 20\ & \gr{0.89} \underline{0.89} & \gr{0.27} 0.27 & \gr{0.13} 0.13 & \gr{0.16} 0.16 & \gr{0.27} 0.27\\
  & 50\ & \gr{0.97} \underline{0.97} & \gr{0.25} 0.25 & \gr{0.20} 0.20 & \gr{0.11} 0.11 & \gr{0.29} 0.29\\
  & 100\ & \gr{0.99} \underline{0.99} & \gr{0.33} 0.33 & \gr{0.20} 0.20 & \gr{0.13} 0.13 & \gr{0.15} 0.15\\
  & 200\ & \gr{1.00} \underline{1.00} & \gr{0.47} 0.47 & \gr{0.12} 0.12 & \gr{0.15} 0.15 & \gr{0.31} 0.31\\
  & 500\ & \gr{1.00} \underline{1.00} & \gr{0.35} 0.35 & \gr{0.23} 0.23 & \gr{0.15} 0.15 & \gr{0.08} 0.08\\
  & 1000\ & \gr{1.00} \underline{1.00} & \gr{0.41} 0.41 & \gr{0.33} 0.33 & \gr{0.11} 0.11 & \gr{0.12} 0.12\\
  & 2000\ & \gr{1.00} \underline{1.00} & \gr{0.43} 0.43 & \gr{0.29} 0.29 & \gr{0.09} 0.09 & \gr{0.23} 0.23\\
  & 5000\ & \gr{1.00} \underline{1.00} & \gr{0.47} 0.47 & \gr{0.33} 0.33 & \gr{0.17} 0.17 & \gr{0.19} 0.19\\
  \hhline{~~|-----}
\end{tabular}%
\quad\mbox{}\\[6pt]
\begin{flushright}
    \captsize \textbf{(c)} \domain{triangle\_tireworld} problem \texttt{p2}
\end{flushright}
\end{minipage}
\caption{Success rate  (higher is better) of \LAMP reaching a goal state. $\mathrm{LM}^{\text{RHW}}$ generated 9, 3, and 4 nontrivial landmarks for these problem instances, respectively.
Underlined values indicate the best performing $\alpha$-value in the row. Boldfaced values indicate where \LAMP significantly ($p < 0.0125$) dominated standard UCT ($\alpha = 0$) at the same number of rollouts. Statistical analysis and results from other problem instances are in Appendix \ref{appendix:results_prob_interesting}. }
\label{fig:prob_interesting}
\end{figure*}

In Figure \ref{fig:prob_interesting}, we selected one problem instance from each domain. Given the potential for failure, for each problem, we reported the success rate of \LAMP in reaching the goal within the execution budget, over 75 runs. Additional results from another 26 problem instances in these domains are presented in Appendix \ref{appendix:results_prob_interesting}. 

In the \domain{exploding\_blocksworld} problem, standard UCT performs significantly worse, even at 1000 rollouts, with a greediness value between 0.2 and 0.5 yielding the best results. In the \domain{tireworld} problem, a greedy approach dominates standard UCT when constrained to fewer than 10 rollouts, but gradually loses its advantage as the number of rollouts increases. 

The \domain{triangle tireworld} problem, depicted in Figure~\ref{fig:tireworld_triangle}, is specifically chosen to elicit pathological behavior from a greedy planner, representing a worst-case scenario for LAMP. Although a safe solution exists, where the car moves to node 9, then to the goal, the identified landmarks can mislead a greedy planner toward an unsafe path. Starting from node 1, the fastest way to achieve the first landmark, $\varphi_1$, is to move to node 3. Although not an immediate deadlock, it is unsafe, which explains why any inclination toward greediness will only lower the overall success rate. A greedy approach may enter unsafe or deadlock states while achieving a landmark, overlooking that this may prevent it from reaching a future subgoal. 

We analyzed the significance of the observed differences between LAMP and baseline UCT using a two-tailed Boschloo exact test (similar to a Fisher exact test) with a Bonferroni adjusted $\alpha_{\textit{stat}}=0.05/4$. \LAMP statistically dominates standard UCT in 6 of 10 rows in the \domain{exploding\_blocksworld} problem, and in 2 of 10 rows in the \domain{tireworld} problem. However, standard UCT dominates all greedy variants in the \domain{traingle\_tireworld} problem. Full statistical analysis and results from other problems are in Appendix~\ref{appendix:results_prob_interesting}. In the complete set of results, \LAMP significantly outperformed standard UCT in 11 of 29 probabilistically interesting benchmark problems, and in 42 of 290 total rows. 

\subsection{Discussion}

The results indicate that in problems without deadlock states, when the algorithm is constrained to a small number of rollouts, a greedy approach ($\alpha > 0$) often outperforms standard UCT ($\alpha = 0$). As the number of rollouts grows, standard UCT will eventually converge to an optimal solution; however, for larger planning problems, this can take prohibitively long. In contrast, a greedy approach decomposes the large problem into smaller subproblems, allowing UCT to converge to viable policies for these subproblems much more quickly than for the top-level goal. This makes a greedy algorithm an advantageous choice in large domains or for anytime settings. However, in a domain with deadlock states, a greedy approach may inadvertently enter a deadlock state while achieving a landmark, thereby preventing it from achieving the top-level goal. While a greedy approach can still outperform the baseline in such scenarios, careful selection of landmarks is crucial to ensure that none of them contain or lead to deadlock states.

\begin{figure}[t]
    \centering\begin{tikzpicture}[scale=1]
    \node[rectangle, draw=black, minimum size=0.25cm] (n1) at (0, 0) {};
    \node[circle, draw=black, minimum size=0.25cm] (n2) at (1, .5) {};
    \node[rectangle, draw=black, minimum size=0.25cm] (n3) at (0, 1) {};
    \node[circle, draw=black, minimum size=0.25cm] (n4) at (2, 1) {};
    \node[circle, draw=black, minimum size=0.25cm] (n5) at (1, 1.5) {};
    \node[circle, draw=black, minimum size=0.25cm] (n6) at (3, 1.5) {};
    \node[rectangle, draw=black, minimum size=0.25cm] (n7) at (0, 2) {};
    \node[rectangle, draw=black, minimum size=0.25cm] (n8) at (2, 2) {};
    \node[circle, draw=black, minimum size=0.25cm] (n9) at (4, 2) {};
    \node[circle, draw=black, minimum size=0.25cm] (n10) at (1, 2.5) {};
    \node[circle, draw=black, minimum size=0.25cm] (n11) at (3, 2.5) {};
    \node[rectangle, draw=black, minimum size=0.25cm] (n12) at (0, 3) {};
    \node[circle, draw=black, minimum size=0.25cm] (n13) at (2, 3) {};
    \node[circle, draw=black, minimum size=0.25cm] (n14) at (1, 3.5) {};
    \node[rectangle, draw=black, minimum size=0.25cm] (n15) at (0, 4) {};
    \draw[-latex] (n1) -- (n2);
    \draw[-latex] (n2) -- (n3);
    \draw[-latex] (n2) -- (n4);
    \draw[-latex] (n3) -- (n5); 
    \draw[-latex] (n4) -- (n5); 
    \draw[-latex] (n4) -- (n6); 
    \draw[-latex] (n5) -- (n7); 
    \draw[-latex] (n4) -- (n8); 
    \draw[-latex] (n6) -- (n8); 
    \draw[-latex] (n6) -- (n9); 
    \draw[-latex] (n7) -- (n10); 
    \draw[-latex] (n8) -- (n13); 
    \draw[-latex] (n8) -- (n11); 
    \draw[-latex] (n9) -- (n11); 
    \draw[-latex] (n10) -- (n12);
    \draw[-latex] (n10) -- (n13); 
    \draw[-latex] (n11) -- (n13); 
    \draw[-latex] (n12) -- (n14); 
    \draw[-latex] (n13) -- (n14); 
    \draw[-latex] (n14) -- (n15); 
    \draw[-latex] (n1) -- (n3); 
    \draw[-latex] (n3) -- (n7); 
    \draw[-latex] (n7) -- (n12); 
    \draw[-latex] (n12) -- (n15); 
    \node at (n1) [xshift=-.75cm] {start};
    \node at (n15) [xshift=-.75cm] {goal};
    \node[draw=black,dashed,inner sep=5pt,rounded corners=1mm,rotate fit=26.57,fit=(n3)(n11)] {};
    \node[draw=black,dashed,inner sep=5pt,rounded corners=1mm,rotate fit=26.57,fit=(n7)(n13)] {};
    \node[draw=black,dashed,inner sep=5pt,rounded corners=1mm,rotate fit=26.57,fit=(n12)(n14)] {};
    \node[draw=black,dashed,inner sep=5pt,rounded corners=1mm,rotate fit=26.57,fit=(n15)] {};
    \node at (n11) [xshift=.6cm, yshift=.3cm] {$\varphi_1$};
    \node at (n13) [xshift=.6cm, yshift=.3cm] {$\varphi_2$};
    \node at (n14) [xshift=.6cm, yshift=.3cm] {$\varphi_3$};
    \node at (n15) [xshift=.6cm, yshift=.3cm] {$\varphi_4$};
    \node at (n1) {\scriptsize{1}};
    \node at (n2) {\scriptsize{2}};
    \node at (n3) {\scriptsize{3}};
    \node at (n4) {\scriptsize{4}};
    \node at (n5) {\scriptsize{5}};
    \node at (n6) {\scriptsize{6}};
    \node at (n7) {\scriptsize{7}};
    \node at (n8) {\scriptsize{8}};
    \node at (n9) {\scriptsize{9}};
    \node at (n10) {\scriptsize{10}};
    \node at (n11) {\scriptsize{11}};
    \node at (n12) {\scriptsize{12}};
    \node at (n13) {\scriptsize{13}};
    \node at (n14) {\scriptsize{14}};
    \node at (n15) {\scriptsize{15}};
\end{tikzpicture}
    \caption{A depiction of \domain{triangle\_tireworld} problem \texttt{p2} of moving a car from node 1 to node 15. Each time the car moves between nodes, there is a 50\% chance it gets a flat tire. There are spare tires at circular nodes. All roads are ``one-way.'' The only safe solution is to move to node 9, then to node 15. The landmarks generated by $\mathrm{LM}^{\text{RHW}}$ are shown in the dashed boxes; e.g. $\varphi_2 \equiv (\text{car-at}(7) \vee \text{car-at}(10) \vee \text{car-at}(13))$.} 
    \label{fig:tireworld_triangle}
\end{figure}
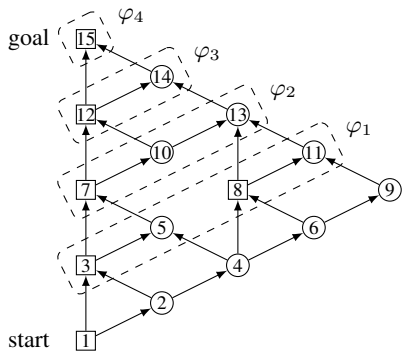

\section{Related work}
\label{sec:related}

\citet{porteous2001extraction} first introduced landmarks for propositional planning. Although deciding whether a propositional formula is a landmark is \textsf{\small PSPACE-complete}, they proposed an algorithm, $\mathrm{LM}^{\text{RPG}}$, that could efficiently extract landmarks and orderings using a delete-relaxed planning graph (RPG). Using the resulting landmark graph, their proposed method restricted the action space by disallowing actions that would achieve landmarks out of order, achieving limited results. \citet{hoffmann2004ordered} later developed an alternative approach that used landmarks to decompose the planning problem by iteratively searching for a plan to the nearest landmark with no predecessors, rather than searching for a plan to the goal, demonstrating a more substantial speed-up. 
\citet{vernhes2013problem} later revisited this problem-splitting approach with a search framework based on landmark orderings for a given landmark graph. 

There have also been significant developments in using landmarks to generate heuristic values to guide planning \cite{zhu2003landmark,richter2008landmarks,helmert2009landmarks,pommerening2013incremental,buchner2024hitting}. Others have extended these approaches beyond classical planning to numeric \cite{scala2017landmarks} and lifted planning \cite{wichlacz2022landmark}.

Landmarks have not been as widely used in probabilistic planning problems. \citet{speck2015necessary} used landmarks to compute a set of necessary observations used to minimize the number of sensors on an agent in a partially observable nondeterministic domain. \citet{sreedharan2020tldr} defined policy landmarks for MDPs to identify subgoals for a given policy. These policy landmarks were then used to provide high-level explanations of policies and required actions \cite{sreedharan2023generalizing}.

Within reinforcement learning, the options framework \cite{sutton1999between} is an example of a hierarchical approach that models temporally extended actions or skills, allowing an agent to solve smaller subproblems and compose the resulting policies. Several works aim to learn options that navigate to ``bottleneck'' states that occur frequently on solution trajectories \cite{mcgovern2001automatic,stolle2002learning,menache2002q,csimcsek2004using,csimcsek2005identifying}, or prototypical states of well-connected regions of the state space \cite{ramesh2019successor}. Then, these options may be used alongside primitive actions in learning methods such as $Q$-learning. Landmarks have also been used to guide reinforcement learning agents through POMDPs by providing guiding rewards based on the value of visiting each landmark in the task \cite{demir2023landmark}. 

\section{Conclusion and future work}

We extended classical landmarks to probabilistic settings and introduced \LAMP, a particular implementation of \LandmarkPlan based on UCT, as a domain-independent probabilistic planner that leverages classical landmarks as subgoals in MDPs. \LAMP outperforms standard UCT on several benchmark planning problems, demonstrating the effectiveness of using landmarks for probabilistic planning, and suggesting that it may be worthwhile to incorporate landmark guidance in more sophisticated probabilistic planning systems. 

We also studied the trade-off between a greedy and non-greedy approach to achieving landmarks during planning. A non-greedy approach---equivalent to UCT---provably converges to an optimal solution without landmarks but may take prohibitively long to find solutions in large planning domains, making it impractical for online or time-constrained applications. In contrast, a greedy approach decomposes the planning task into a sequence of simpler subproblems. It often outperforms plain UCT with fewer rollouts, albeit at the cost of completeness. Problems where landmarks may include deadlock states or unsafe states may pose a challenge to a greedy algorithm. 

Future work could focus on integrating deadlock detection and avoidance into \LAMP, mitigating the risk of poorly chosen landmarks. More broadly, it could investigate alternative methods, beyond probabilistic landmarks, to identify effective subgoals for problem decomposition. Since we observed that the optimal greediness value $\alpha$ depends on both the problem domain and the number of rollouts, future work could investigate the potential for an adaptive $\alpha$-value that changes as planning progresses. Finally, using other probabilistic planning techniques, including non-sampling-based approaches, to implement the general landmark-based decomposition approach outlined in \LandmarkPlan (Algorithm~\ref{alg:landmarkplan}) could yield improvements similar to those achieved by our UCT-based approach.


\begin{ack}
We thank NRL for funding this research. We also thank Paul Zaidins for many helpful discussions.  
\end{ack}


\bibliography{mybibfile}

\appendix
\onecolumn
\renewcommand{\thefigure}{A\arabic{figure}}
\setcounter{figure}{0}
\renewcommand{\thepage}{A\arabic{page}}
\setcounter{page}{1}

\section{Sequential landmark planning}
\label{appendix:sequential}

Let $\prob = (\Sigma, s_0, g)$ be a probabilistic planning problem, let $\langle \varphi_1, \dots, \varphi_k\rangle$ be a sequence of landmarks where $\varphi_k = g$. For $i \in \{1, \dots, k\}$, suppose policy $\pi_i$ achieves $\varphi_i$. A solution to $\prob$ can be obtained by running policies $\pi_1, \dots, \pi_k$ in sequence.
\begin{algorithm}[h]
    \centering
    \caption{SequentialPlan}\label{sequentialplan}
    \begin{algorithmic}[1]
\Procedure{Sequential-Plan}{$\Sigma, s_0, \langle \varphi_1, \dots, \varphi_k \rangle$}
   \State $s \gets s_0$
   \For{$i = 1 \dots k$}
      \State $\pi_i \gets$ policy for $(\Sigma, s, \varphi_i)$
      \While{$s \not\models \varphi_i$}
         \State $s \gets \mathsf{apply}(s, \pi_i(s))$
      \EndWhile
   \EndFor
\EndProcedure
\end{algorithmic}
    \label{alg:seq}
\end{algorithm}

\noindent In the best case, for a well-chosen sequence of landmarks, this can yield an exponential reduction in planning time. 

\begin{figure}[h]
    \centering \begin{tikzpicture}
    \node[circle, draw=black,minimum size=0.4cm,inner sep=0pt] (root1) at (0, 0) {\scriptsize{$s_0$}};
    \node[circle, draw=black,minimum size=0.4cm,inner sep=0pt] (goal1) at (3, 0) {\scriptsize{$s_g$}};
    \draw[draw=black] (root1) -- (3, 2);
    \draw[draw=black] (root1) -- (3, -2);
    \draw[decorate,decoration=zigzag] (root1) -- (goal1);

    \node[circle, draw=black,minimum size=0.4cm,inner sep=0pt] (root2) at (4, 0) {\scriptsize{$s_0$}};
    \node[circle, draw=black,minimum size=0.4cm,inner sep=0pt] (lm1) at (5, .1) {\scriptsize{$s_1$}};
    \node[circle, draw=black,minimum size=0.4cm,inner sep=0pt] (lm2) at (6, -.1) {\scriptsize{$s_2$}};
    \node[circle, draw=black,minimum size=0.4cm,inner sep=0pt] (goal2) at (7, 0) {\scriptsize{$s_g$}};
    \draw[draw=black] (root2) -- (5, 2/3);
    \draw[draw=black] (root2) -- (5, -2/3);
    \draw[decorate,decoration=zigzag] (root2) -- (lm1);
    \draw[draw=black] (lm1) -- (6, .1 + 2/3);
    \draw[draw=black] (lm1) -- (6, .1 - 2/3);
    \draw[decorate,decoration=zigzag] (lm2) -- (lm1);
    \draw[draw=black] (lm2) -- (7, 2/3 - .1);
    \draw[draw=black] (lm2) -- (7, -2/3 - .1);
    \draw[decorate,decoration=zigzag] (lm2) -- (goal2);
\end{tikzpicture}
    \caption{Suppose $\varphi_1$ and $\varphi_2$ are landmarks such that $s_1 \models \varphi_1$ and $s_2 \models \varphi_2$. In the best case, sequential planning can yield an exponential speed-up. }
    \label{fig:speedup}
\end{figure}
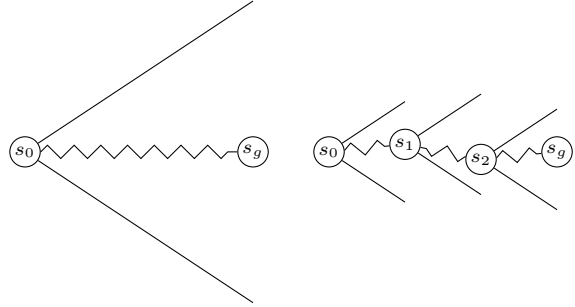

However, in general, for a problem $\prob = (\Sigma, s_0, g)$ and an unordered collection of landmarks $\Phi$, there is no unique ``optimal'' sequencing of $\Phi$ that minimizes expected cost. As illustrated in Figure~\ref{fig:ordering}, in the classical setting, the best order to achieve a set of landmarks in may depend on the subplans used for each landmark. Moreover, in the probabilistic setting, the optimal sequence for attaining landmarks can even depend on the probabilistic outcomes of actions.
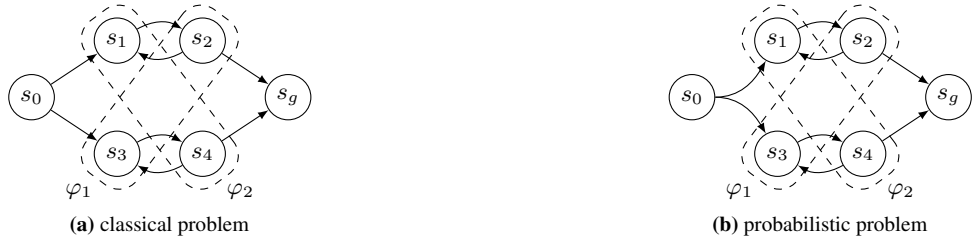
\begin{figure}[h]
    \centering
    \begin{minipage}[b]{.45\linewidth}
        \centering\begin{tikzpicture}[scale=1.5]
    \node[circle, draw=black,minimum size=0.6cm,inner sep=0pt] (r1) at (0, 0) {$s_0$};
    \node[circle, draw=black,minimum size=0.6cm,inner sep=0pt] (s1) at (.75, .5) {$s_1$};
    \node[circle, draw=black,minimum size=0.6cm,inner sep=0pt] (s2) at (1.5, .5) {$s_2$};
    \node[circle, draw=black,minimum size=0.6cm,inner sep=0pt] (s3) at (.75, -.5) {$s_3$};
    \node[circle, draw=black,minimum size=0.6cm,inner sep=0pt] (s4) at (1.5, -.5) {$s_4$};
    \node[circle, draw=black,minimum size=0.6cm,inner sep=0pt] (g1) at (2.25, 0) {$s_g$};
    \draw [-latex] (r1) to (s1);
    \draw [-latex] (r1) to (s3);
    \draw [-latex] (s1) to [out=30, in=150] (s2);
    \draw [-latex] (s2) to [out=210, in=330] (s1);
    \draw [-latex] (s3) to [out=30, in=150] (s4);
    \draw [-latex] (s4) to [out=210, in=330] (s3);
    \draw [-latex] (s2) to (g1);
    \draw [-latex] (s4) to (g1);
    \node[draw=black,dashed,inner sep=5pt,rounded corners=.3cm,rotate fit=53.13,fit=(s2)(s3)] {};
    \node[draw=black,dashed,inner sep=5pt,rounded corners=.3cm,rotate fit=-53.13,fit=(s1)(s4)] {};
    \node at (s3) [xshift=-.5cm, yshift=-.5cm] {$\varphi_1$};
    \node at (s4) [xshift=.5cm, yshift=-.5cm] {$\varphi_2$};
\end{tikzpicture}\\[0pt]
        
        \captsize \textbf{(a)} classical problem
    \end{minipage}
    \hfil
    \begin{minipage}[b]{.45\linewidth}
        \centering\begin{tikzpicture}[scale=1.5]
    \node[circle, draw=black,minimum size=0.6cm,inner sep=0pt] (r1) at (3, 0) {$s_0$};
    \node[circle, draw=black,minimum size=0.6cm,inner sep=0pt] (s1) at (3.75, .5) {$s_1$};
    \node[circle, draw=black,minimum size=0.6cm,inner sep=0pt] (s2) at (4.5, .5) {$s_2$};
    \node[circle, draw=black,minimum size=0.6cm,inner sep=0pt] (s3) at (3.75, -.5) {$s_3$};
    \node[circle, draw=black,minimum size=0.6cm,inner sep=0pt] (s4) at (4.5, -.5) {$s_4$};
    \node[circle, draw=black,minimum size=0.6cm,inner sep=0pt] (g1) at (5.25, 0) {$s_g$};
    \draw [-latex] (r1) to [out=0,in=240] (s1);
    \draw [-latex] (r1) to [out=0,in=120] (s3);
    \draw [-latex] (s1) to [out=30, in=150] (s2);
    \draw [-latex] (s2) to [out=210, in=330] (s1);
    \draw [-latex] (s3) to [out=30, in=150] (s4);
    \draw [-latex] (s4) to [out=210, in=330] (s3);
    \draw [-latex] (s2) to (g1);
    \draw [-latex] (s4) to (g1);
    \node[draw=black,dashed,inner sep=5pt,rounded corners=.3cm,rotate fit=53.13,fit=(s2)(s3)] {};
    \node[draw=black,dashed,inner sep=5pt,rounded corners=.3cm,rotate fit=-53.13,fit=(s1)(s4)] {};
    \node at (s3) [xshift=-.5cm, yshift=-.5cm] {$\varphi_1$};
    \node at (s4) [xshift=.5cm, yshift=-.5cm] {$\varphi_2$};
\end{tikzpicture}\\[0pt]
        \captsize \textbf{(b)} probabilistic problem 
    \end{minipage}
    \caption{In both problems, suppose $\varphi_1$ and $\varphi_2$ are landmarks such that $s_2, s_3 \models \varphi_1$ and $s_1, s_4 \models \varphi_2$. Without first fixing policies $\pi_1$ and $\pi_2$ which achieve $\varphi_1$ and $\varphi_2$, respectively, it is unclear whether $\mathsc{SequentialPlan}(\Sigma, s_0, \langle \varphi_1, \varphi_2\rangle)$ or $\mathsc{SequentialPlan}(\Sigma, s_0, \langle \varphi_2, \varphi_1\rangle)$ would be better. 
    }
    \label{fig:ordering}
\end{figure}

Moreover, the \textsc{Sequential-Plan} procedure is incomplete. Failures can occur in two situations: (1) the chosen landmark is unreachable, or (2) a deadlock state is encountered from which the final goal is unreachable.
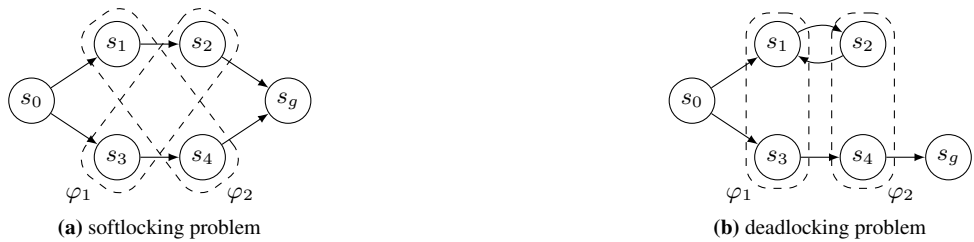
\begin{figure}[h!]
    \centering
    \begin{minipage}[b]{.45\linewidth}
        \centering\begin{tikzpicture}[scale=1.5]
    \node[circle, draw=black,minimum size=0.6cm,inner sep=0pt] (r1) at (0, 0) {$s_0$};
    \node[circle, draw=black,minimum size=0.6cm,inner sep=0pt] (s1) at (.75, .5) {$s_1$};
    \node[circle, draw=black,minimum size=0.6cm,inner sep=0pt] (s2) at (1.5, .5) {$s_2$};
    \node[circle, draw=black,minimum size=0.6cm,inner sep=0pt] (s3) at (.75, -.5) {$s_3$};
    \node[circle, draw=black,minimum size=0.6cm,inner sep=0pt] (s4) at (1.5, -.5) {$s_4$};
    \node[circle, draw=black,minimum size=0.6cm,inner sep=0pt] (g1) at (2.25, 0) {$s_g$};
    \draw [-latex] (r1) to (s1);
    \draw [-latex] (r1) to (s3);
    \draw [-latex] (s1) to (s2);
    \draw [-latex] (s3) to (s4);
    \draw [-latex] (s2) to (g1);
    \draw [-latex] (s4) to (g1);
    \node[draw=black,dashed,inner sep=5pt,rounded corners=.3cm,rotate fit=53.13,fit=(s2)(s3)] {};
    \node[draw=black,dashed,inner sep=5pt,rounded corners=.3cm,rotate fit=-53.13,fit=(s1)(s4)] {};
    \node at (s3) [xshift=-.5cm, yshift=-.5cm] {$\varphi_1$};
    \node at (s4) [xshift=.5cm, yshift=-.5cm] {$\varphi_2$};
\end{tikzpicture}\\[0pt]
        
        \captsize \textbf{(a)} softlocking problem
    \end{minipage}
    \hfil
    \begin{minipage}[b]{.45\linewidth}
        \centering\begin{tikzpicture}[scale=1.5]
    \node[circle, draw=black,minimum size=0.6cm,inner sep=0pt] (r1) at (0, 0) {$s_0$};
    \node[circle, draw=black,minimum size=0.6cm,inner sep=0pt] (s1) at (.75, .5) {$s_1$};
    \node[circle, draw=black,minimum size=0.6cm,inner sep=0pt] (s2) at (1.5, .5) {$s_2$};
    \node[circle, draw=black,minimum size=0.6cm,inner sep=0pt] (s3) at (.75, -.5) {$s_3$};
    \node[circle, draw=black,minimum size=0.6cm,inner sep=0pt] (s4) at (1.5, -.5) {$s_4$};
    \node[circle, draw=black,minimum size=0.6cm,inner sep=0pt] (g1) at (2.25, -.5) {$s_g$};
    \draw [-latex] (r1) to (s1);
    \draw [-latex] (r1) to (s3);
    \draw [-latex] (s3) to (s4);
    \draw [-latex] (s1) to [out=30, in=150] (s2);
    \draw [-latex] (s2) to [out=210, in=330] (s1);
    \draw [-latex] (s4) to (g1);
    \node[draw=black,dashed,inner sep=3pt,rounded corners=.3cm,fit=(s1)(s3)] {};
    \node[draw=black,dashed,inner sep=3pt,rounded corners=.3cm,fit=(s2)(s4)] {};
    \node at (s3) [xshift=-.5cm, yshift=-.5cm] {$\varphi_1$};
    \node at (s4) [xshift=.5cm, yshift=-.5cm] {$\varphi_2$};
\end{tikzpicture}\\[0pt]
        \captsize \textbf{(b)} deadlocking problem
    \end{minipage}
    \caption{In both problems, consider the execution of $\mathsc{SequentialPlan}(\Sigma, s_0, \langle \varphi_1, \varphi_2\rangle)$. In \textbf{(a)}, if $\varphi_1$ is achieved by reaching state $s_2$, it is impossible to achieve $\varphi_2$. In \textbf{(b)}, if $\varphi_1$ is achieved by reaching state $s_1$, it is impossible to reach the final goal. 
    }
    \label{fig:locks}
\end{figure}
Consider the execution of $\mathsc{SequentialPlan}(\Sigma, s_0, \langle \varphi_1, \varphi_2\rangle)$. In Figure~\ref{fig:locks}(a), suppose policy $\pi_1$ achieves $\varphi_1$ by reaching state $s_2$. Upon completion of $\varphi_1$, the next landmark $\varphi_2$ is impossible to achieve, and thus the \mathsc{SequentialPlan} procedure fails on Line 4. In Figure~\ref{fig:locks}(b), if \mathsc{SequentialPlan} achieves landmarks $\varphi_1$ by reaching state $s_1$, a deadlock state is reached where the final goal is no longer achievable. 

\section{Additional early benchmarks}
\label{appendix:results_early}

\begin{figure}
    \centering
    \begin{minipage}[b]{4.1in}
\colorlet{maxcolor}{white}
\colorlet{mincolor}{blue!50}

\newcolumntype{C}[1]{>{\raggedleft\arraybackslash}m{#1}}

\renewcommand\gr[1]{\gradient{#1}{23.2}{106.7}{190.1}{maxcolor}{mincolor}}

\setlength{\tabcolsep}{2pt}
\setlength{\colwidth}{0.7cm}

\begin{tabular}[b]{rr|C{\colwidth}C{\colwidth}|C{\colwidth}C{\colwidth}|C{\colwidth}C{\colwidth}|C{\colwidth}C{\colwidth}|C{\colwidth}C{\colwidth}|}
  & \multicolumn{1}{c}{} & \multicolumn{10}{c}{less greedy $\quad \longleftarrow \alpha \longrightarrow \quad$ more greedy}\\
    & \multicolumn{1}{c}{} & \multicolumn{2}{c}{0/UCT} & \multicolumn{2}{c}{0.2} & \multicolumn{2}{c}{0.5} & \multicolumn{2}{c}{0.8} & \multicolumn{2}{c}{1}\\
    \hhline{~~|----------}
  & \multicolumn{1}{c|}{} & \multicolumn{1}{c}{$\overline{x}$} & \multicolumn{1}{c|}{$\sigma$} & \multicolumn{1}{c}{$\overline{x}$} & \multicolumn{1}{c|}{$\sigma$} & \multicolumn{1}{c}{$\overline{x}$} & \multicolumn{1}{c|}{$\sigma$} & \multicolumn{1}{c}{$\overline{x}$} & \multicolumn{1}{c|}{$\sigma$} & \multicolumn{1}{c}{$\overline{x}$} & \multicolumn{1}{c|}{$\sigma$}\\
  \hhline{~~|----------}
  \parbox[t]{5mm}{\multirow{10}{*}{\rotatebox[origin=c]{90}{number of rollouts}}} & 5 &\gr{190.1} 190.1 & \gr{190.1} 32.3 & \gr{162.4} \underline{\textbf{162.4}} & \gr{162.4} 62.7 & \gr{167.5} \textbf{167.5} & \gr{167.5} 58.9 & \gr{180.8} 180.8 & \gr{180.8} 45.8 & \gr{173.8} 173.8 & \gr{173.8} 52.9\\
  & 10\ & \gr{184.0} 184.0 & \gr{184.0} 41.5 & \gr{145.5} \textbf{145.5} & \gr{145.5} 65.7 & \gr{142.1} \underline{\textbf{142.1}} & \gr{142.1} 70.5 & \gr{152.2} \textbf{152.2} & \gr{152.2} 67.7 & \gr{167.3} 167.3 & \gr{167.3} 59.4\\
  & 20\ & \gr{175.0} 175.0 & \gr{175.0} 51.2 & \gr{120.8} \textbf{120.8} & \gr{120.8} 68.6 & \gr{114.8} \underline{\textbf{114.8}} & \gr{114.8} 75.6 & \gr{118.3} \textbf{118.3} & \gr{118.3} 70.4 & \gr{127.6} \textbf{127.6} & \gr{127.6} 70.9\\
  & 50\ & \gr{157.8} 157.8 & \gr{157.8} 58.9 & \gr{69.1} \underline{\textbf{69.1}} & \gr{69.1} 57.0 & \gr{70.1} \textbf{70.1} & \gr{70.1} 58.2 & \gr{86.5} \textbf{86.5} & \gr{86.5} 70.6 & \gr{87.8} \textbf{87.8} & \gr{87.8} 65.6\\
\end{tabular}
        \raggedleft Average cost  \hspace{1.3in}
    \end{minipage}
    \begin{minipage}[b]{2.8in}
\colorlet{maxcolor}{white}
\colorlet{mincolor}{blue!50}
\newcolumntype{C}[1]{>{\centering\arraybackslash}m{#1}}
\renewcommand\gr[1]{\gradient{#1}{-0.01}{0.50}{1.00}{maxcolor}{mincolor}}
\setlength{\tabcolsep}{2pt}
\setlength{\colwidth}{1.2cm}
\begin{tabular}[b]{|C{\colwidth}C{\colwidth}C{\colwidth}C{\colwidth}C{\colwidth}|}
  \multicolumn{5}{c}{less greedy $\quad \longleftarrow \alpha \longrightarrow \quad$ more greedy}\\
    \multicolumn{1}{c}{0/UCT} & \multicolumn{1}{c}{0.2} & \multicolumn{1}{c}{0.5} & \multicolumn{1}{c}{0.8} & \multicolumn{1}{c}{1}\\
    \hhline{|-----}
N/A & \textbf{8.36e-04} & \textbf{4.02e-03} & 1.52e-01 & 2.34e-02\\
N/A & \textbf{3.23e-05} & \textbf{1.79e-05} & \textbf{6.92e-04} & 4.80e-02\\
N/A & \textbf{1.72e-07} & \textbf{5.93e-08} & \textbf{8.43e-08} & \textbf{6.05e-06}\\
N/A & \textbf{1.08e-16} & \textbf{3.68e-16} & \textbf{3.82e-10} & \textbf{1.58e-10}\\
N/A & \textbf{9.23e-27} & \textbf{2.78e-22} & \textbf{1.67e-19} & \textbf{2.17e-15}\\
N/A & \textbf{1.03e-21} & \textbf{1.08e-14} & \textbf{2.38e-20} & \textbf{3.63e-18}\\
N/A & \textbf{1.61e-11} & \textbf{3.90e-13} & \textbf{4.61e-09} & \textbf{1.38e-12}\\
N/A & \textbf{2.33e-06} & \textbf{1.86e-12} & \textbf{5.90e-10} & \textbf{9.85e-13}\\
N/A & \textbf{3.57e-14} & \textbf{2.44e-13} & \textbf{1.75e-10} & \textbf{4.86e-12}\\
N/A & \textbf{4.65e-12} & \textbf{2.95e-10} & \textbf{3.50e-11} & \textbf{1.77e-08}\\
  \hhline{|-----}
\end{tabular}%
        \centering $p$-values
    \end{minipage}
    \begin{center}
        \captsize \textbf{(a)} \domain{prob\_blocksworld} problem \texttt{p1}. The landmark extraction algorithm $\mathrm{LM}^{\text{RHW}}$ generated 11 nontrivial landmarks for this problem. 
    \end{center}
    
    \begin{minipage}[b]{4.1in}
\colorlet{maxcolor}{white}
\colorlet{mincolor}{blue!50}
\newcolumntype{C}[1]{>{\raggedleft\arraybackslash}m{#1}}
\renewcommand\gr[1]{\gradient{#1}{18.1}{105.7}{193.3}{maxcolor}{mincolor}}
\setlength{\tabcolsep}{2pt}
\setlength{\colwidth}{0.7cm}

\begin{tabular}[b]{rr|C{\colwidth}C{\colwidth}|C{\colwidth}C{\colwidth}|C{\colwidth}C{\colwidth}|C{\colwidth}C{\colwidth}|C{\colwidth}C{\colwidth}|}
  & \multicolumn{1}{c}{} & \multicolumn{10}{c}{less greedy $\quad \longleftarrow \alpha \longrightarrow \quad$ more greedy}\\
    & \multicolumn{1}{c}{} & \multicolumn{2}{c}{0/UCT} & \multicolumn{2}{c}{0.2} & \multicolumn{2}{c}{0.5} & \multicolumn{2}{c}{0.8} & \multicolumn{2}{c}{1}\\
    \hhline{~~|----------}
  & \multicolumn{1}{c|}{} & \multicolumn{1}{c}{$\overline{x}$} & \multicolumn{1}{c|}{$\sigma$} & \multicolumn{1}{c}{$\overline{x}$} & \multicolumn{1}{c|}{$\sigma$} & \multicolumn{1}{c}{$\overline{x}$} & \multicolumn{1}{c|}{$\sigma$} & \multicolumn{1}{c}{$\overline{x}$} & \multicolumn{1}{c|}{$\sigma$} & \multicolumn{1}{c}{$\overline{x}$} & \multicolumn{1}{c|}{$\sigma$}\\
  \hhline{~~|----------}
  \parbox[t]{3mm}{\multirow{10}{*}{\rotatebox[origin=c]{90}{number of rollouts}}} & 5 &\gr{193.3} 193.3 & \gr{193.3} 23.3 & \gr{124.9} \underline{\textbf{124.9}} & \gr{124.9} 77.7 & \gr{161.0} \textbf{161.0} & \gr{161.0} 64.3 & \gr{145.5} \textbf{145.5} & \gr{145.5} 70.4 & \gr{138.9} \textbf{138.9} & \gr{138.9} 73.1\\
  & 10\ & \gr{188.3} 188.3 & \gr{188.3} 34.5 & \gr{99.1} \underline{\textbf{99.1}} & \gr{99.1} 76.6 & \gr{111.7} \textbf{111.7} & \gr{111.7} 79.1 & \gr{138.5} \textbf{138.5} & \gr{138.5} 72.8 & \gr{118.6} \textbf{118.6} & \gr{118.6} 78.8\\
  & 20\ & \gr{170.1} 170.1 & \gr{170.1} 55.4 & \gr{85.9} \underline{\textbf{85.9}} & \gr{85.9} 68.4 & \gr{107.6} \textbf{107.6} & \gr{107.6} 77.4 & \gr{107.7} \textbf{107.7} & \gr{107.7} 76.2 & \gr{118.8} \textbf{118.8} & \gr{118.8} 77.5\\
  & 50\ & \gr{133.4} 133.4 & \gr{133.4} 68.7 & \gr{52.3} \underline{\textbf{52.3}} & \gr{52.3} 45.5 & \gr{68.9} \textbf{68.9} & \gr{68.9} 62.9 & \gr{71.1} \textbf{71.1} & \gr{71.1} 63.3 & \gr{72.6} \textbf{72.6} & \gr{72.6} 66.3\\
\end{tabular}
        \raggedleft Average cost  \hspace{1.3in}
    \end{minipage}
    \begin{minipage}[b]{2.8in}
\colorlet{maxcolor}{white}
\colorlet{mincolor}{blue!50}
\newcolumntype{C}[1]{>{\centering\arraybackslash}m{#1}}
\renewcommand\gr[1]{\gradient{#1}{-0.01}{0.50}{1.00}{maxcolor}{mincolor}}
\setlength{\tabcolsep}{2pt}
\setlength{\colwidth}{1.2cm}

\begin{tabular}[b]{|C{\colwidth}C{\colwidth}C{\colwidth}C{\colwidth}C{\colwidth}|}
  \multicolumn{5}{c}{less greedy $\quad \longleftarrow \alpha \longrightarrow \quad$ more greedy}\\
    \multicolumn{1}{c}{0/UCT} & \multicolumn{1}{c}{0.2} & \multicolumn{1}{c}{0.5} & \multicolumn{1}{c}{0.8} & \multicolumn{1}{c}{1}\\
    \hhline{|-----}
N/A & \textbf{1.60e-11} & \textbf{7.01e-05} & \textbf{1.06e-07} & \textbf{7.34e-09}\\
N/A & \textbf{3.38e-16} & \textbf{2.01e-12} & \textbf{3.35e-07} & \textbf{7.74e-11}\\
N/A & \textbf{6.30e-14} & \textbf{6.57e-08} & \textbf{5.16e-08} & \textbf{6.87e-06}\\
N/A & \textbf{1.58e-14} & \textbf{1.47e-08} & \textbf{4.40e-08} & \textbf{1.52e-07}\\
N/A & \textbf{1.62e-17} & \textbf{4.12e-16} & \textbf{1.17e-11} & \textbf{2.44e-11}\\
N/A & \textbf{6.81e-10} & \textbf{2.55e-11} & \textbf{3.45e-09} & \textbf{1.60e-09}\\
N/A & \textbf{4.75e-10} & \textbf{7.28e-10} & \textbf{9.18e-08} & \textbf{2.84e-04}\\
N/A & \textbf{1.00e-08} & \textbf{1.25e-03} & 2.69e-02 & 5.07e-01\\
N/A & \textbf{8.69e-04} & 5.47e-02 & 2.04e-01 & 3.95e-01\\
N/A & \textbf{1.37e-04} & 3.49e-01 & 6.09e-01 & 5.58e-02\\
  \hhline{|-----}
\end{tabular}
        \centering $p$-values
    \end{minipage}
    \begin{center}
        \captsize\textbf{(b)} \domain{prob\_blocksworld} problem \texttt{p2}. The landmark extraction algorithm $\mathrm{LM}^{\text{RHW}}$ generated 11 nontrivial landmarks for this problem. 
    \end{center}
    
    \begin{minipage}[b]{4.1in}
\colorlet{maxcolor}{white}
\colorlet{mincolor}{blue!50}
\newcolumntype{C}[1]{>{\raggedleft\arraybackslash}m{#1}}
\renewcommand\gr[1]{\gradient{#1}{16.1}{103.1}{190.0}{maxcolor}{mincolor}}
\setlength{\tabcolsep}{2pt}
\setlength{\colwidth}{0.7cm}

\begin{tabular}[b]{rr|C{\colwidth}C{\colwidth}|C{\colwidth}C{\colwidth}|C{\colwidth}C{\colwidth}|C{\colwidth}C{\colwidth}|C{\colwidth}C{\colwidth}|}
  & \multicolumn{1}{c}{} & \multicolumn{10}{c}{less greedy $\quad \longleftarrow \alpha \longrightarrow \quad$ more greedy}\\
    & \multicolumn{1}{c}{} & \multicolumn{2}{c}{0/UCT} & \multicolumn{2}{c}{0.2} & \multicolumn{2}{c}{0.5} & \multicolumn{2}{c}{0.8} & \multicolumn{2}{c}{1}\\
    \hhline{~~|----------}
  & \multicolumn{1}{c|}{} & \multicolumn{1}{c}{$\overline{x}$} & \multicolumn{1}{c|}{$\sigma$} & \multicolumn{1}{c}{$\overline{x}$} & \multicolumn{1}{c|}{$\sigma$} & \multicolumn{1}{c}{$\overline{x}$} & \multicolumn{1}{c|}{$\sigma$} & \multicolumn{1}{c}{$\overline{x}$} & \multicolumn{1}{c|}{$\sigma$} & \multicolumn{1}{c}{$\overline{x}$} & \multicolumn{1}{c|}{$\sigma$}\\
  \hhline{~~|----------}
  \parbox[t]{3mm}{\multirow{10}{*}{\rotatebox[origin=c]{90}{number of rollouts}}} & 5 &\gr{190.0} 190.0 & \gr{190.0} 32.6 & \gr{160.5} \textbf{160.5} & \gr{160.5} 64.6 & \gr{158.9} \underline{\textbf{158.9}} & \gr{158.9} 68.6 & \gr{159.7} \textbf{159.7} & \gr{159.7} 64.2 & \gr{159.5} \textbf{159.5} & \gr{159.5} 68.5\\
  & 10\ & \gr{182.9} 182.9 & \gr{182.9} 41.7 & \gr{121.8} \textbf{121.8} & \gr{121.8} 77.1 & \gr{111.7} \underline{\textbf{111.7}} & \gr{111.7} 79.0 & \gr{138.7} \textbf{138.7} & \gr{138.7} 74.1 & \gr{144.9} \textbf{144.9} & \gr{144.9} 71.6\\
  & 20\ & \gr{166.0} 166.0 & \gr{166.0} 60.8 & \gr{103.5} \underline{\textbf{103.5}} & \gr{103.5} 73.2 & \gr{108.5} \textbf{108.5} & \gr{108.5} 78.9 & \gr{130.5} \textbf{130.5} & \gr{130.5} 77.4 & \gr{118.8} \textbf{118.8} & \gr{118.8} 77.9\\
  & 50\ & \gr{129.9} 129.9 & \gr{129.9} 72.0 & \gr{65.7} \textbf{65.7} & \gr{65.7} 60.3 & \gr{64.2} \underline{\textbf{64.2}} & \gr{64.2} 53.3 & \gr{92.4} \textbf{92.4} & \gr{92.4} 73.4 & \gr{96.1} \textbf{96.1} & \gr{96.1} 70.8\\
\end{tabular}
        \raggedleft Average cost  \hspace{1.3in}
    \end{minipage}
    \begin{minipage}[b]{2.8in}
\colorlet{maxcolor}{white}
\colorlet{mincolor}{blue!50}
\newcolumntype{C}[1]{>{\centering\arraybackslash}m{#1}}
\renewcommand\gr[1]{\gradient{#1}{-0.01}{0.50}{1.00}{maxcolor}{mincolor}}
\setlength{\tabcolsep}{2pt}
\setlength{\colwidth}{1.2cm}

\begin{tabular}[b]{|C{\colwidth}C{\colwidth}C{\colwidth}C{\colwidth}C{\colwidth}|}
  \multicolumn{5}{c}{less greedy $\quad \longleftarrow \alpha \longrightarrow \quad$ more greedy}\\
    \multicolumn{1}{c}{0/UCT} & \multicolumn{1}{c}{0.2} & \multicolumn{1}{c}{0.5} & \multicolumn{1}{c}{0.8} & \multicolumn{1}{c}{1}\\
    \hhline{|-----}
N/A & \textbf{5.64e-04} & \textbf{5.24e-04} & \textbf{3.73e-04} & \textbf{6.57e-04}\\
N/A & \textbf{1.24e-08} & \textbf{1.36e-10} & \textbf{1.39e-05} & \textbf{1.08e-04}\\
N/A & \textbf{6.75e-08} & \textbf{1.63e-06} & \textbf{2.16e-03} & \textbf{6.00e-05}\\
N/A & \textbf{2.21e-08} & \textbf{2.46e-09} & \textbf{1.93e-03} & \textbf{4.30e-03}\\
N/A & \textbf{2.11e-10} & \textbf{8.36e-06} & \textbf{1.02e-02} & \textbf{4.00e-03}\\
N/A & \textbf{1.07e-08} & \textbf{2.81e-06} & \textbf{1.55e-03} & \textbf{1.57e-03}\\
N/A & \textbf{3.99e-07} & \textbf{1.16e-09} & \textbf{2.47e-07} & \textbf{1.64e-05}\\
N/A & \textbf{8.43e-08} & \textbf{1.24e-06} & \textbf{5.54e-04} & \textbf{3.93e-03}\\
N/A & \textbf{1.21e-05} & 1.99e-02 & \textbf{5.14e-06} & 9.53e-02\\
N/A & \textbf{1.34e-03} & 2.12e-02 & 7.67e-02 & 1.50e-01\\
  \hhline{|-----}
\end{tabular}
        \centering $p$-values
    \end{minipage}
    \begin{center}
        \captsize \textbf{(c)} \domain{prob\_blocksworld} problem \texttt{p3}. The landmark extraction algorithm $\mathrm{LM}^{\text{RHW}}$ generated 12 nontrivial landmarks for this problem. 
    \end{center}
    \caption{The left tables report the average cost (lower is better) of the solutions generated by \LAMP in the \domain{prob\_blocksworld} domain over 75 runs. Underlined values indicate the best performing $\alpha$-value in the row. The right tables report $p$-values comparing the average cost of solutions generated by LAMP with the average cost of solutions generated by standard UCT ($\alpha = 0$) at the same number of rollouts. In all tables, boldfaced values indicate where LAMP significantly dominated standard UCT at the same number of rollouts, with $p < 0.0125$, the adjusted level using a Bonferroni adjustment of $\alpha_{\textit{stat}} = 0.05/4$.}
    \label{fig:blocks_full}
\end{figure}

\begin{figure}
    \centering
    \begin{minipage}[b]{4.1in}
\colorlet{maxcolor}{white}
\colorlet{mincolor}{blue!50}
\newcolumntype{C}[1]{>{\raggedleft\arraybackslash}m{#1}}
\renewcommand\gr[1]{\gradient{#1}{20.1}{105.9}{191.8}{maxcolor}{mincolor}}
\setlength{\tabcolsep}{2pt}
\setlength{\colwidth}{0.7cm}

\begin{tabular}[b]{rr|C{\colwidth}C{\colwidth}|C{\colwidth}C{\colwidth}|C{\colwidth}C{\colwidth}|C{\colwidth}C{\colwidth}|C{\colwidth}C{\colwidth}|}
  & \multicolumn{1}{c}{} & \multicolumn{10}{c}{less greedy $\quad \longleftarrow \alpha \longrightarrow \quad$ more greedy}\\
    & \multicolumn{1}{c}{} & \multicolumn{2}{c}{0/UCT} & \multicolumn{2}{c}{0.2} & \multicolumn{2}{c}{0.5} & \multicolumn{2}{c}{0.8} & \multicolumn{2}{c}{1}\\
    \hhline{~~|----------}
  & \multicolumn{1}{c|}{} & \multicolumn{1}{c}{$\overline{x}$} & \multicolumn{1}{c|}{$\sigma$} & \multicolumn{1}{c}{$\overline{x}$} & \multicolumn{1}{c|}{$\sigma$} & \multicolumn{1}{c}{$\overline{x}$} & \multicolumn{1}{c|}{$\sigma$} & \multicolumn{1}{c}{$\overline{x}$} & \multicolumn{1}{c|}{$\sigma$} & \multicolumn{1}{c}{$\overline{x}$} & \multicolumn{1}{c|}{$\sigma$}\\
  \hhline{~~|----------}
  \parbox[t]{3mm}{\multirow{10}{*}{\rotatebox[origin=c]{90}{number of rollouts}}} & 5 &\gr{191.8} 191.8 & \gr{191.8} 28.0 & \gr{161.2} \textbf{161.2} & \gr{161.2} 63.5 & \gr{152.0} \underline{\textbf{152.0}} & \gr{152.0} 67.4 & \gr{160.9} \textbf{160.9} & \gr{160.9} 63.5 & \gr{166.4} \textbf{166.4} & \gr{166.4} 58.1\\
  & 10\ & \gr{189.4} 189.4 & \gr{189.4} 34.4 & \gr{136.6} \underline{\textbf{136.6}} & \gr{136.6} 71.6 & \gr{139.2} \textbf{139.2} & \gr{139.2} 69.2 & \gr{150.2} \textbf{150.2} & \gr{150.2} 63.6 & \gr{151.7} \textbf{151.7} & \gr{151.7} 68.3\\
  & 20\ & \gr{170.5} 170.5 & \gr{170.5} 55.2 & \gr{130.9} \textbf{130.9} & \gr{130.9} 71.6 & \gr{129.5} \underline{\textbf{129.5}} & \gr{129.5} 72.5 & \gr{142.8} \textbf{142.8} & \gr{142.8} 65.7 & \gr{145.4} 145.4 & \gr{145.4} 67.8\\
  & 50\ & \gr{158.4} 158.4 & \gr{158.4} 58.6 & \gr{96.5} \textbf{96.5} & \gr{96.5} 67.8 & \gr{96.1} \underline{\textbf{96.1}} & \gr{96.1} 64.6 & \gr{96.3} \textbf{96.3} & \gr{96.3} 66.4 & \gr{115.1} \textbf{115.1} & \gr{115.1} 71.9\\
\end{tabular}
        \raggedleft Average cost  \hspace{1.3in}
    \end{minipage}
    \begin{minipage}[b]{2.8in}
\colorlet{maxcolor}{white}
\colorlet{mincolor}{blue!50}
\newcolumntype{C}[1]{>{\centering\arraybackslash}m{#1}}
\renewcommand\gr[1]{\gradient{#1}{-0.01}{0.50}{1.00}{maxcolor}{mincolor}}
\setlength{\tabcolsep}{2pt}
\setlength{\colwidth}{1.2cm}

\begin{tabular}[b]{|C{\colwidth}C{\colwidth}C{\colwidth}C{\colwidth}C{\colwidth}|}
  \multicolumn{5}{c}{less greedy $\quad \longleftarrow \alpha \longrightarrow \quad$ more greedy}\\
    \multicolumn{1}{c}{0/UCT} & \multicolumn{1}{c}{0.2} & \multicolumn{1}{c}{0.5} & \multicolumn{1}{c}{0.8} & \multicolumn{1}{c}{1}\\
    \hhline{|-----}
N/A & \textbf{2.02e-04} & \textbf{5.45e-06} & \textbf{1.78e-04} & \textbf{8.32e-04}\\
N/A & \textbf{4.76e-08} & \textbf{8.59e-08} & \textbf{5.80e-06} & \textbf{3.46e-05}\\
N/A & \textbf{2.11e-04} & \textbf{1.46e-04} & \textbf{5.82e-03} & 1.40e-02\\
N/A & \textbf{1.61e-08} & \textbf{5.55e-09} & \textbf{1.02e-08} & \textbf{8.57e-05}\\
N/A & \textbf{1.38e-13} & \textbf{3.51e-11} & \textbf{6.16e-15} & \textbf{8.57e-05}\\
N/A & \textbf{1.18e-16} & \textbf{1.01e-15} & \textbf{8.80e-13} & \textbf{8.57e-07}\\
N/A & \textbf{5.96e-12} & \textbf{3.33e-10} & \textbf{1.70e-10} & \textbf{1.13e-07}\\
N/A & \textbf{1.36e-09} & \textbf{3.64e-08} & \textbf{4.65e-09} & \textbf{6.28e-07}\\
N/A & \textbf{4.09e-11} & \textbf{3.78e-08} & \textbf{4.41e-10} & \textbf{2.70e-06}\\
N/A & \textbf{4.09e-09} & \textbf{2.44e-07} & \textbf{2.33e-06} & \textbf{1.03e-04}\\
  \hhline{|-----}
\end{tabular}
        \centering $p$-values
    \end{minipage}
    \begin{center}
        \captsize \textbf{(d)} \domain{prob\_blocksworld} problem \texttt{p4}. The landmark extraction algorithm $\mathrm{LM}^{\text{RHW}}$ generated 12 nontrivial landmarks for this problem. 
    \end{center}
    
    \begin{minipage}[b]{4.1in}
\colorlet{maxcolor}{white}
\colorlet{mincolor}{blue!50}
\newcolumntype{C}[1]{>{\raggedleft\arraybackslash}m{#1}}
\renewcommand\gr[1]{\gradient{#1}{16.1}{103.4}{190.7}{maxcolor}{mincolor}}
\setlength{\tabcolsep}{2pt}
\setlength{\colwidth}{0.7cm}
\begin{tabular}[b]{rr|C{\colwidth}C{\colwidth}|C{\colwidth}C{\colwidth}|C{\colwidth}C{\colwidth}|C{\colwidth}C{\colwidth}|C{\colwidth}C{\colwidth}|}
  & \multicolumn{1}{c}{} & \multicolumn{10}{c}{less greedy $\quad \longleftarrow \alpha \longrightarrow \quad$ more greedy}\\
    & \multicolumn{1}{c}{} & \multicolumn{2}{c}{0/UCT} & \multicolumn{2}{c}{0.2} & \multicolumn{2}{c}{0.5} & \multicolumn{2}{c}{0.8} & \multicolumn{2}{c}{1}\\
    \hhline{~~|----------}
  & \multicolumn{1}{c|}{} & \multicolumn{1}{c}{$\overline{x}$} & \multicolumn{1}{c|}{$\sigma$} & \multicolumn{1}{c}{$\overline{x}$} & \multicolumn{1}{c|}{$\sigma$} & \multicolumn{1}{c}{$\overline{x}$} & \multicolumn{1}{c|}{$\sigma$} & \multicolumn{1}{c}{$\overline{x}$} & \multicolumn{1}{c|}{$\sigma$} & \multicolumn{1}{c}{$\overline{x}$} & \multicolumn{1}{c|}{$\sigma$}\\
  \hhline{~~|----------}
  \parbox[t]{3mm}{\multirow{10}{*}{\rotatebox[origin=c]{90}{number of rollouts}}} & 5 &\gr{190.7} 190.7 & \gr{190.7} 34.8 & \gr{154.8} \textbf{154.8} & \gr{154.8} 69.6 & \gr{144.1} \underline{\textbf{144.1}} & \gr{144.1} 72.4 & \gr{150.7} \textbf{150.7} & \gr{150.7} 68.9 & \gr{163.0} \textbf{163.0} & \gr{163.0} 63.9\\
  & 10\ & \gr{179.3} 179.3 & \gr{179.3} 46.6 & \gr{113.7} \underline{\textbf{113.7}} & \gr{113.7} 75.9 & \gr{123.1} \textbf{123.1} & \gr{123.1} 75.2 & \gr{135.7} \textbf{135.7} & \gr{135.7} 72.8 & \gr{122.1} \textbf{122.1} & \gr{122.1} 79.9\\
  & 20\ & \gr{168.9} 168.9 & \gr{168.9} 55.2 & \gr{106.3} \underline{\textbf{106.3}} & \gr{106.3} 73.1 & \gr{108.3} \textbf{108.3} & \gr{108.3} 70.5 & \gr{109.6} \textbf{109.6} & \gr{109.6} 77.4 & \gr{133.3} \textbf{133.3} & \gr{133.3} 74.2\\
  & 50\ & \gr{149.0} 149.0 & \gr{149.0} 62.3 & \gr{62.8} \underline{\textbf{62.8}} & \gr{62.8} 54.0 & \gr{75.4} \textbf{75.4} & \gr{75.4} 63.1 & \gr{72.5} \textbf{72.5} & \gr{72.5} 62.9 & \gr{89.1} \textbf{89.1} & \gr{89.1} 67.6\\
\end{tabular}
        \raggedleft Average cost  \hspace{1.3in}
    \end{minipage}
    \begin{minipage}[b]{2.8in}
\colorlet{maxcolor}{white}
\colorlet{mincolor}{blue!50}
\newcolumntype{C}[1]{>{\centering\arraybackslash}m{#1}}
\renewcommand\gr[1]{\gradient{#1}{-0.01}{0.50}{1.00}{maxcolor}{mincolor}}
\setlength{\tabcolsep}{2pt}
\setlength{\colwidth}{1.2cm}
\begin{tabular}[b]{|C{\colwidth}C{\colwidth}C{\colwidth}C{\colwidth}C{\colwidth}|}
  \multicolumn{5}{c}{less greedy $\quad \longleftarrow \alpha \longrightarrow \quad$ more greedy}\\
    \multicolumn{1}{c}{0/UCT} & \multicolumn{1}{c}{0.2} & \multicolumn{1}{c}{0.5} & \multicolumn{1}{c}{0.8} & \multicolumn{1}{c}{1}\\
    \hhline{|-----}
N/A & \textbf{1.02e-04} & \textbf{1.44e-06} & \textbf{1.47e-05} & \textbf{1.23e-03}\\
N/A & \textbf{2.06e-09} & \textbf{1.56e-07} & \textbf{2.37e-05} & \textbf{3.11e-07}\\
N/A & \textbf{2.27e-08} & \textbf{2.99e-08} & \textbf{2.61e-07} & \textbf{1.08e-03}\\
N/A & \textbf{7.31e-16} & \textbf{2.99e-11} & \textbf{5.86e-12} & \textbf{8.57e-08}\\
N/A & \textbf{1.35e-13} & \textbf{2.56e-11} & \textbf{3.40e-11} & \textbf{1.90e-04}\\
N/A & \textbf{1.37e-09} & \textbf{6.15e-11} & \textbf{1.01e-06} & \textbf{5.00e-05}\\
N/A & \textbf{2.12e-08} & \textbf{1.46e-06} & \textbf{1.79e-05} & \textbf{9.97e-05}\\
N/A & \textbf{5.73e-06} & \textbf{7.46e-05} & \textbf{1.50e-03} & \textbf{8.43e-03}\\
N/A & \textbf{1.13e-03} & \textbf{1.12e-02} & \textbf{5.75e-03} & 2.01e-01\\
N/A & \textbf{1.27e-04} & \textbf{2.94e-04} & 2.16e-02 & 1.86e-01\\
  \hhline{|-----}
\end{tabular}
        \centering $p$-values
    \end{minipage}
    \begin{center}
        \captsize\textbf{(e)} \domain{prob\_blocksworld} problem \texttt{p5}. The landmark extraction algorithm $\mathrm{LM}^{\text{RHW}}$ generated 10 nontrivial landmarks for this problem. 
    \end{center}
    
    \begin{minipage}[b]{4.1in}
\colorlet{maxcolor}{white}
\colorlet{mincolor}{blue!50}
\newcolumntype{C}[1]{>{\raggedleft\arraybackslash}m{#1}}
\renewcommand\gr[1]{\gradient{#1}{190.3}{200}{200.0}{maxcolor}{mincolor}}
\setlength{\tabcolsep}{2pt}
\setlength{\colwidth}{0.7cm}

\begin{tabular}[b]{rr|C{\colwidth}C{\colwidth}|C{\colwidth}C{\colwidth}|C{\colwidth}C{\colwidth}|C{\colwidth}C{\colwidth}|C{\colwidth}C{\colwidth}|}
  & \multicolumn{1}{c}{} & \multicolumn{10}{c}{less greedy $\quad \longleftarrow \alpha \longrightarrow \quad$ more greedy}\\
    & \multicolumn{1}{c}{} & \multicolumn{2}{c}{0\tiny{/UCT}} & \multicolumn{2}{c}{0.2} & \multicolumn{2}{c}{0.5} & \multicolumn{2}{c}{0.8} & \multicolumn{2}{c}{1}\\
    \hhline{~~|----------}
  & \multicolumn{1}{c|}{} & \multicolumn{1}{c}{$\overline{x}$} & \multicolumn{1}{c|}{$\sigma$} & \multicolumn{1}{c}{$\overline{x}$} & \multicolumn{1}{c|}{$\sigma$} & \multicolumn{1}{c}{$\overline{x}$} & \multicolumn{1}{c|}{$\sigma$} & \multicolumn{1}{c}{$\overline{x}$} & \multicolumn{1}{c|}{$\sigma$} & \multicolumn{1}{c}{$\overline{x}$} & \multicolumn{1}{c|}{$\sigma$}\\
  \hhline{~~|----------}
  \parbox[t]{3mm}{\multirow{10}{*}{\rotatebox[origin=c]{90}{number of rollouts}}} & 5 &\gr{200.0} 200.0 & \gr{200.0} 0.0 & \gr{200.0} 200.0 & \gr{200.0} 0.0 & \gr{200.0} 200.0 & \gr{200.0} 0.0 & \gr{197.8} \underline{197.8} & \gr{197.8} 13.8 & \gr{200.0} 200.0 & \gr{200.0} 0.0\\
  & 10\ & \gr{200.0} 200.0 & \gr{200.0} 0.0 & \gr{200.0} 200.0 & \gr{200.0} 0.0 & \gr{198.2} \underline{198.2} & \gr{198.2} 15.9 & \gr{200.0} 200.0 & \gr{200.0} 0.0 & \gr{200.0} 200.0 & \gr{200.0} 0.0\\
  & 20\ & \gr{200.0} 200.0 & \gr{200.0} 0.0 & \gr{200.0} 200.0 & \gr{200.0} 0.0 & \gr{198.3} \underline{198.3} & \gr{198.3} 14.3 & \gr{199.3} 199.3 & \gr{199.3} 5.8 & \gr{200.0} 200.0 & \gr{200.0} 0.0\\
  & 50\ & \gr{200.0} 200.0 & \gr{200.0} 0.0 & \gr{191.8} 191.8 & \gr{191.8} 31.1 & \gr{197.0} 197.0 & \gr{197.0} 19.1 & \gr{190.3} \underline{190.3} & \gr{190.3} 33.7 & \gr{196.0} 196.0 & \gr{196.0} 19.8\\
  & 100\ & \gr{200.0} 200.0 & \gr{200.0} 0.0 & \gr{198.1} 198.1 & \gr{198.1} 13.8 & \gr{192.8} \underline{192.8} & \gr{192.8} 30.6 & \gr{198.3} 198.3 & \gr{198.3} 15.0 & \gr{198.5} 198.5 & \gr{198.5} 12.9\\
\end{tabular}
        \raggedleft Average cost  \hspace{1.3in}
    \end{minipage}
    \begin{minipage}[b]{2.8in}
\colorlet{maxcolor}{white}
\colorlet{mincolor}{blue!50}
\newcolumntype{C}[1]{>{\centering\arraybackslash}m{#1}}
\renewcommand\gr[1]{\gradient{#1}{-0.01}{0.50}{1.00}{maxcolor}{mincolor}}
\setlength{\tabcolsep}{2pt}
\setlength{\colwidth}{1.2cm}

\begin{tabular}[b]{|C{\colwidth}C{\colwidth}C{\colwidth}C{\colwidth}C{\colwidth}|}
  \multicolumn{5}{c}{less greedy $\quad \longleftarrow \alpha \longrightarrow \quad$ more greedy}\\
    \multicolumn{1}{c}{0/UCT} & \multicolumn{1}{c}{0.2} & \multicolumn{1}{c}{0.5} & \multicolumn{1}{c}{0.8} & \multicolumn{1}{c}{1}\\
    \hhline{|-----}
N/A & NaN & NaN & 1.74e-01 & NaN\\
N/A & NaN & 3.19e-01 & NaN & NaN\\
N/A & NaN & 3.19e-01 & 3.19e-01 & NaN\\
N/A & 2.37e-02 & 1.76e-01 & 1.41e-02 & 8.20e-02\\
N/A & 2.44e-01 & 4.38e-02 & 3.19e-01 & 3.19e-01\\
N/A & 1.58e-01 & NaN & NaN & NaN\\
N/A & 4.38e-02 & NaN & NaN & NaN\\
N/A & 1.57e-01 & 1.02e-01 & NaN & 1.66e-01\\
N/A & NaN & 1.57e-01 & 8.68e-02 & 1.57e-01\\
N/A & 3.21e-01 & 3.21e-01 & NaN & 3.21e-01\\
  \hhline{|-----}
\end{tabular}
        \centering $p$-values
    \end{minipage}
    \begin{center}
        \captsize \textbf{(f)} \domain{prob\_blocksworld} problem \texttt{p6}. The landmark extraction algorithm $\mathrm{LM}^{\text{RHW}}$ generated 20 nontrivial landmarks for this problem. 
    \end{center}
    \continuefigcaption{The left tables report the average cost (lower is better) of the solutions generated by \LAMP in the \domain{prob\_blocksworld} domain over 75 runs. Underlined values indicate the best performing $\alpha$-value in the row. The right tables report $p$-values comparing the average cost of solutions generated by LAMP with the average cost of solutions generated by standard UCT ($\alpha = 0$) at the same number of rollouts. In all tables, boldfaced values indicate where LAMP significantly dominated standard UCT at the same number of rollouts, with $p < 0.0125$, the adjusted level using a Bonferroni adjustment of $\alpha_{\textit{stat}} = 0.05/4$.}
\end{figure}

\begin{figure}
    \centering
    \begin{minipage}[b]{4.1in}
\colorlet{maxcolor}{white}
\colorlet{mincolor}{blue!50}
\newcolumntype{C}[1]{>{\raggedleft\arraybackslash}m{#1}}
\renewcommand\gr[1]{\gradient{#1}{199.0}{200.0}{201.0}{maxcolor}{mincolor}}
\setlength{\tabcolsep}{2pt}
\setlength{\colwidth}{0.7cm}

\begin{tabular}[b]{rr|C{\colwidth}C{\colwidth}|C{\colwidth}C{\colwidth}|C{\colwidth}C{\colwidth}|C{\colwidth}C{\colwidth}|C{\colwidth}C{\colwidth}|}
  & \multicolumn{1}{c}{} & \multicolumn{10}{c}{less greedy $\quad \longleftarrow \alpha \longrightarrow \quad$ more greedy}\\
    & \multicolumn{1}{c}{} & \multicolumn{2}{c}{0\tiny{/UCT}} & \multicolumn{2}{c}{0.2} & \multicolumn{2}{c}{0.5} & \multicolumn{2}{c}{0.8} & \multicolumn{2}{c}{1}\\
    \hhline{~~|----------}
  & \multicolumn{1}{c|}{} & \multicolumn{1}{c}{$\overline{x}$} & \multicolumn{1}{c|}{$\sigma$} & \multicolumn{1}{c}{$\overline{x}$} & \multicolumn{1}{c|}{$\sigma$} & \multicolumn{1}{c}{$\overline{x}$} & \multicolumn{1}{c|}{$\sigma$} & \multicolumn{1}{c}{$\overline{x}$} & \multicolumn{1}{c|}{$\sigma$} & \multicolumn{1}{c}{$\overline{x}$} & \multicolumn{1}{c|}{$\sigma$}\\
  \hhline{~~|----------}
  \parbox[t]{3mm}{\multirow{10}{*}{\rotatebox[origin=c]{90}{number of rollouts}}} & 5 &\gr{200.0} \underline{200.0} & \gr{200.0} 0.0 & \gr{200.0} \underline{200.0} & \gr{200.0} 0.0 & \gr{200.0} \underline{200.0} & \gr{200.0} 0.0 & \gr{200.0} \underline{200.0} & \gr{200.0} 0.0 & \gr{200.0} \underline{200.0} & \gr{200.0} 0.0\\
  & 10\ & \gr{200.0} \underline{200.0} & \gr{200.0} 0.0 & \gr{200.0} \underline{200.0} & \gr{200.0} 0.0 & \gr{200.0} \underline{200.0} & \gr{200.0} 0.0 & \gr{200.0} \underline{200.0} & \gr{200.0} 0.0 & \gr{200.0} \underline{200.0} & \gr{200.0} 0.0\\
  & 20\ & \gr{200.0} \underline{200.0} & \gr{200.0} 0.0 & \gr{200.0} \underline{200.0} & \gr{200.0} 0.0 & \gr{200.0} \underline{200.0} & \gr{200.0} 0.0 & \gr{200.0} \underline{200.0} & \gr{200.0} 0.0 & \gr{200.0} \underline{200.0} & \gr{200.0} 0.0\\
  & 50\ & \gr{200.0} \underline{200.0} & \gr{200.0} 0.0 & \gr{200.0} \underline{200.0} & \gr{200.0} 0.0 & \gr{200.0} \underline{200.0} & \gr{200.0} 0.0 & \gr{200.0} \underline{200.0} & \gr{200.0} 0.0 & \gr{200.0} \underline{200.0} & \gr{200.0} 0.0\\
\end{tabular}
        \raggedleft Average cost  \hspace{1.3in}
    \end{minipage}
    \begin{minipage}[b]{2.8in}
\colorlet{maxcolor}{white}
\colorlet{mincolor}{blue!50}
\newcolumntype{C}[1]{>{\centering\arraybackslash}m{#1}}
\renewcommand\gr[1]{\gradient{#1}{-0.01}{0.50}{1.00}{maxcolor}{mincolor}}
\setlength{\tabcolsep}{2pt}
\setlength{\colwidth}{1.2cm}

\begin{tabular}[b]{|C{\colwidth}C{\colwidth}C{\colwidth}C{\colwidth}C{\colwidth}|}
  \multicolumn{5}{c}{less greedy $\quad \longleftarrow \alpha \longrightarrow \quad$ more greedy}\\
    \multicolumn{1}{c}{0/UCT} & \multicolumn{1}{c}{0.2} & \multicolumn{1}{c}{0.5} & \multicolumn{1}{c}{0.8} & \multicolumn{1}{c}{1}\\
    \hhline{|-----}
N/A & NaN & NaN & NaN & NaN\\
N/A & NaN & NaN & NaN & NaN\\
N/A & NaN & NaN & NaN & NaN\\
N/A & NaN & NaN & NaN & NaN\\
N/A & NaN & NaN & NaN & NaN\\
N/A & NaN & NaN & NaN & NaN\\
N/A & NaN & NaN & NaN & NaN\\
N/A & NaN & NaN & NaN & NaN\\
N/A & NaN & NaN & NaN & NaN\\
N/A & NaN & NaN & NaN & NaN\\
  \hhline{|-----}
\end{tabular}
        \centering $p$-values
    \end{minipage}
    \begin{center}
        \captsize \textbf{(g)} \domain{prob\_blocksworld} problem \texttt{p7}. The landmark extraction algorithm $\mathrm{LM}^{\text{RHW}}$ generated 26 nontrivial landmarks for this problem. 
    \end{center}
    
    \begin{minipage}[b]{4.1in}
\colorlet{maxcolor}{white}
\colorlet{mincolor}{blue!50}
\newcolumntype{C}[1]{>{\raggedleft\arraybackslash}m{#1}}
\renewcommand\gr[1]{\gradient{#1}{140.0}{170.0}{200.0}{maxcolor}{mincolor}}
\setlength{\tabcolsep}{2pt}
\setlength{\colwidth}{0.7cm}

\begin{tabular}[b]{rr|C{\colwidth}C{\colwidth}|C{\colwidth}C{\colwidth}|C{\colwidth}C{\colwidth}|C{\colwidth}C{\colwidth}|C{\colwidth}C{\colwidth}|}
  & \multicolumn{1}{c}{} & \multicolumn{10}{c}{less greedy $\quad \longleftarrow \alpha \longrightarrow \quad$ more greedy}\\
    & \multicolumn{1}{c}{} & \multicolumn{2}{c}{0\tiny{/UCT}} & \multicolumn{2}{c}{0.2} & \multicolumn{2}{c}{0.5} & \multicolumn{2}{c}{0.8} & \multicolumn{2}{c}{1}\\
    \hhline{~~|----------}
  & \multicolumn{1}{c|}{} & \multicolumn{1}{c}{$\overline{x}$} & \multicolumn{1}{c|}{$\sigma$} & \multicolumn{1}{c}{$\overline{x}$} & \multicolumn{1}{c|}{$\sigma$} & \multicolumn{1}{c}{$\overline{x}$} & \multicolumn{1}{c|}{$\sigma$} & \multicolumn{1}{c}{$\overline{x}$} & \multicolumn{1}{c|}{$\sigma$} & \multicolumn{1}{c}{$\overline{x}$} & \multicolumn{1}{c|}{$\sigma$}\\
  \hhline{~~|----------}
  \parbox[t]{3mm}{\multirow{10}{*}{\rotatebox[origin=c]{90}{number of rollouts}}} & 5 &\gr{200.0} 200.0 & \gr{200.0} 0.0 & \gr{200.0} 200.0 & \gr{200.0} 0.0 & \gr{199.4} \underline{199.4} & \gr{199.4} 4.8 & \gr{200.0} 200.0 & \gr{200.0} 0.0 & \gr{200.0} 200.0 & \gr{200.0} 0.0\\
  & 10\ & \gr{200.0} 200.0 & \gr{200.0} 0.0 & \gr{199.7} \underline{199.7} & \gr{199.7} 2.9 & \gr{200.0} 200.0 & \gr{200.0} 0.0 & \gr{200.0} 200.0 & \gr{200.0} 0.0 & \gr{200.0} 200.0 & \gr{200.0} 0.0\\
  & 20\ & \gr{200.0} 200.0 & \gr{200.0} 0.0 & \gr{193.3} 193.3 & \gr{193.3} 26.1 & \gr{191.3} \underline{\textbf{191.3}} & \gr{191.3} 28.8 & \gr{194.9} 194.9 & \gr{194.9} 22.0 & \gr{196.7} 196.7 & \gr{196.7} 16.8\\
  & 50\ & \gr{200.0} 200.0 & \gr{200.0} 0.0 & \gr{184.4} \underline{\textbf{184.4}} & \gr{184.4} 38.9 & \gr{184.4} \textbf{184.4} & \gr{184.4} 38.7 & \gr{191.7} \textbf{191.7} & \gr{191.7} 28.3 & \gr{189.1} \textbf{189.1} & \gr{189.1} 32.9\\
\end{tabular}
        \raggedleft Average cost  \hspace{1.3in}
    \end{minipage}
    \begin{minipage}[b]{2.8in}
\colorlet{maxcolor}{white}
\colorlet{mincolor}{blue!50}
\newcolumntype{C}[1]{>{\centering\arraybackslash}m{#1}}
\renewcommand\gr[1]{\gradient{#1}{-0.01}{0.50}{1.00}{maxcolor}{mincolor}}
\setlength{\tabcolsep}{2pt}
\setlength{\colwidth}{1.2cm}

\begin{tabular}[b]{|C{\colwidth}C{\colwidth}C{\colwidth}C{\colwidth}C{\colwidth}|}
  \multicolumn{5}{c}{less greedy $\quad \longleftarrow \alpha \longrightarrow \quad$ more greedy}\\
    \multicolumn{1}{c}{0/UCT} & \multicolumn{1}{c}{0.2} & \multicolumn{1}{c}{0.5} & \multicolumn{1}{c}{0.8} & \multicolumn{1}{c}{1}\\
    \hhline{|-----}
N/A & NaN & 3.19e-01 & NaN & NaN\\
N/A & 3.19e-01 & NaN & NaN & NaN\\
N/A & 2.64e-02 & \textbf{9.41e-03} & 4.69e-02 & 8.70e-02\\
N/A & \textbf{6.59e-04} & \textbf{6.42e-04} & \textbf{1.24e-02} & \textbf{4.79e-03}\\
N/A & 4.71e-02 & \textbf{7.05e-03} & 4.36e-02 & 4.46e-02\\
N/A & 4.40e-02 & 3.19e-01 & 4.51e-02 & 6.62e-02\\
N/A & 4.48e-02 & 3.19e-01 & 1.57e-01 & 4.26e-02\\
N/A & 5.29e-02 & 4.71e-02 & 1.69e-01 & 2.32e-02\\
N/A & \textbf{3.52e-03} & \textbf{1.04e-03} & \textbf{6.66e-03} & \textbf{7.84e-06}\\
N/A & \textbf{3.69e-11} & \textbf{9.47e-08} & \textbf{4.13e-04} & \textbf{9.49e-08}\\
  \hhline{|-----}
\end{tabular}
        \centering $p$-values
    \end{minipage}
    \begin{center}
        \captsize\textbf{(h)} \domain{prob\_blocksworld} problem \texttt{p8}. The landmark extraction algorithm $\mathrm{LM}^{\text{RHW}}$ generated 24 nontrivial landmarks for this problem. 
    \end{center}
    
    \begin{minipage}[b]{4.1in}
\colorlet{maxcolor}{white}
\colorlet{mincolor}{blue!50}
\newcolumntype{C}[1]{>{\raggedleft\arraybackslash}m{#1}}
\renewcommand\gr[1]{\gradient{#1}{138.6}{200}{200.0}{maxcolor}{mincolor}}
\setlength{\tabcolsep}{2pt}
\setlength{\colwidth}{0.7cm}

\begin{tabular}[b]{rr|C{\colwidth}C{\colwidth}|C{\colwidth}C{\colwidth}|C{\colwidth}C{\colwidth}|C{\colwidth}C{\colwidth}|C{\colwidth}C{\colwidth}|}
  & \multicolumn{1}{c}{} & \multicolumn{10}{c}{less greedy $\quad \longleftarrow \alpha \longrightarrow \quad$ more greedy}\\
    & \multicolumn{1}{c}{} & \multicolumn{2}{c}{0/UCT} & \multicolumn{2}{c}{0.2} & \multicolumn{2}{c}{0.5} & \multicolumn{2}{c}{0.8} & \multicolumn{2}{c}{1}\\
    \hhline{~~|----------}
  & \multicolumn{1}{c|}{} & \multicolumn{1}{c}{$\overline{x}$} & \multicolumn{1}{c|}{$\sigma$} & \multicolumn{1}{c}{$\overline{x}$} & \multicolumn{1}{c|}{$\sigma$} & \multicolumn{1}{c}{$\overline{x}$} & \multicolumn{1}{c|}{$\sigma$} & \multicolumn{1}{c}{$\overline{x}$} & \multicolumn{1}{c|}{$\sigma$} & \multicolumn{1}{c}{$\overline{x}$} & \multicolumn{1}{c|}{$\sigma$}\\
  \hhline{~~|----------}
  \parbox[t]{3mm}{\multirow{10}{*}{\rotatebox[origin=c]{90}{number of rollouts}}} & 5 &\gr{200.0} 200.0 & \gr{200.0} 0.0 & \gr{199.3} 199.3 & \gr{199.3} 6.0 & \gr{200.0} 200.0 & \gr{200.0} 0.0 & \gr{196.9} \underline{196.9} & \gr{196.9} 16.0 & \gr{200.0} 200.0 & \gr{200.0} 0.0\\
  & 10\ & \gr{200.0} 200.0 & \gr{200.0} 0.0 & \gr{196.9} 196.9 & \gr{196.9} 18.9 & \gr{193.1} \underline{193.1} & \gr{193.1} 26.6 & \gr{197.0} 197.0 & \gr{197.0} 18.3 & \gr{200.0} 200.0 & \gr{200.0} 0.0\\
  & 20\ & \gr{200.0} 200.0 & \gr{200.0} 0.0 & \gr{198.0} 198.0 & \gr{198.0} 16.4 & \gr{198.5} 198.5 & \gr{198.5} 12.7 & \gr{195.3} \underline{195.3} & \gr{195.3} 23.7 & \gr{197.0} 197.0 & \gr{197.0} 18.5\\
  & 50\ & \gr{200.0} 200.0 & \gr{200.0} 0.0 & \gr{181.8} \underline{\textbf{181.8}} & \gr{181.8} 45.1 & \gr{185.8} \textbf{185.8} & \gr{185.8} 42.0 & \gr{184.7} \textbf{184.7} & \gr{184.7} 42.1 & \gr{188.4} \textbf{188.4} & \gr{188.4} 36.9\\
  & 100\ & \gr{200.0} 200.0 & \gr{200.0} 0.0 & \gr{195.5} 195.5 & \gr{195.5} 23.2 & \gr{191.6} 191.6 & \gr{191.6} 33.3 & \gr{191.5} \underline{191.5} & \gr{191.5} 32.9 & \gr{194.8} 194.8 & \gr{194.8} 25.7\\
\end{tabular}%
        \raggedleft Average cost  \hspace{1.3in}
    \end{minipage}
    \begin{minipage}[b]{2.8in}
\colorlet{maxcolor}{white}
\colorlet{mincolor}{blue!50}
\newcolumntype{C}[1]{>{\centering\arraybackslash}m{#1}}
\renewcommand\gr[1]{\gradient{#1}{-0.01}{0.50}{1.00}{maxcolor}{mincolor}}
\setlength{\tabcolsep}{2pt}
\setlength{\colwidth}{1.2cm}

\begin{tabular}[b]{|C{\colwidth}C{\colwidth}C{\colwidth}C{\colwidth}C{\colwidth}|}
  \multicolumn{5}{c}{less greedy $\quad \longleftarrow \alpha \longrightarrow \quad$ more greedy}\\
    \multicolumn{1}{c}{0/UCT} & \multicolumn{1}{c}{0.2} & \multicolumn{1}{c}{0.5} & \multicolumn{1}{c}{0.8} & \multicolumn{1}{c}{1}\\
    \hhline{|-----}
N/A & 3.19e-01 & NaN & 1.01e-01 & NaN\\
N/A & 1.57e-01 & 2.66e-02 & 1.57e-01 & NaN\\
N/A & 2.99e-01 & 3.19e-01 & 8.52e-02 & 1.58e-01\\
N/A & \textbf{6.21e-04} & \textbf{3.90e-03} & \textbf{2.02e-03} & \textbf{7.12e-03}\\
N/A & 9.51e-02 & 3.04e-02 & 2.71e-02 & 8.28e-02\\
N/A & 6.74e-02 & 1.26e-01 & 1.10e-01 & 1.82e-01\\
N/A & 8.51e-02 & 8.17e-02 & 1.60e-01 & 4.58e-02\\
N/A & 1.29e-02 & 4.34e-02 & 5.34e-02 & 1.64e-02\\
N/A & \textbf{3.38e-04} & \textbf{5.12e-04} & \textbf{4.97e-04} & \textbf{1.23e-02}\\
N/A & \textbf{2.47e-11} & \textbf{5.76e-09} & \textbf{6.73e-09} & \textbf{8.39e-08}\\
  \hhline{|-----}
\end{tabular}%
        \centering $p$-values
    \end{minipage}
    \begin{center}
        \captsize \textbf{(i)} \domain{prob\_blocksworld} problem \texttt{p9}. The landmark extraction algorithm $\mathrm{LM}^{\text{RHW}}$ generated 22 nontrivial landmarks for this problem. 
    \end{center}
    \continuefigcaption{The left tables report the average cost (lower is better) of the solutions generated by \LAMP in the \domain{prob\_blocksworld} domain over 75 runs. Underlined values indicate the best performing $\alpha$-value in the row. The right tables report $p$-values comparing the average cost of solutions generated by LAMP with the average cost of solutions generated by standard UCT ($\alpha = 0$) at the same number of rollouts. In all tables, boldfaced values indicate where LAMP significantly dominated standard UCT at the same number of rollouts, with $p < 0.0125$, the adjusted level using a Bonferroni adjustment of $\alpha_{\textit{stat}} = 0.05/4$.}
\end{figure}

\begin{figure}
    \centering
    \begin{minipage}[b]{4.1in}
\colorlet{maxcolor}{white}
\colorlet{mincolor}{blue!50}
\newcolumntype{C}[1]{>{\raggedleft\arraybackslash}m{#1}}
\renewcommand\gr[1]{\gradient{#1}{14.8}{101.3}{187.9}{maxcolor}{mincolor}}
\setlength{\tabcolsep}{2pt}
\setlength{\colwidth}{0.7cm}

\begin{tabular}[b]{rr|C{\colwidth}C{\colwidth}|C{\colwidth}C{\colwidth}|C{\colwidth}C{\colwidth}|C{\colwidth}C{\colwidth}|C{\colwidth}C{\colwidth}|}
  & \multicolumn{1}{c}{} & \multicolumn{10}{c}{less greedy $\quad \longleftarrow \alpha \longrightarrow \quad$ more greedy}\\
    & \multicolumn{1}{c}{} & \multicolumn{2}{c}{0/UCT} & \multicolumn{2}{c}{0.2} & \multicolumn{2}{c}{0.5} & \multicolumn{2}{c}{0.8} & \multicolumn{2}{c}{1}\\
    \hhline{~~|----------}
  & \multicolumn{1}{c|}{} & \multicolumn{1}{c}{$\overline{x}$} & \multicolumn{1}{c|}{$\sigma$} & \multicolumn{1}{c}{$\overline{x}$} & \multicolumn{1}{c|}{$\sigma$} & \multicolumn{1}{c}{$\overline{x}$} & \multicolumn{1}{c|}{$\sigma$} & \multicolumn{1}{c}{$\overline{x}$} & \multicolumn{1}{c|}{$\sigma$} & \multicolumn{1}{c}{$\overline{x}$} & \multicolumn{1}{c|}{$\sigma$}\\
  \hhline{~~|----------}
  \parbox[t]{3mm}{\multirow{10}{*}{\rotatebox[origin=c]{90}{number of rollouts}}} & 5 &\gr{184.7} 184.7 & \gr{184.7} 40.7 & \gr{85.1} \textbf{85.1} & \gr{85.1} 60.1 & \gr{72.9} \underline{\textbf{72.9}} & \gr{72.9} 58.8 & \gr{87.1} \textbf{87.1} & \gr{87.1} 65.8 & \gr{80.2} \textbf{80.2} & \gr{80.2} 67.3\\
  & 10\ & \gr{187.9} 187.9 & \gr{187.9} 36.1 & \gr{58.8} \textbf{58.8} & \gr{58.8} 52.7 & \gr{52.6} \textbf{52.6} & \gr{52.6} 44.2 & \gr{43.1} \underline{\textbf{43.1}} & \gr{43.1} 32.3 & \gr{64.7} \textbf{64.7} & \gr{64.7} 53.4\\
  & 20\ & \gr{187.7} 187.7 & \gr{187.7} 40.6 & \gr{32.9} \textbf{32.9} & \gr{32.9} 15.5 & \gr{33.8} \textbf{33.8} & \gr{33.8} 26.2 & \gr{32.3} \underline{\textbf{32.3}} & \gr{32.3} 24.8 & \gr{32.4} \textbf{32.4} & \gr{32.4} 22.6\\
  & 50\ & \gr{164.8} 164.8 & \gr{164.8} 53.7 & \gr{26.1} \underline{\textbf{26.1}} & \gr{26.1} 14.0 & \gr{29.1} \textbf{29.1} & \gr{29.1} 22.2 & \gr{27.4} \textbf{27.4} & \gr{27.4} 22.0 & \gr{31.2} \textbf{31.2} & \gr{31.2} 25.3\\
\end{tabular}
        \raggedleft Average cost  \hspace{1.3in}
    \end{minipage}
    \begin{minipage}[b]{2.8in}
\colorlet{maxcolor}{white}
\colorlet{mincolor}{blue!50}
\newcolumntype{C}[1]{>{\centering\arraybackslash}m{#1}}
\renewcommand\gr[1]{\gradient{#1}{-0.01}{0.50}{1.00}{maxcolor}{mincolor}}
\setlength{\tabcolsep}{2pt}
\setlength{\colwidth}{1.2cm}

\begin{tabular}[b]{|C{\colwidth}C{\colwidth}C{\colwidth}C{\colwidth}C{\colwidth}|}
  \multicolumn{5}{c}{less greedy $\quad \longleftarrow \alpha \longrightarrow \quad$ more greedy}\\
    \multicolumn{1}{c}{0/UCT} & \multicolumn{1}{c}{0.2} & \multicolumn{1}{c}{0.5} & \multicolumn{1}{c}{0.8} & \multicolumn{1}{c}{1}\\
    \hhline{|-----}
N/A & \textbf{2.78e-23} & \textbf{1.11e-27} & \textbf{1.00e-20} & \textbf{2.78e-22}\\
N/A & \textbf{7.79e-38} & \textbf{3.70e-45} & \textbf{9.02e-57} & \textbf{1.82e-35}\\
N/A & \textbf{2.43e-66} & \textbf{3.27e-60} & \textbf{1.46e-61} & \textbf{8.82e-63}\\
N/A & \textbf{1.06e-47} & \textbf{2.06e-44} & \textbf{4.14e-45} & \textbf{1.09e-42}\\
N/A & \textbf{5.55e-34} & \textbf{2.78e-36} & \textbf{1.58e-32} & \textbf{6.24e-33}\\
N/A & \textbf{3.99e-22} & \textbf{7.97e-23} & \textbf{1.00e-23} & \textbf{1.52e-21}\\
N/A & \textbf{4.15e-18} & \textbf{2.50e-20} & \textbf{1.92e-19} & \textbf{2.66e-20}\\
N/A & \textbf{7.41e-19} & \textbf{4.76e-19} & \textbf{4.60e-22} & \textbf{8.49e-17}\\
N/A & \textbf{2.64e-14} & \textbf{1.42e-18} & \textbf{2.39e-21} & \textbf{2.62e-11}\\
N/A & \textbf{3.66e-09} & \textbf{1.11e-13} & \textbf{1.84e-22} & \textbf{1.40e-08}\\
  \hhline{|-----}
\end{tabular}
        \centering $p$-values
    \end{minipage}
    \begin{center}
        \captsize \textbf{(a)} \domain{elevators} problem \texttt{p1}. The landmark extraction algorithm $\mathrm{LM}^{\text{RHW}}$ generated 10 nontrivial landmarks for this problem. 
    \end{center}

    \begin{minipage}[b]{4.1in}
\colorlet{maxcolor}{white}
\colorlet{mincolor}{blue!50}
\newcolumntype{C}[1]{>{\raggedleft\arraybackslash}m{#1}}
\renewcommand\gr[1]{\gradient{#1}{8.5}{42.5}{76.5}{maxcolor}{mincolor}}
\setlength{\tabcolsep}{2pt}
\setlength{\colwidth}{0.7cm}

\begin{tabular}[b]{rr|C{\colwidth}C{\colwidth}|C{\colwidth}C{\colwidth}|C{\colwidth}C{\colwidth}|C{\colwidth}C{\colwidth}|C{\colwidth}C{\colwidth}|}
  & \multicolumn{1}{c}{} & \multicolumn{10}{c}{less greedy $\quad \longleftarrow \alpha \longrightarrow \quad$ more greedy}\\
    & \multicolumn{1}{c}{} & \multicolumn{2}{c}{0/UCT} & \multicolumn{2}{c}{0.2} & \multicolumn{2}{c}{0.5} & \multicolumn{2}{c}{0.8} & \multicolumn{2}{c}{1}\\
    \hhline{~~|----------}
  & \multicolumn{1}{c|}{} & \multicolumn{1}{c}{$\overline{x}$} & \multicolumn{1}{c|}{$\sigma$} & \multicolumn{1}{c}{$\overline{x}$} & \multicolumn{1}{c|}{$\sigma$} & \multicolumn{1}{c}{$\overline{x}$} & \multicolumn{1}{c|}{$\sigma$} & \multicolumn{1}{c}{$\overline{x}$} & \multicolumn{1}{c|}{$\sigma$} & \multicolumn{1}{c}{$\overline{x}$} & \multicolumn{1}{c|}{$\sigma$}\\
  \hhline{~~|----------}
  \parbox[t]{3mm}{\multirow{10}{*}{\rotatebox[origin=c]{90}{number of rollouts}}} & 5 &\gr{76.5} 76.5 & \gr{76.5} 56.7 & \gr{37.3} \textbf{37.3} & \gr{37.3} 46.0 & \gr{40.3} \textbf{40.3} & \gr{40.3} 47.5 & \gr{27.1} \underline{\textbf{27.1}} & \gr{27.1} 25.6 & \gr{36.2} \textbf{36.2} & \gr{36.2} 41.2\\
  & 10\ & \gr{40.2} 40.2 & \gr{40.2} 31.8 & \gr{19.1} \textbf{19.1} & \gr{19.1} 19.3 & \gr{14.3} \textbf{14.3} & \gr{14.3} 8.6 & \gr{13.8} \underline{\textbf{13.8}} & \gr{13.8} 9.3 & \gr{14.7} \textbf{14.7} & \gr{14.7} 14.5\\
  & 20\ & \gr{23.2} 23.2 & \gr{23.2} 16.2 & \gr{10.3} \underline{\textbf{10.3}} & \gr{10.3} 2.4 & \gr{10.7} \textbf{10.7} & \gr{10.7} 4.2 & \gr{10.6} \textbf{10.6} & \gr{10.6} 5.5 & \gr{11.3} \textbf{11.3} & \gr{11.3} 10.5\\
  & 50\ & \gr{12.9} 12.9 & \gr{12.9} 4.4 & \gr{9.2} \textbf{9.2} & \gr{9.2} 2.7 & \gr{9.1} \underline{\textbf{9.1}} & \gr{9.1} 1.6 & \gr{9.2} \textbf{9.2} & \gr{9.2} 1.3 & \gr{9.4} \textbf{9.4} & \gr{9.4} 1.6\\
\end{tabular}
        \raggedleft Average cost  \hspace{1.3in}
    \end{minipage}
    \begin{minipage}[b]{2.8in}
\colorlet{maxcolor}{white}
\colorlet{mincolor}{blue!50}
\newcolumntype{C}[1]{>{\centering\arraybackslash}m{#1}}
\renewcommand\gr[1]{\gradient{#1}{-0.01}{0.50}{1.00}{maxcolor}{mincolor}}
\setlength{\tabcolsep}{2pt}
\setlength{\colwidth}{1.2cm}

\begin{tabular}[b]{|C{\colwidth}C{\colwidth}C{\colwidth}C{\colwidth}C{\colwidth}|}
  \multicolumn{5}{c}{less greedy $\quad \longleftarrow \alpha \longrightarrow \quad$ more greedy}\\
    \multicolumn{1}{c}{0/UCT} & \multicolumn{1}{c}{0.2} & \multicolumn{1}{c}{0.5} & \multicolumn{1}{c}{0.8} & \multicolumn{1}{c}{1}\\
    \hhline{|-----}
N/A & \textbf{7.36e-06} & \textbf{3.98e-05} & \textbf{1.59e-10} & \textbf{1.81e-06}\\
N/A & \textbf{2.23e-06} & \textbf{2.31e-10} & \textbf{1.34e-10} & \textbf{2.97e-09}\\
N/A & \textbf{2.98e-10} & \textbf{1.45e-09} & \textbf{2.42e-09} & \textbf{4.01e-07}\\
N/A & \textbf{3.85e-09} & \textbf{6.54e-11} & \textbf{8.42e-11} & \textbf{1.29e-09}\\
N/A & \textbf{3.01e-07} & \textbf{1.68e-06} & \textbf{7.35e-05} & \textbf{7.67e-04}\\
N/A & \textbf{1.48e-08} & 2.80e-02 & \textbf{1.67e-05} & 6.87e-02\\
N/A & \textbf{5.18e-07} & \textbf{2.93e-03} & \textbf{6.21e-03} & \textbf{8.70e-05}\\
N/A & \textbf{1.42e-04} & 3.03e-02 & \textbf{6.70e-05} & \textbf{4.29e-07}\\
N/A & \textbf{3.86e-05} & \textbf{2.77e-05} & \textbf{4.98e-06} & \textbf{4.04e-08}\\
N/A & \textbf{3.13e-03} & 6.57e-02 & \textbf{4.70e-03} & \textbf{1.65e-03}\\
  \hhline{|-----}
\end{tabular}
        \centering $p$-values
    \end{minipage}
    \begin{center}
        \captsize \textbf{(b)} \domain{elevators} problem \texttt{p2}. The landmark extraction algorithm $\mathrm{LM}^{\text{RHW}}$ generated 6 nontrivial landmarks for this problem. 
    \end{center}

    \begin{minipage}[b]{4.1in}
\colorlet{maxcolor}{white}
\colorlet{mincolor}{blue!50}
\newcolumntype{C}[1]{>{\raggedleft\arraybackslash}m{#1}}
\renewcommand\gr[1]{\gradient{#1}{17.2}{107.2}{197.1}{maxcolor}{mincolor}}
\setlength{\tabcolsep}{2pt}
\setlength{\colwidth}{0.7cm}

\begin{tabular}[b]{rr|C{\colwidth}C{\colwidth}|C{\colwidth}C{\colwidth}|C{\colwidth}C{\colwidth}|C{\colwidth}C{\colwidth}|C{\colwidth}C{\colwidth}|}
  & \multicolumn{1}{c}{} & \multicolumn{10}{c}{less greedy $\quad \longleftarrow \alpha \longrightarrow \quad$ more greedy}\\
    & \multicolumn{1}{c}{} & \multicolumn{2}{c}{0/UCT} & \multicolumn{2}{c}{0.2} & \multicolumn{2}{c}{0.5} & \multicolumn{2}{c}{0.8} & \multicolumn{2}{c}{1}\\
    \hhline{~~|----------}
  & \multicolumn{1}{c|}{} & \multicolumn{1}{c}{$\overline{x}$} & \multicolumn{1}{c|}{$\sigma$} & \multicolumn{1}{c}{$\overline{x}$} & \multicolumn{1}{c|}{$\sigma$} & \multicolumn{1}{c}{$\overline{x}$} & \multicolumn{1}{c|}{$\sigma$} & \multicolumn{1}{c}{$\overline{x}$} & \multicolumn{1}{c|}{$\sigma$} & \multicolumn{1}{c}{$\overline{x}$} & \multicolumn{1}{c|}{$\sigma$}\\
  \hhline{~~|----------}
  \parbox[t]{3mm}{\multirow{10}{*}{\rotatebox[origin=c]{90}{number of rollouts}}} & 5 &\gr{195.3} 195.3 & \gr{195.3} 21.5 & \gr{78.8} \textbf{78.8} & \gr{78.8} 65.1 & \gr{68.6} \textbf{68.6} & \gr{68.6} 51.5 & \gr{68.1} \underline{\textbf{68.1}} & \gr{68.1} 54.0 & \gr{73.9} \textbf{73.9} & \gr{73.9} 60.2\\
  & 10\ & \gr{197.1} 197.1 & \gr{197.1} 18.2 & \gr{80.9} \textbf{80.9} & \gr{80.9} 67.6 & \gr{83.2} \textbf{83.2} & \gr{83.2} 69.7 & \gr{98.3} \textbf{98.3} & \gr{98.3} 73.8 & \gr{73.5} \underline{\textbf{73.5}} & \gr{73.5} 70.2\\
  & 20\ & \gr{195.6} 195.6 & \gr{195.6} 26.5 & \gr{63.1} \textbf{63.1} & \gr{63.1} 65.5 & \gr{56.7} \underline{\textbf{56.7}} & \gr{56.7} 56.0 & \gr{65.4} \textbf{65.4} & \gr{65.4} 67.4 & \gr{64.4} \textbf{64.4} & \gr{64.4} 66.8\\
  & 50\ & \gr{195.1} 195.1 & \gr{195.1} 24.5 & \gr{70.3} \textbf{70.3} & \gr{70.3} 76.5 & \gr{63.5} \textbf{63.5} & \gr{63.5} 71.0 & \gr{57.4} \underline{\textbf{57.4}} & \gr{57.4} 69.4 & \gr{88.4} \textbf{88.4} & \gr{88.4} 83.3\\
\end{tabular}
        \raggedleft Average cost  \hspace{1.3in}
    \end{minipage}
    \begin{minipage}[b]{2.8in}
\colorlet{maxcolor}{white}
\colorlet{mincolor}{blue!50}
\newcolumntype{C}[1]{>{\centering\arraybackslash}m{#1}}
\renewcommand\gr[1]{\gradient{#1}{-0.01}{0.50}{1.00}{maxcolor}{mincolor}}
\setlength{\tabcolsep}{2pt}
\setlength{\colwidth}{1.2cm}

\begin{tabular}[b]{|C{\colwidth}C{\colwidth}C{\colwidth}C{\colwidth}C{\colwidth}|}
  \multicolumn{5}{c}{less greedy $\quad \longleftarrow \alpha \longrightarrow \quad$ more greedy}\\
    \multicolumn{1}{c}{0/UCT} & \multicolumn{1}{c}{0.2} & \multicolumn{1}{c}{0.5} & \multicolumn{1}{c}{0.8} & \multicolumn{1}{c}{1}\\
    \hhline{|-----}
N/A & \textbf{9.70e-31} & \textbf{4.67e-43} & \textbf{2.06e-41} & \textbf{3.26e-35}\\
N/A & \textbf{7.08e-30} & \textbf{4.40e-28} & \textbf{1.29e-21} & \textbf{6.97e-31}\\
N/A & \textbf{1.14e-34} & \textbf{1.61e-42} & \textbf{5.66e-33} & \textbf{1.37e-33}\\
N/A & \textbf{1.93e-27} & \textbf{5.84e-32} & \textbf{1.31e-34} & \textbf{5.47e-20}\\
N/A & \textbf{2.26e-35} & \textbf{2.42e-46} & \textbf{2.06e-41} & \textbf{2.36e-40}\\
N/A & \textbf{1.26e-37} & \textbf{6.31e-40} & \textbf{4.42e-46} & \textbf{5.37e-43}\\
N/A & \textbf{5.50e-34} & \textbf{4.90e-27} & \textbf{5.25e-27} & \textbf{1.48e-30}\\
N/A & \textbf{7.06e-12} & \textbf{2.97e-20} & \textbf{5.98e-17} & \textbf{2.06e-18}\\
N/A & \textbf{8.01e-10} & \textbf{3.00e-12} & \textbf{1.84e-12} & \textbf{3.03e-10}\\
N/A & \textbf{1.51e-18} & \textbf{1.50e-16} & \textbf{2.11e-20} & \textbf{9.87e-18}\\
  \hhline{|-----}
\end{tabular}
        \centering $p$-values
    \end{minipage}
    \begin{center}
        \captsize \textbf{(c)} \domain{elevators} problem \texttt{p3}. The landmark extraction algorithm $\mathrm{LM}^{\text{RHW}}$ generated 11 nontrivial landmarks for this problem. 
    \end{center}
    \caption{The left tables report the average cost (lower is better) of the solutions generated by \LAMP in the \domain{elevators} domain over 75 runs. Underlined values indicate the best performing $\alpha$-value in the row. The right tables report $p$-values comparing the average cost of solutions generated by LAMP with the average cost of solutions generated by standard UCT ($\alpha = 0$) at the same number of rollouts. In all tables, boldfaced values indicate where LAMP significantly dominated standard UCT at the same number of rollouts, with $p < 0.0125$, the adjusted level using a Bonferroni adjustment of $\alpha_{\textit{stat}} = 0.05/4$.}
    \label{fig:elevators_full}
\end{figure}

\begin{figure}
    \centering
    \begin{minipage}[b]{4.1in}
\colorlet{maxcolor}{white}
\colorlet{mincolor}{blue!50}
\newcolumntype{C}[1]{>{\raggedleft\arraybackslash}m{#1}}
\renewcommand\gr[1]{\gradient{#1}{13.8}{71.6}{129.3}{maxcolor}{mincolor}}
\setlength{\tabcolsep}{2pt}
\setlength{\colwidth}{0.7cm}

\begin{tabular}[b]{rr|C{\colwidth}C{\colwidth}|C{\colwidth}C{\colwidth}|C{\colwidth}C{\colwidth}|C{\colwidth}C{\colwidth}|C{\colwidth}C{\colwidth}|}
  & \multicolumn{1}{c}{} & \multicolumn{10}{c}{less greedy $\quad \longleftarrow \alpha \longrightarrow \quad$ more greedy}\\
    & \multicolumn{1}{c}{} & \multicolumn{2}{c}{0/UCT} & \multicolumn{2}{c}{0.2} & \multicolumn{2}{c}{0.5} & \multicolumn{2}{c}{0.8} & \multicolumn{2}{c}{1}\\
    \hhline{~~|----------}
  & \multicolumn{1}{c|}{} & \multicolumn{1}{c}{$\overline{x}$} & \multicolumn{1}{c|}{$\sigma$} & \multicolumn{1}{c}{$\overline{x}$} & \multicolumn{1}{c|}{$\sigma$} & \multicolumn{1}{c}{$\overline{x}$} & \multicolumn{1}{c|}{$\sigma$} & \multicolumn{1}{c}{$\overline{x}$} & \multicolumn{1}{c|}{$\sigma$} & \multicolumn{1}{c}{$\overline{x}$} & \multicolumn{1}{c|}{$\sigma$}\\
  \hhline{~~|----------}
  \parbox[t]{3mm}{\multirow{10}{*}{\rotatebox[origin=c]{90}{number of rollouts}}} & 5 &\gr{129.3} 129.3 & \gr{129.3} 70.8 & \gr{66.6} \underline{\textbf{66.6}} & \gr{66.6} 62.1 & \gr{72.1} \textbf{72.1} & \gr{72.1} 66.3 & \gr{78.8} \textbf{78.8} & \gr{78.8} 69.0 & \gr{73.7} \textbf{73.7} & \gr{73.7} 71.9\\
  & 10\ & \gr{113.9} 113.9 & \gr{113.9} 68.7 & \gr{45.2} \textbf{45.2} & \gr{45.2} 51.8 & \gr{46.3} \textbf{46.3} & \gr{46.3} 57.4 & \gr{44.0} \underline{\textbf{44.0}} & \gr{44.0} 56.1 & \gr{52.7} \textbf{52.7} & \gr{52.7} 64.1\\
  & 20\ & \gr{105.6} 105.6 & \gr{105.6} 58.3 & \gr{61.1} \textbf{61.1} & \gr{61.1} 66.0 & \gr{43.5} \textbf{43.5} & \gr{43.5} 46.9 & \gr{41.9} \underline{\textbf{41.9}} & \gr{41.9} 51.2 & \gr{49.7} \textbf{49.7} & \gr{49.7} 55.3\\
  & 50\ & \gr{63.5} 63.5 & \gr{63.5} 37.7 & \gr{30.9} \textbf{30.9} & \gr{30.9} 34.1 & \gr{26.8} \underline{\textbf{26.8}} & \gr{26.8} 24.9 & \gr{34.5} \textbf{34.5} & \gr{34.5} 31.7 & \gr{35.1} \textbf{35.1} & \gr{35.1} 39.2\\
\end{tabular}
        \raggedleft Average cost  \hspace{1.3in}
    \end{minipage}
    \begin{minipage}[b]{2.8in}
\colorlet{maxcolor}{white}
\colorlet{mincolor}{blue!50}
\newcolumntype{C}[1]{>{\centering\arraybackslash}m{#1}}
\renewcommand\gr[1]{\gradient{#1}{-0.01}{0.50}{1.00}{maxcolor}{mincolor}}
\setlength{\tabcolsep}{2pt}
\setlength{\colwidth}{1.2cm}

\begin{tabular}[b]{|C{\colwidth}C{\colwidth}C{\colwidth}C{\colwidth}C{\colwidth}|}
  \multicolumn{5}{c}{less greedy $\quad \longleftarrow \alpha \longrightarrow \quad$ more greedy}\\
    \multicolumn{1}{c}{0/UCT} & \multicolumn{1}{c}{0.2} & \multicolumn{1}{c}{0.5} & \multicolumn{1}{c}{0.8} & \multicolumn{1}{c}{1}\\
    \hhline{|-----}
N/A & \textbf{4.64e-08} & \textbf{1.01e-06} & \textbf{1.87e-05} & \textbf{4.36e-06}\\
N/A & \textbf{1.25e-10} & \textbf{9.02e-10} & \textbf{2.00e-10} & \textbf{8.14e-08}\\
N/A & \textbf{2.33e-05} & \textbf{2.98e-11} & \textbf{4.61e-11} & \textbf{1.32e-08}\\
N/A & \textbf{1.27e-07} & \textbf{6.90e-11} & \textbf{1.05e-06} & \textbf{1.25e-05}\\
N/A & \textbf{9.66e-12} & \textbf{3.11e-14} & \textbf{7.90e-14} & \textbf{3.52e-09}\\
N/A & \textbf{4.19e-12} & \textbf{9.96e-13} & \textbf{4.17e-16} & \textbf{2.02e-11}\\
N/A & \textbf{2.08e-07} & \textbf{3.66e-09} & \textbf{5.63e-10} & \textbf{1.44e-05}\\
N/A & \textbf{5.91e-03} & \textbf{3.01e-08} & \textbf{1.64e-06} & \textbf{5.13e-03}\\
N/A & \textbf{5.98e-05} & \textbf{1.88e-07} & \textbf{5.81e-09} & \textbf{1.04e-04}\\
N/A & 6.76e-02 & \textbf{5.15e-07} & \textbf{1.18e-03} & 2.36e-02\\
  \hhline{|-----}
\end{tabular}
        \centering $p$-values
    \end{minipage}
    \begin{center}
        \captsize \textbf{(d) }\domain{elevators} problem \texttt{p4}. The landmark extraction algorithm $\mathrm{LM}^{\text{RHW}}$ generated 11 nontrivial landmarks for this problem. 
    \end{center}

    \begin{minipage}[b]{4.1in}
        \input{float/tab/Elevators_4_full}
        \raggedleft Average cost  \hspace{1.3in}
    \end{minipage}
    \begin{minipage}[b]{2.8in}
        \input{float/tab/Elevators_4_pvals}
        \centering $p$-values
    \end{minipage}
    \begin{center}
        \captsize \textbf{(e) }\domain{elevators} problem \texttt{p5}. The landmark extraction algorithm $\mathrm{LM}^{\text{RHW}}$ generated 11 nontrivial landmarks for this problem. 
    \end{center}

    \begin{minipage}[b]{4.1in}
        \input{float/tab/Elevators_5_full}
        \raggedleft Average cost  \hspace{1.3in}
    \end{minipage}
    \begin{minipage}[b]{2.8in}
        \input{float/tab/Elevators_5_pvals}
        \centering $p$-values
    \end{minipage}
    \begin{center}
        \captsize \textbf{(f) }\domain{elevators} problem \texttt{p6}. The landmark extraction algorithm $\mathrm{LM}^{\text{RHW}}$ generated 15 nontrivial landmarks for this problem. 
    \end{center}
    \continuefigcaption{The left tables report the average cost (lower is better) of the solutions generated by \LAMP in the \domain{elevators} domain over 75 runs. Underlined values indicate the best performing $\alpha$-value in the row. The right tables report $p$-values comparing the average cost of solutions generated by LAMP with the average cost of solutions generated by standard UCT ($\alpha = 0$) at the same number of rollouts. In all tables, boldfaced values indicate where LAMP significantly dominated standard UCT at the same number of rollouts, with $p < 0.0125$, the adjusted level using a Bonferroni adjustment of $\alpha_{\textit{stat}} = 0.05/4$.}
\end{figure}

\begin{figure}
    \centering
    \begin{minipage}[b]{4.1in}
        \input{float/tab/Elevators_6_full}
        \raggedleft Average cost  \hspace{1.3in}
    \end{minipage}
    \begin{minipage}[b]{2.8in}
\colorlet{maxcolor}{white}
\colorlet{mincolor}{blue!50}
\newcolumntype{C}[1]{>{\centering\arraybackslash}m{#1}}
\renewcommand\gr[1]{\gradient{#1}{-0.01}{0.50}{1.00}{maxcolor}{mincolor}}
\setlength{\tabcolsep}{2pt}
\setlength{\colwidth}{1.2cm}

\begin{tabular}[b]{|C{\colwidth}C{\colwidth}C{\colwidth}C{\colwidth}C{\colwidth}|}
  \multicolumn{5}{c}{less greedy $\quad \longleftarrow \alpha \longrightarrow \quad$ more greedy}\\
    \multicolumn{1}{c}{0/UCT} & \multicolumn{1}{c}{0.2} & \multicolumn{1}{c}{0.5} & \multicolumn{1}{c}{0.8} & \multicolumn{1}{c}{1}\\
    \hhline{|-----}
N/A & 1.31e-02 & \textbf{1.12e-02} & 1.54e-02 & 1.12e-01\\
N/A & 4.18e-02 & 2.80e-02 & 6.11e-02 & 5.11e-02\\
N/A & 3.16e-02 & 4.19e-02 & \textbf{9.43e-03} & 2.59e-02\\
N/A & \textbf{1.10e-04} & \textbf{7.47e-04} & \textbf{9.09e-06} & \textbf{1.28e-04}\\
N/A & \textbf{2.00e-12} & \textbf{1.11e-09} & \textbf{1.27e-08} & \textbf{1.42e-09}\\
N/A & \textbf{2.03e-16} & \textbf{6.09e-14} & \textbf{5.83e-19} & \textbf{1.43e-16}\\
N/A & \textbf{1.76e-31} & \textbf{2.37e-27} & \textbf{1.98e-31} & \textbf{7.54e-26}\\
N/A & \textbf{7.80e-47} & \textbf{2.04e-52} & \textbf{3.16e-49} & \textbf{5.86e-42}\\
N/A & \textbf{8.29e-97} & \textbf{4.86e-90} & \textbf{2.02e-75} & \textbf{4.10e-73}\\
N/A & \textbf{1.2e-132} & \textbf{5.62e-88} & \textbf{6.78e-91} & \textbf{4.33e-71}\\
  \hhline{|-----}
\end{tabular}
        \centering $p$-values
    \end{minipage}
    \begin{center}
        \captsize \textbf{(g) }\domain{elevators} problem \texttt{p7}. The landmark extraction algorithm $\mathrm{LM}^{\text{RHW}}$ generated 16 nontrivial landmarks for this problem. 
    \end{center}

    \begin{minipage}[b]{4.1in}
        \input{float/tab/Elevators_7_full}
        \raggedleft Average cost  \hspace{1.3in}
    \end{minipage}
    \begin{minipage}[b]{2.8in}
        \input{float/tab/Elevators_7_pvals}
        \centering $p$-values
    \end{minipage}
    \begin{center}
        \captsize \textbf{(h) }\domain{elevators} problem \texttt{p8}. The landmark extraction algorithm $\mathrm{LM}^{\text{RHW}}$ generated 16 nontrivial landmarks for this problem. 
    \end{center}
    
    \begin{minipage}[b]{4.1in}
        \input{float/tab/Elevators_8_full}
        \raggedleft Average cost  \hspace{1.3in}
    \end{minipage}
    \begin{minipage}[b]{2.8in}
        \input{float/tab/Elevators_8_pvals}
        \centering $p$-values
    \end{minipage}
    \begin{center}
        \captsize \textbf{(i) }\domain{elevators} problem \texttt{p9}. The landmark extraction algorithm $\mathrm{LM}^{\text{RHW}}$ generated 14 nontrivial landmarks for this problem. 
    \end{center}

    \continuefigcaption{The left tables report the average cost (lower is better) of the solutions generated by \LAMP in the \domain{elevators} domain over 75 runs. Underlined values indicate the best performing $\alpha$-value in the row. The right tables report $p$-values comparing the average cost of solutions generated by LAMP with the average cost of solutions generated by standard UCT ($\alpha = 0$) at the same number of rollouts. In all tables, boldfaced values indicate where LAMP significantly dominated standard UCT at the same number of rollouts, with $p < 0.0125$, the adjusted level using a Bonferroni adjustment of $\alpha_{\textit{stat}} = 0.05/4$.}
\end{figure}

\begin{figure}
    \centering
    \begin{minipage}[b]{4.1in}
        \input{float/tab/Elevators_9_full}
        \raggedleft Average cost  \hspace{1.3in}
    \end{minipage}
    \begin{minipage}[b]{2.8in}
        \input{float/tab/Elevators_9_pvals}
        \centering $p$-values
    \end{minipage}
    \begin{center}
        \captsize \textbf{(j) }\domain{elevators} problem \texttt{p10}. The landmark extraction algorithm $\mathrm{LM}^{\text{RHW}}$ generated 16 nontrivial landmarks for this problem. 
    \end{center}

    \begin{minipage}[b]{4.1in}
        \input{float/tab/Elevators_10_full}
        \raggedleft Average cost  \hspace{1.3in}
    \end{minipage}
    \begin{minipage}[b]{2.8in}
        \input{float/tab/Elevators_10_pvals}
        \centering $p$-values
    \end{minipage}
    \begin{center}
        \captsize \textbf{(k) }\domain{elevators} problem \texttt{p11}. The landmark extraction algorithm $\mathrm{LM}^{\text{RHW}}$ generated 26 nontrivial landmarks for this problem. 
    \end{center}
    
    \begin{minipage}[b]{4.1in}
        \input{float/tab/Elevators_11_full}
        \raggedleft Average cost  \hspace{1.3in}
    \end{minipage}
    \begin{minipage}[b]{2.8in}
        \input{float/tab/Elevators_11_pvals}
        \centering $p$-values
    \end{minipage}
    \begin{center}
        \captsize \textbf{(l) }\domain{elevators} problem \texttt{p12}. The landmark extraction algorithm $\mathrm{LM}^{\text{RHW}}$ generated 18 nontrivial landmarks for this problem. 
    \end{center}

    \continuefigcaption{The left tables report the average cost (lower is better) of the solutions generated by \LAMP in the \domain{elevators} domain over 75 runs. Underlined values indicate the best performing $\alpha$-value in the row. The right tables report $p$-values comparing the average cost of solutions generated by LAMP with the average cost of solutions generated by standard UCT ($\alpha = 0$) at the same number of rollouts. In all tables, boldfaced values indicate where LAMP significantly dominated standard UCT at the same number of rollouts, with $p < 0.0125$, the adjusted level using a Bonferroni adjustment of $\alpha_{\textit{stat}} = 0.05/4$.}
\end{figure}

\begin{figure}
    \centering 

    \begin{minipage}[b]{4.1in}
        \input{float/tab/Elevators_12_full}
        \raggedleft Average cost  \hspace{1.3in}
    \end{minipage}
    \begin{minipage}[b]{2.8in}
        \input{float/tab/Elevators_12_pvals}
        \centering $p$-values
    \end{minipage}
    \begin{center}
        \captsize \textbf{(m)} \domain{elevators} problem \texttt{p13}. The landmark extraction algorithm $\mathrm{LM}^{\text{RHW}}$ generated 30 nontrivial landmarks for this problem. 
    \end{center}

    \begin{minipage}[b]{4.1in}
        \input{float/tab/Elevators_13_full}
        \raggedleft Average cost  \hspace{1.3in}
    \end{minipage}
    \begin{minipage}[b]{2.8in}
        \input{float/tab/Elevators_13_pvals}
        \centering $p$-values
    \end{minipage}
    \begin{center}
        \captsize \textbf{(n)} \domain{elevators} problem \texttt{p14}. The landmark extraction algorithm $\mathrm{LM}^{\text{RHW}}$ generated 24 nontrivial landmarks for this problem. 
    \end{center}

    \begin{minipage}[b]{4.1in}
        \input{float/tab/Elevators_14_full}
        \raggedleft Average cost  \hspace{1.3in}
    \end{minipage}
    \begin{minipage}[b]{2.8in}
        \input{float/tab/Elevators_14_pvals}
        \centering $p$-values
    \end{minipage}
    \begin{center}
        \captsize \textbf{(o)} \domain{elevators} problem \texttt{p15}. The landmark extraction algorithm $\mathrm{LM}^{\text{RHW}}$ generated 24 nontrivial landmarks for this problem. 
    \end{center}

    \continuefigcaption{The left tables report the average cost (lower is better) of the solutions generated by \LAMP in the \domain{elevators} domain over 75 runs. Underlined values indicate the best performing $\alpha$-value in the row. The right tables report $p$-values comparing the average cost of solutions generated by LAMP with the average cost of solutions generated by standard UCT ($\alpha = 0$) at the same number of rollouts. In all tables, boldfaced values indicate where LAMP significantly dominated standard UCT at the same number of rollouts, with $p < 0.0125$, the adjusted level using a Bonferroni adjustment of $\alpha_{\textit{stat}} = 0.05/4$.}
\end{figure}

\begin{figure}
    \centering
    \begin{minipage}[b]{4.1in}
        \input{float/tab/Zeno_0_full}
        \raggedleft Average cost  \hspace{1.3in}
    \end{minipage}
    \begin{minipage}[b]{2.8in}
        \input{float/tab/Zeno_0_pvals}
        \centering $p$-values
    \end{minipage}
    \begin{center}
        \captsize \textbf{(a)} \domain{zenotravel} problem \texttt{p1}. The landmark extraction algorithm $\mathrm{LM}^{\text{RHW}}$ generated 0 nontrivial landmarks for this problem. This problem instance is trivial---the initial state satisfies the goal condition.
    \end{center}
    
    \begin{minipage}[b]{4.1in}
        \input{float/tab/Zeno_1_full}
        \raggedleft Average cost  \hspace{1.3in}
    \end{minipage}
    \begin{minipage}[b]{2.8in}
        \input{float/tab/Zeno_1_pvals}
        \centering $p$-values
    \end{minipage}
    \begin{center}
        \captsize\textbf{(b)} \domain{zenotravel} problem \texttt{p2}. The landmark extraction algorithm $\mathrm{LM}^{\text{RHW}}$ generated 11 nontrivial landmarks for this problem. 
    \end{center}
    
    \begin{minipage}[b]{4.1in}
        \input{float/tab/Zeno_2_full}
        \raggedleft Average cost  \hspace{1.3in}
    \end{minipage}
    \begin{minipage}[b]{2.8in}
        \input{float/tab/Zeno_2_pvals}
        \centering $p$-values
    \end{minipage}
    \begin{center}
        \captsize \textbf{(c)} \domain{zenotravel} problem \texttt{p3}. The landmark extraction algorithm $\mathrm{LM}^{\text{RHW}}$ generated 13 nontrivial landmarks for this problem. 
    \end{center}
    \caption{The left tables report the average cost (lower is better) of the solutions generated by \LAMP in the \domain{zenotravel} domain over 75 runs. Underlined values indicate the best performing $\alpha$-value in the row. The right tables report $p$-values comparing the average cost of solutions generated by LAMP with the average cost of solutions generated by standard UCT ($\alpha = 0$) at the same number of rollouts. In all tables, boldfaced values indicate where LAMP significantly dominated standard UCT at the same number of rollouts, with $p < 0.0125$, the adjusted level using a Bonferroni adjustment of $\alpha_{\textit{stat}} = 0.05/4$.}
    \label{fig:zeno_full}
\end{figure}

\begin{figure}
    \centering
    \begin{minipage}[b]{4.1in}
        \input{float/tab/Zeno_3_full}
        \raggedleft Average cost  \hspace{1.3in}
    \end{minipage}
    \begin{minipage}[b]{2.8in}
        \input{float/tab/Zeno_3_pvals}
        \centering $p$-values
    \end{minipage}
    \begin{center}
        \captsize \textbf{(d)} \domain{zenotravel} problem \texttt{p4}. The landmark extraction algorithm $\mathrm{LM}^{\text{RHW}}$ generated 11 nontrivial landmarks for this problem. 
    \end{center}
    
    \begin{minipage}[b]{4.1in}
        \input{float/tab/Zeno_4_full}
        \raggedleft Average cost  \hspace{1.3in}
    \end{minipage}
    \begin{minipage}[b]{2.8in}
        \input{float/tab/Zeno_4_pvals}
        \centering $p$-values
    \end{minipage}
    \begin{center}
        \captsize\textbf{(e)} \domain{zenotravel} problem \texttt{p5}. The landmark extraction algorithm $\mathrm{LM}^{\text{RHW}}$ generated 12 nontrivial landmarks for this problem. 
    \end{center}

    \continuefigcaption{The left tables report the average cost (lower is better) of the solutions generated by \LAMP in the \domain{zenotravel} domain over 75 runs. Underlined values indicate the best performing $\alpha$-value in the row. The right tables report $p$-values comparing the average cost of solutions generated by LAMP with the average cost of solutions generated by standard UCT ($\alpha = 0$) at the same number of rollouts. In all tables, boldfaced values indicate where LAMP significantly dominated standard UCT at the same number of rollouts, with $p < 0.0125$, the adjusted level using a Bonferroni adjustment of $\alpha_{\textit{stat}} = 0.05/4$.}
\end{figure}

\clearpage
\section{Additional probabilistically interesting benchmarks}
\label{appendix:results_prob_interesting}

\begin{figure}
    \centering
    \begin{minipage}[b]{3in}
        \input{float/tab/Ex_blocks_0_full}
        \raggedleft Success rate  \hspace{.8in}
    \end{minipage}
    \begin{minipage}[b]{2.8in}
        \input{float/tab/Ex_blocks_0_pvals}
        \centering $p$-values
    \end{minipage}
    \begin{center}
        \captsize \textbf{(a)} \domain{exploding\_blocksworld} problem \texttt{p1}. The landmark extraction algorithm $\mathrm{LM}^{\text{RHW}}$ generated 7 nontrivial landmarks for this problem. 
    \end{center}
    \begin{minipage}[b]{3in}
        \input{float/tab/Ex_blocks_1_full}
        \raggedleft Success rate  \hspace{.8in}
    \end{minipage}
    \begin{minipage}[b]{2.8in}
        \input{float/tab/Ex_blocks_1_pvals}
        \centering $p$-values
    \end{minipage}
    \begin{center}
        \captsize \textbf{(b)} \domain{exploding\_blocksworld} problem \texttt{p2}. The landmark extraction algorithm $\mathrm{LM}^{\text{RHW}}$ generated 4 nontrivial landmarks for this problem. 
    \end{center}
    \begin{minipage}[b]{3in}
        \input{float/tab/Ex_blocks_2_full}
        \raggedleft Success rate  \hspace{.8in}
    \end{minipage}
    \begin{minipage}[b]{2.8in}
        \input{float/tab/Ex_blocks_2_pvals}
        \centering $p$-values
    \end{minipage}
    \begin{center}
        \captsize \textbf{(c)} \domain{exploding\_blocksworld} problem \texttt{p3}. The landmark extraction algorithm $\mathrm{LM}^{\text{RHW}}$ generated 7 nontrivial landmarks for this problem. 
    \end{center}
    \caption{The left tables report the success rate (higher is better) of the solutions generated by \LAMP in the \domain{exploding\_blocksworld} domain over 75 runs. Underlined values indicate the best performing $\alpha$-value in the row. The right tables report $p$-values computed using a two-tailed Boschloo exact test comparing the success rate of LAMP with the success rate of standard UCT ($\alpha = 0$) at the same number of rollouts. In all tables, boldfaced values indicate where LAMP significantly dominated standard UCT at the same number of rollouts, with $p < 0.0125$, the adjusted level using a Bonferroni adjustment of $\alpha_{\textit{stat}} = 0.05/4$.}
    \label{fig:eblocks_full}
\end{figure}

\begin{figure}
    \centering
    \begin{minipage}[b]{3in}
        \input{float/tab/Ex_blocks_3_full}
        \raggedleft Success rate  \hspace{.8in}
    \end{minipage}
    \begin{minipage}[b]{2.8in}
        \input{float/tab/Ex_blocks_3_pvals}
        \centering $p$-values
    \end{minipage}
    \begin{center}
        \captsize \textbf{(d)} \domain{exploding\_blocksworld} problem \texttt{p4}. The landmark extraction algorithm $\mathrm{LM}^{\text{RHW}}$ generated 9 nontrivial landmarks for this problem. 
    \end{center}
    \begin{minipage}[b]{3in}
        \input{float/tab/Ex_blocks_4_full}
        \raggedleft Success rate  \hspace{.8in}
    \end{minipage}
    \begin{minipage}[b]{2.8in}
        \input{float/tab/Ex_blocks_4_pvals}
        \centering $p$-values
    \end{minipage}
    \begin{center}
        \captsize \textbf{(e)} \domain{exploding\_blocksworld} problem \texttt{p5}. The landmark extraction algorithm $\mathrm{LM}^{\text{RHW}}$ generated 0 nontrivial landmarks for this problem. This problem instance is trivial---the initial state satisfies the goal condition.
    \end{center}
    \begin{minipage}[b]{3in}
        \input{float/tab/Ex_blocks_5_full}
        \raggedleft Success rate  \hspace{.8in}
    \end{minipage}
    \begin{minipage}[b]{2.8in}
        \input{float/tab/Ex_blocks_5_pvals}
        \centering $p$-values
    \end{minipage}
    \begin{center}
        \captsize \textbf{(f)} \domain{exploding\_blocksworld} problem \texttt{p6}. The landmark extraction algorithm $\mathrm{LM}^{\text{RHW}}$ generated 16 nontrivial landmarks for this problem. 
    \end{center}
    \continuefigcaption{The left tables report the success rate (higher is better) of the solutions generated by \LAMP in the \domain{exploding\_blocksworld} domain over 75 runs. Underlined values indicate the best performing $\alpha$-value in the row. The right tables report $p$-values computed using a two-tailed Boschloo exact test comparing the success rate of LAMP with the success rate of standard UCT ($\alpha = 0$) at the same number of rollouts. In all tables, boldfaced values indicate where LAMP significantly dominated standard UCT at the same number of rollouts, with $p < 0.0125$, the adjusted level using a Bonferroni adjustment of $\alpha_{\textit{stat}} = 0.05/4$.}
\end{figure}

\begin{figure}
    \centering
    \begin{minipage}[b]{3in}
        \input{float/tab/Ex_blocks_6_full}
        \raggedleft Success rate  \hspace{.8in}
    \end{minipage}
    \begin{minipage}[b]{2.8in}
        \input{float/tab/Ex_blocks_6_pvals}
        \centering $p$-values
    \end{minipage}
    \begin{center}
        \captsize \textbf{(g)} \domain{exploding\_blocksworld} problem \texttt{p7}. The landmark extraction algorithm $\mathrm{LM}^{\text{RHW}}$ generated 17 nontrivial landmarks for this problem. 
    \end{center}
    \begin{minipage}[b]{3in}
        \input{float/tab/Ex_blocks_7_full}
        \raggedleft Success rate  \hspace{.8in}
    \end{minipage}
    \begin{minipage}[b]{2.8in}
        \input{float/tab/Ex_blocks_7_pvals}
        \centering $p$-values
    \end{minipage}
    \begin{center}
        \captsize \textbf{(h)} \domain{exploding\_blocksworld} problem \texttt{p8}. The landmark extraction algorithm $\mathrm{LM}^{\text{RHW}}$ generated 13 nontrivial landmarks for this problem. 
    \end{center}
    \begin{minipage}[b]{3in}
        \input{float/tab/Ex_blocks_8_full}
        \raggedleft Success rate  \hspace{.8in}
    \end{minipage}
    \begin{minipage}[b]{2.8in}
        \input{float/tab/Ex_blocks_8_pvals}
        \centering $p$-values
    \end{minipage}
    \begin{center}
        \captsize \textbf{(i)} \domain{exploding\_blocksworld} problem \texttt{p9}. The landmark extraction algorithm $\mathrm{LM}^{\text{RHW}}$ generated 17 nontrivial landmarks for this problem. 
    \end{center}
    \continuefigcaption{The left tables report the success rate (higher is better) of the solutions generated by \LAMP in the \domain{exploding\_blocksworld} domain over 75 runs. Underlined values indicate the best performing $\alpha$-value in the row. The right tables report $p$-values computed using a two-tailed Boschloo exact test comparing the success rate of LAMP with the success rate of standard UCT ($\alpha = 0$) at the same number of rollouts. In all tables, boldfaced values indicate where LAMP significantly dominated standard UCT at the same number of rollouts, with $p < 0.0125$, the adjusted level using a Bonferroni adjustment of $\alpha_{\textit{stat}} = 0.05/4$.}
\end{figure}

\begin{figure}
    \centering
    \begin{minipage}[b]{3in}
        \input{float/tab/Tire_0_full}
        \raggedleft Success rate  \hspace{.8in}
    \end{minipage}
    \begin{minipage}[b]{2.8in}
        \input{float/tab/Tire_0_pvals}
        \centering $p$-values
    \end{minipage}
    \begin{center}
        \captsize \textbf{(a)} \domain{tireworld} problem \texttt{p1}. The landmark extraction algorithm $\mathrm{LM}^{\text{RHW}}$ generated 5 nontrivial landmarks for this problem. 
    \end{center}

    \begin{minipage}[b]{3in}
        \input{float/tab/Tire_1_full}
        \raggedleft Success rate  \hspace{.8in}
    \end{minipage}
    \begin{minipage}[b]{2.8in}
        \input{float/tab/Tire_1_pvals}
        \centering $p$-values
    \end{minipage}
    \begin{center}
        \captsize \textbf{(b)} \domain{tireworld} problem \texttt{p2}. The landmark extraction algorithm $\mathrm{LM}^{\text{RHW}}$ generated 1 nontrivial landmark for this problem. 
    \end{center}

    \begin{minipage}[b]{3in}
        \input{float/tab/Tire_2_full}
        \raggedleft Success rate  \hspace{.8in}
    \end{minipage}
    \begin{minipage}[b]{2.8in}
        \input{float/tab/Tire_2_pvals}
        \centering $p$-values
    \end{minipage}
    \begin{center}
        \captsize \textbf{(c)} \domain{tireworld} problem \texttt{p3}. The landmark extraction algorithm $\mathrm{LM}^{\text{RHW}}$ generated 2 nontrivial landmarks for this problem. 
    \end{center}
    \caption{The left tables report the success rate (higher is better) of the solutions generated by \LAMP in the \domain{tireworld} domain over 75 runs. Underlined values indicate the best performing $\alpha$-value in the row. The right tables report $p$-values computed using a two-tailed Boschloo exact test comparing the success rate of LAMP with the success rate of standard UCT ($\alpha = 0$) at the same number of rollouts. In all tables, boldfaced values indicate where LAMP significantly dominated standard UCT at the same number of rollouts, with $p < 0.0125$, the adjusted level using a Bonferroni adjustment of $\alpha_{\textit{stat}} = 0.05/4$.}
    \label{fig:tire_full}
\end{figure}

\begin{figure}
    \centering
    \begin{minipage}[b]{3in}
        \input{float/tab/Tire_3_full}
        \raggedleft Success rate  \hspace{.8in}
    \end{minipage}
    \begin{minipage}[b]{2.8in}
        \input{float/tab/Tire_3_pvals}
        \centering $p$-values
    \end{minipage}
    \begin{center}
        \captsize \textbf{(d) }\domain{tireworld} problem \texttt{p4}. The landmark extraction algorithm $\mathrm{LM}^{\text{RHW}}$ generated 2 nontrivial landmarks for this problem. 
    \end{center}

    \begin{minipage}[b]{3in}
        \input{float/tab/Tire_4_full}
        \raggedleft Success rate  \hspace{.8in}
    \end{minipage}
    \begin{minipage}[b]{2.8in}
        \input{float/tab/Tire_4_pvals}
        \centering $p$-values
    \end{minipage}
    \begin{center}
        \captsize \textbf{(e) }\domain{tireworld} problem \texttt{p5}. The landmark extraction algorithm $\mathrm{LM}^{\text{RHW}}$ generated 1 nontrivial landmark for this problem. 
    \end{center}

    \begin{minipage}[b]{3in}
        \input{float/tab/Tire_5_full}
        \raggedleft Success rate  \hspace{.8in}
    \end{minipage}
    \begin{minipage}[b]{2.8in}
        \input{float/tab/Tire_5_pvals}
        \centering $p$-values
    \end{minipage}
    \begin{center}
        \captsize \textbf{(f) }\domain{tireworld} problem \texttt{p6}. The landmark extraction algorithm $\mathrm{LM}^{\text{RHW}}$ generated 2 nontrivial landmarks for this problem. 
    \end{center}
    \continuefigcaption{The left tables report the success rate (higher is better) of the solutions generated by \LAMP in the \domain{tireworld} domain over 75 runs. Underlined values indicate the best performing $\alpha$-value in the row. The right tables report $p$-values computed using a two-tailed Boschloo exact test comparing the success rate of LAMP with the success rate of standard UCT ($\alpha = 0$) at the same number of rollouts. In all tables, boldfaced values indicate where LAMP significantly dominated standard UCT at the same number of rollouts, with $p < 0.0125$, the adjusted level using a Bonferroni adjustment of $\alpha_{\textit{stat}} = 0.05/4$.}
\end{figure}

\begin{figure}
    \centering
    \begin{minipage}[b]{3in}
        \input{float/tab/Tire_6_full}
        \raggedleft Success rate  \hspace{.8in}
    \end{minipage}
    \begin{minipage}[b]{2.8in}
        \input{float/tab/Tire_6_pvals}
        \centering $p$-values
    \end{minipage}
    \begin{center}
        \captsize \textbf{(g) }\domain{tireworld} problem \texttt{p7}. The landmark extraction algorithm $\mathrm{LM}^{\text{RHW}}$ generated 2 nontrivial landmarks for this problem. 
    \end{center}

    \begin{minipage}[b]{3in}
        \input{float/tab/Tire_7_full}
        \raggedleft Success rate  \hspace{.8in}
    \end{minipage}
    \begin{minipage}[b]{2.8in}
        \input{float/tab/Tire_7_pvals}
        \centering $p$-values
    \end{minipage}
    \begin{center}
        \captsize \textbf{(h) }\domain{tireworld} problem \texttt{p8}. The landmark extraction algorithm $\mathrm{LM}^{\text{RHW}}$ generated 1 nontrivial landmark for this problem. 
    \end{center}
    
    \begin{minipage}[b]{3in}
        \input{float/tab/Tire_8_full}
        \raggedleft Success rate  \hspace{.8in}
    \end{minipage}
    \begin{minipage}[b]{2.8in}
        \input{float/tab/Tire_8_pvals}
        \centering $p$-values
    \end{minipage}
    \begin{center}
        \captsize \textbf{(i) }\domain{tireworld} problem \texttt{p9}. The landmark extraction algorithm $\mathrm{LM}^{\text{RHW}}$ generated 1 nontrivial landmark for this problem. 
    \end{center}

    \continuefigcaption{The left tables report the success rate (higher is better) of the solutions generated by \LAMP in the \domain{tireworld} domain over 75 runs. Underlined values indicate the best performing $\alpha$-value in the row. The right tables report $p$-values computed using a two-tailed Boschloo exact test comparing the success rate of LAMP with the success rate of standard UCT ($\alpha = 0$) at the same number of rollouts. In all tables, boldfaced values indicate where LAMP significantly dominated standard UCT at the same number of rollouts, with $p < 0.0125$, the adjusted level using a Bonferroni adjustment of $\alpha_{\textit{stat}} = 0.05/4$.}
\end{figure}

\begin{figure}
    \centering
    \begin{minipage}[b]{3in}
        \input{float/tab/Tire_9_full}
        \raggedleft Success rate  \hspace{.8in}
    \end{minipage}
    \begin{minipage}[b]{2.8in}
        \input{float/tab/Tire_9_pvals}
        \centering $p$-values
    \end{minipage}
    \begin{center}
        \captsize \textbf{(j) }\domain{tireworld} problem \texttt{p10}. The landmark extraction algorithm $\mathrm{LM}^{\text{RHW}}$ generated 1 nontrivial landmark for this problem. 
    \end{center}

    \begin{minipage}[b]{3in}
        \input{float/tab/Tire_10_full}
        \raggedleft Success rate  \hspace{.8in}
    \end{minipage}
    \begin{minipage}[b]{2.8in}
        \input{float/tab/Tire_10_pvals}
        \centering $p$-values
    \end{minipage}
    \begin{center}
        \captsize \textbf{(k) }\domain{tireworld} problem \texttt{p11}. The landmark extraction algorithm $\mathrm{LM}^{\text{RHW}}$ generated 1 nontrivial landmark for this problem. 
    \end{center}
    
    \begin{minipage}[b]{3in}
        \input{float/tab/Tire_11_full}
        \raggedleft Success rate  \hspace{.8in}
    \end{minipage}
    \begin{minipage}[b]{2.8in}
        \input{float/tab/Tire_11_pvals}
        \centering $p$-values
    \end{minipage}
    \begin{center}
        \captsize \textbf{(l) }\domain{tireworld} problem \texttt{p12}. The landmark extraction algorithm $\mathrm{LM}^{\text{RHW}}$ generated 1 nontrivial landmark for this problem. 
    \end{center}

    \continuefigcaption{The left tables report the success rate (higher is better) of the solutions generated by \LAMP in the \domain{tireworld} domain over 75 runs. Underlined values indicate the best performing $\alpha$-value in the row. The right tables report $p$-values computed using a two-tailed Boschloo exact test comparing the success rate of LAMP with the success rate of standard UCT ($\alpha = 0$) at the same number of rollouts. In all tables, boldfaced values indicate where LAMP significantly dominated standard UCT at the same number of rollouts, with $p < 0.0125$, the adjusted level using a Bonferroni adjustment of $\alpha_{\textit{stat}} = 0.05/4$.}
\end{figure}

\begin{figure}
    \centering 

    \begin{minipage}[b]{3in}
        \input{float/tab/Tire_12_full}
        \raggedleft Success rate  \hspace{.8in}
    \end{minipage}
    \begin{minipage}[b]{2.8in}
        \input{float/tab/Tire_12_pvals}
        \centering $p$-values
    \end{minipage}
    \begin{center}
        \captsize \textbf{(m)} \domain{tireworld} problem \texttt{p13}. The landmark extraction algorithm $\mathrm{LM}^{\text{RHW}}$ generated 1 nontrivial landmark for this problem. 
    \end{center}

    \begin{minipage}[b]{3in}
        \input{float/tab/Tire_13_full}
        \raggedleft Success rate  \hspace{.8in}
    \end{minipage}
    \begin{minipage}[b]{2.8in}
        \input{float/tab/Tire_13_pvals}
        \centering $p$-values
    \end{minipage}
    \begin{center}
        \captsize \textbf{(n)} \domain{tireworld} problem \texttt{p14}. The landmark extraction algorithm $\mathrm{LM}^{\text{RHW}}$ generated 1 nontrivial landmark for this problem. 
    \end{center}

    \begin{minipage}[b]{3in}
        \input{float/tab/Tire_14_full}
        \raggedleft Success rate  \hspace{.8in}
    \end{minipage}
    \begin{minipage}[b]{2.8in}
        \input{float/tab/Tire_14_pvals}
        \centering $p$-values
    \end{minipage}
    \begin{center}
        \captsize \textbf{(o)} \domain{tireworld} problem \texttt{p15}. The landmark extraction algorithm $\mathrm{LM}^{\text{RHW}}$ generated 3 nontrivial landmarks for this problem. 
    \end{center}

    \continuefigcaption{The left tables report the success rate (higher is better) of the solutions generated by \LAMP in the \domain{tireworld} domain over 75 runs. Underlined values indicate the best performing $\alpha$-value in the row. The right tables report $p$-values computed using a two-tailed Boschloo exact test comparing the success rate of LAMP with the success rate of standard UCT ($\alpha = 0$) at the same number of rollouts. In all tables, boldfaced values indicate where LAMP significantly dominated standard UCT at the same number of rollouts, with $p < 0.0125$, the adjusted level using a Bonferroni adjustment of $\alpha_{\textit{stat}} = 0.05/4$.}
\end{figure}

\begin{figure}
    \centering
    \begin{minipage}[b]{3in}
        \input{float/tab/Tire_triangle_0_full}
        \raggedleft Success rate  \hspace{.8in}
    \end{minipage}
    \begin{minipage}[b]{2.8in}
        \input{float/tab/Tire_triangle_0_pvals}
        \centering $p$-values
    \end{minipage}
    \begin{center}
        \captsize \textbf{(a)} \domain{triangle\_tireworld} problem \texttt{p1}. The landmark extraction algorithm $\mathrm{LM}^{\text{RHW}}$ generated 2 nontrivial landmarks for this problem. 
    \end{center}
    \begin{minipage}[b]{3in}
        \input{float/tab/Tire_triangle_1_full}
        \raggedleft Success rate  \hspace{.8in}
    \end{minipage}
    \begin{minipage}[b]{2.8in}
        \input{float/tab/Tire_triangle_1_pvals}
        \centering $p$-values
    \end{minipage}
    \begin{center}
        \captsize \textbf{(b)} \domain{triangle\_tireworld} problem \texttt{p2}. The landmark extraction algorithm $\mathrm{LM}^{\text{RHW}}$ generated 4 nontrivial landmarks for this problem. 
    \end{center}
    \begin{minipage}[b]{3in}
        \input{float/tab/Tire_triangle_2_full}
        \raggedleft Success rate  \hspace{.8in}
    \end{minipage}
    \begin{minipage}[b]{2.8in}
        \input{float/tab/Tire_triangle_2_pvals}
        \centering $p$-values
    \end{minipage}
    \begin{center}
        \captsize \textbf{(c)} \domain{triangle\_tireworld} problem \texttt{p3}. The landmark extraction algorithm $\mathrm{LM}^{\text{RHW}}$ generated 4 nontrivial landmarks for this problem. 
    \end{center}
    \caption{The left tables report the success rate (higher is better) of the solutions generated by \LAMP in the \domain{triangle\_tireworld} domain over 75 runs. Underlined values indicate the best performing $\alpha$-value in the row. The right tables report $p$-values computed using a two-tailed Boschloo exact test comparing the success rate of LAMP with the success rate of standard UCT ($\alpha = 0$) at the same number of rollouts.}
    \label{fig:ttire_full}
\end{figure}

\begin{figure}  
    \centering
    \begin{minipage}[b]{3in}
        \input{float/tab/Tire_triangle_3_full}
        \raggedleft Success rate  \hspace{.8in}
    \end{minipage}
    \begin{minipage}[b]{2.8in}
        \input{float/tab/Tire_triangle_3_pvals}
        \centering $p$-values
    \end{minipage}
    \begin{center}
        \captsize \textbf{(d)} \domain{triangle\_tireworld} problem \texttt{p4}. The landmark extraction algorithm $\mathrm{LM}^{\text{RHW}}$ generated 4 nontrivial landmarks for this problem. 
    \end{center}

    \begin{minipage}[b]{3in}
        \input{float/tab/Tire_triangle_4_full}
        \raggedleft Success rate  \hspace{.8in}
    \end{minipage}
    \begin{minipage}[b]{2.8in}
        \input{float/tab/Tire_triangle_4_pvals}
        \centering $p$-values
    \end{minipage}
    \begin{center}
        \captsize \textbf{(e)} \domain{triangle\_tireworld} problem \texttt{p5}. The landmark extraction algorithm $\mathrm{LM}^{\text{RHW}}$ generated 4 nontrivial landmarks for this problem. 
    \end{center}
    \continuefigcaption{The left tables report the success rate (higher is better) of the solutions generated by \LAMP in the \domain{triangle\_tireworld} domain over 75 runs. Underlined values indicate the best performing $\alpha$-value in the row. The right tables report $p$-values computed using a two-tailed Boschloo exact test comparing the success rate of LAMP with the success rate of standard UCT ($\alpha = 0$) at the same number of rollouts.}
\end{figure}

\end{document}